%% file: neurips_2025.tex
\documentclass{article}

\input{packages}

\title{\ourmethod{}: Key Similarity-Based KV Cache Eviction for Long-Context LLM Inference in Resource-Constrained Environments}

\author{
Junyoung Park\thanks{Corresponding author}
\quad Dalton Jones
\quad Matthew J Morse \\
\quad \textbf{Raghavv Goel}
\quad \textbf{Mingu Lee}
\quad \textbf{Chris Lott} \\
Qualcomm AI Research \\
\texttt{\{junpark,daltjone,mattmors,raghgoel,mingul,clott\}@qti.qualcomm.com}\\
}

\begin{document}

\maketitle

\begin{abstract}
\input{contents/00_abstract}

\end{abstract}

\input{contents/01_intro}
\input{contents/02_background}
\input{contents/03_method}
\input{contents/04_experiment}
\input{contents/05_related_work}
\input{contents/06_conclusion}

\bibliographystyle{plainnat}
\bibliography{reference}

% APPENDIX
\newpage
\appendix
\onecolumn
\input{contents/07_appendix}

% CHECKLIST (not part of appendix)
\clearpage
\section*{NeurIPS Paper Checklist}
\label{sec:checklist}
% If you have a ToC and want it listed there, uncomment the next line:
% \addcontentsline{toc}{section}{NeurIPS Paper Checklist}
\input{contents/checklist}

\end{document}

%% file: packages.tex
    \PassOptionsToPackage{numbers, compress}{natbib}
    \usepackage[final]{neurips_2025}

\usepackage[utf8]{inputenc} % allow utf-8 input
\usepackage[T1]{fontenc}    % use 8-bit T1 fonts
\usepackage{url}            % simple URL typesetting
\usepackage{booktabs}       % professional-quality tables
\usepackage{amsfonts}       % blackboard math symbols
\usepackage{nicefrac}       % compact symbols for 1/2, etc.
\usepackage{microtype}      % microtypography
\usepackage[dvipsnames]{xcolor}         % colors

\usepackage{amsmath}
\usepackage{amssymb}
\usepackage{mathtools}
\usepackage{amsthm}

\usepackage[colorlinks=true, linkcolor=NavyBlue, citecolor=NavyBlue, urlcolor=magenta]{hyperref}
\usepackage[capitalize,noabbrev,nameinlink]{cleveref}

\usepackage{color}

\usepackage{xargs}

\newcommandx{\JP}[2][1=]{%
  \todo[inline,
        linecolor=blue,
        backgroundcolor=blue!25,
        bordercolor=blue,
        #1,
        size=\small]{JP: #2}%
}

\newcommandx{\DJJ}[2][1=]{%
  \todo[inline,
        linecolor=red,
        backgroundcolor=red!25,
        bordercolor=red,
        #1,
        size=\small]{DJ: #2}%
}
\newcommandx{\MM}[2][1=]{%
  \todo[inline,
        linecolor=orange,
        backgroundcolor=orange!25,
        bordercolor=orange,
        #1,
        size=\small]{MM: #2}%
}

\newcommand{\ourmethod}{\textsc{KeyDiff}}

\theoremstyle{plain}
\newtheorem{theorem}{Theorem}[section]

\newtheorem{lemma}[theorem]{Lemma}

\theoremstyle{definition}

\theoremstyle{remark}

\usepackage[textsize=tiny]{todonotes}

\usepackage{multirow}

\usepackage{caption}
\usepackage{subcaption}

\usepackage[toc,page,header]{appendix}
\usepackage{minitoc}

\usepackage{enumitem} 

\usepackage{soul}

\usepackage{adjustbox}

\usepackage{enumitem}

\usepackage{optidef}

\usepackage{gradient-text}

\usepackage[most]{tcolorbox}
\tcbuselibrary{theorems}

\newtcbtheorem[number within=section]{Theorem}{Theorem}%
{colback=blue!5!white, colframe=blue!75!black,
  fonttitle=\bfseries}{th}

\usepackage{wrapfig}

\usepackage{xfrac}

%% file: contents/00_abstract.tex
We demonstrate that geometrically distinctive keys during LLM inference tend to have high attention scores. Based on the phenomenon we propose \ourmethod{}, a training-free KV cache eviction method based solely on key similarity. Unlike other KV cache eviction methods, \ourmethod{} can process arbitrarily long prompts within strict resource constraints and efficiently generate responses. We provide a theoretical basis for \ourmethod{} by relating key diversity with attention scores. These results imply \ourmethod{} can efficiently identify the most important tokens to retain. Notably \ourmethod{} does not rely on attention scores, allowing the use of optimized attention mechanisms like FlashAttention. Under a strict memory allowance, we demonstrate the effectiveness of \ourmethod{} for the Llama and Qwen model families by observing a performance gap of less than 0.04\% with 8K cache budget ($\sim$ 23\% KV cache reduction) from the non-evicting baseline on LongBench for Llama 3.1-8B and Llama 3.2-3B. We also observe near baseline performance for Deepseek-R1-Distill-Llama-8B on the Math500 reasoning benchmark and decrease end-to-end inference latency by up to 30\% compared to the other token-eviction methods.

%% file: contents/01_intro.tex
\section{Introduction}
\label{sec:intro}

Key-Value (KV) caching is a standard technique to accelerate large language model (LLM) inference that reuses key and value states (KVs) from previously processed tokens, enabling efficient autoregressive generation. This is crucial for long-context applications such as document summarization, code generation, question answering \cite{brown2020language, raffel2020exploring, touvron2023llama, dubey2024llama}, retrieval augmented generation \cite{lewis2020retrieval} and reasoning \cite{wei2022chain, kojima2022large, yao2024tree}. However, the memory footprint of the stored KV cache grows linearly with input length, which becomes a bottleneck in memory-constrained environments.

This challenge is particularly acute for LLM inference on edge device, where compute, memory, and power resources are limited \cite{alizadeh2023llm, liu2024mobilellm, van2024gptvq, xu2024empowering}. While \textit{KV cache eviction} policies have been proposed to bound memory overhead by removing unimportant KVs (often measured by attention scores) \cite{xiaoefficient, oren2024transformers, zhang2024h2o}, they typically process the entire prompt at once and violate memory constraints during intermediate computation.

To enforce strict memory bounds throughout the prompt prefill and token generation inference phases, we adopt a block-wise inference strategy: the input prompt is divided into smaller blocks which are processed sequentially by the model, similar to \cite{kwon2023efficient, agrawal2023sarathi, xu2024think}. After processing each block, we evict some of the cached KVs by according to an eviction policy that scores each KV, as illustrated in \cref{fig:block-processing}. Unlike previous approaches that apply eviction after processing the entire prompt, this strategy satisfies memory constraints throughout the full inference process. However, we observe a degradation in accuracy when applying existing eviction methods in this setting (\cref{table:longbench_reduced}).

We hypothesize that the performance drop stems from a mismatch in design: existing eviction methods assume access to full-prompt attention, where key importance is computed over the entire input. During block prompt processing, however, attention is computed using only the current block’s tokens without access to future blocks. As a result, attention scores based on a limited context often fail to reflect a token’s true importance across the full prompt. 

To this end, we observe that keys with lower average pairwise cosine similarity tend to receive higher attention scores across a variety of inputs. This suggests that key diversity serves as a strong proxy for global token importance, even without access to future tokens. This insight enables an attention-free approach to cache eviction that is based on the geometry of the cached keys.

Motivated by these observations, we propose \ourmethod{}, an attention-free cache eviction method that removes redundancy among cached keys, operates effectively during block-wise inference, and avoids excessive memory overhead. Our contributions are summarized as follows:
\begin{itemize}[leftmargin=*]
    \item \textbf{Insight.} We observe that lower pairwise cosine similarity among keys correlates with higher attention scores, suggesting its utility as a proxy for token importance. (\cref{subsec:attention_and_dissimilarity})
    \item \textbf{Method.} We introduce \ourmethod{}, an eviction strategy that selects keys based on their similarity to other cached entries without relying on attention scores or future tokens. (\cref{subsec:keydiff})    
    % \item \textbf{Theory.} In \cref{lemma:key-cossim-reduced} and \cref{theorem:keydiff-key-query} we provide a rigorous description of key and query geometry and leverage this to compute importance through \ourmethod{} for eviction. We also show that \ourmethod{} solves an optimal subset selection problem that maximizes key diversity. (\cref{subsec:keydiff-theory})
    \item \textbf{Theory.} Through our analysis of key and query geometry we provide a theoretical understanding how/why \ourmethod{} works. We also show that \ourmethod{} solves an optimal subset selection problem that maximizes key diversity. (\cref{subsec:keydiff-theory})

    %We show that \ourmethod{} solves an optimal subset selection problem that maximizes key diversity and provide a theoretical analysis that connects cosine similarity to the standard attention operator. (\cref{subsec:keydiff-theory})

    \item \textbf{Performance.} \ourmethod{} achieves $\leq1.5\%$ and $\leq0.04\%$ accuracy drop on LongBench with 6K and 8K cache budgets, respectively, outperforming state-of-the-art eviction methods across Llama and Qwen models (\cref{subsec:longbench}), and near non-evicting baseline performance for Deepseek-R1-Distill-Llama-8B on the Math-500 reasoning benchmark. (\cref{subsec:reasoning})
    \item \textbf{Efficiency.} We observe up to $30\%$ end-to-end inference latency reduction using \ourmethod{} compared to existing KV cache eviction methods. (\cref{subsec:inf-latency})
\end{itemize}

%% file: contents/02_background.tex
\section{Background}

\begin{figure*}
    \centering
    \includegraphics[width=.85\linewidth]{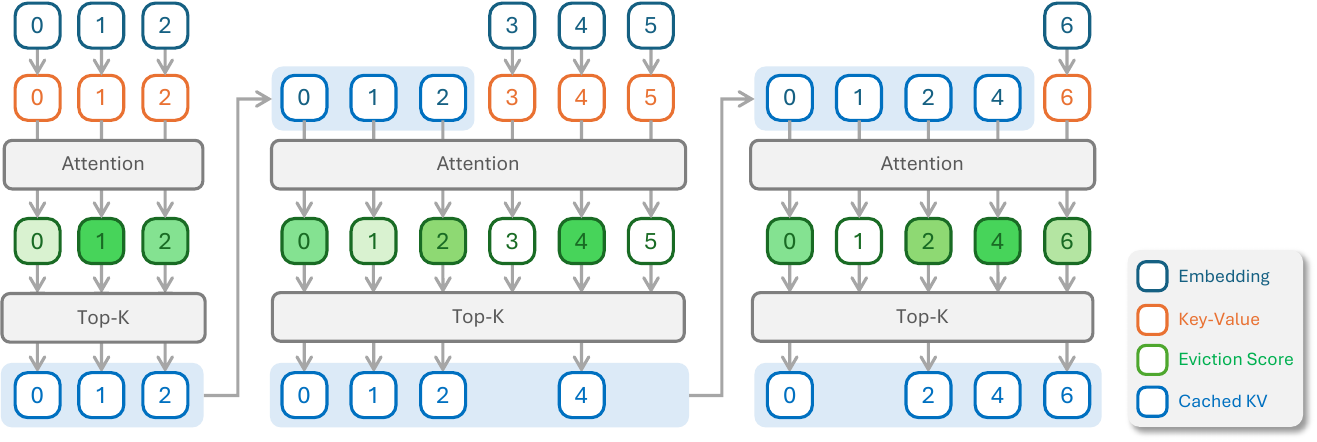}    
    \caption{\textbf{An example of block prompt processing with KV cache eviction.} The input prompt having length of 7 is segmented by three blocks, and a transformer layer in LLM processes each block by \textbf{(1)} computing key-value states from inputs, \textbf{(2)} computing attention, \textbf{(3)} computing the eviction score, and \textbf{(4)} performing eviction based on the eviction score to satisfy the memory constraints (e.g., at most 4 tokens can reside in the cache). After each block processing, the KV cache is updated and passed to the next round of block processing, satisfying imposed memory constraints on the KV cache.}
    \label{fig:block-processing}
    \vspace{-5mm}
\end{figure*}

\label{sec:background}
\subsection{Transformers}
\label{subsec:transformers}

The Transformer architecture \citep{vaswani2017attention} processes input data using a sequence of transformer blocks.
A transformer block $f$ takes a sequence $X = (x_1, x_2, \ldots, x_T) \in \mathbb{R}^{T \times d}$ as input and applies the causal self-attention operator $\operatorname{Attention}$ followed by a feed-forward network $\operatorname{FF}$ with optional gating \cite{shazeer2020glu} 
to produce the output $X' = (x'_1, x'_2, \ldots, x'_T) \in \mathbb{R}^{T \times d}$:
\begin{equation}    
    X' = f(X) = \operatorname{FF}(\operatorname{Attention}(X)),    
\label{eq:transformer_block}
\end{equation}

The causal $\operatorname{Attention}$ operator projects each input token $x_t$ with matrices $W_q, W_k, W_v \in \mathbb{R}^{d \times d}$ into key, query, and value matrices ($K = XW_k$, $Q = XW_Q$, $V=XW_V$, respectively) then applies the following relation to produce the attention output \footnote{The multi-head extension and output projections are omitted for brevity.}:
\begin{equation}
    O^{\operatorname{attn}} = \operatorname{Softmax}\left(\Large \sfrac{Q K^{\top}}{\sqrt{d}} + M \right) V = AV
\label{eq:attention_output}
\end{equation}
where $O^{\operatorname{attn}} \in \mathbb{R}^{T\times d}$, and the causal attention mask $M$ is an upper triangular matrix with nonzero values of $-\infty$.

\subsection{KV Caching}
\label{sec:kv_caching}
When the $\operatorname{Attention}$ operator processes a new token $x_{T+1}$, it must also recompute the prior KV states for tokens $x_0, \hdots, x_T$. 
This can be avoided by storing previously computed KVs in a \textit{KV cache} $\mathcal{C} = (K, V)$ for later reuse and append the new KV corresponding $x_{T+1}$ to the cache.
We can apply \cref{eq:attention_output} to an existing KV cache $\mathcal{C}$ as follows:
\begin{equation}
    o^{\operatorname{attn}}_{T+1} = \operatorname{Softmax}\left( \Large \sfrac{q_{T+1} [K \| k_{T+1}]^{\top}}{\sqrt{d}} + M \right) [V \| v_{T+1}],
\label{eq:transformer_block_with_kv_cache}
\end{equation}
where $k_{T+1}$, $q_{T+1}$, $v_{T+1} $ are the key, query, and value states of $x_{T+1}$, and  $[X \| x_{T+1}]$ represents the concatenation of $x_{T+1}$ to an existing tensor $X$ along the time dimension, $M$ is the causal attention mask accounting for both the KV cache and $x_{T+1}$.

KV caching dramatically reduces the latency of $\operatorname{Attention}$ by only computing $k_{T+1}$, $q_{T+1}$, $v_{T+1}$ for each token $x_{T+1}$ and reusing the KVs in $\mathcal{C}$. 
However, the size of the KV cache increases linearly with the number of processed tokens and dominates the memory footprint in long-context applications \cite{yang2024pyramidinfer}.

\subsection{KV Cache Eviction Methods}
\label{subsec:kv_cache_eviction}
To limit the memory footprint of the KV cache, we fix a \textit{cache budget} $N$, which is the maximum number of tokens to be stored in the cache.
If a new KV is added to the cache and the updated cache size is greater than $N$, we must evict KVs from the cache until the cache budget is met.
The \textit{eviction policy} $\pi_N(\mathcal{C})$ 
evicts a subset of KVs from $\mathcal{C}$ and returns a new cache $\mathcal{C}^\prime$ containing at most $N$ KVs:
\begin{equation}
\begin{aligned}     
    \mathcal{C}&\gets 
    ([K \| k_{t+1}], [V \| v_{t+1}]) \\
    \mathcal{C}^\prime &\gets \pi_N(\mathcal{C})
\label{eq:eviction_policy}
\end{aligned}
\end{equation}

\paragraph{Attention-Based Eviction Policies}
Attention-based eviction policies $\pi_N^{\text{attn}}$ use aggregated attention values to rank each KVs' relative importance and keep the $N$ highest scoring KVs. 
For a given attention weight aggregation function $\phi$, 
the eviction policy $\pi_N^{\text{attn}}$ performs the following steps:
\begin{equation} 
    \begin{aligned}
    S &= \texttt{topk}(\phi(A),N) \\
    K^\prime &= \texttt{gather}(K, S),\quad
    V^{\prime} = \texttt{gather}(V, S)
    \end{aligned}
\label{eq:attn_based_eviction_policy}
\end{equation}
where $\texttt{topk}(x, N)$ returns the indices of $N$-largest values of  $x$ and $\texttt{gather}(X, S)$ gathers columns indexed by $S$.

Attention-based eviction methods prioritize KV pairs with higher attention scores to past tokens. This is problematic when applying block prompt processing: all input tokens are not simultaneously accessible within $\operatorname{Attention}$, only those in the current block and cache. This can result in an incorrect eviction decision.
Additionally, attention-based eviction often requires explicitly materializing $A$, which can be resource intensive.
We discuss the attention-based eviction policies further in \cref{sec:related_work}.

\subsection{KV Caching in Resource-Constrained Environments}
\label{subsec:llm_in_resource_restricted_setup}
Existing eviction policies like \citet{zhang2024h2o, oren2024transformers} focus on processing the entire input prompt \textit{at once}: KVs are computed for each token in the prompt and stored in a cache $\mathcal{C}$, then the eviction policy \( \pi_N \) is applied to reduce the number of tokens in  \( \mathcal{C}^\prime \) to $N$, before token generation. 
However, the intermediate cache $\mathcal{C}$ before eviction will grow to the size of the input prompt.
This can often exceed model's allocated memory limit when deploying long context applications in resource constrained environments.

As demonstrated in efficient LLM inference frameworks \cite{agrawal2023sarathi, holmes2024deepspeed, kwon2023efficient, xu2024empowering}, one solution is to apply $\pi_N$ more frequently by segmenting $X$ into non-overlapping blocks $X = [X_0, X_1, \ldots, X_{m-1}]$, where $X_i=[x_{Bi}, \hdots, x_{B(i+1)-1}]$, $B$ is the block size, and $m = \lceil T / B \rceil$, and iteratively updating the cache by exploiting causality, applying \cref{eq:eviction_policy} in a block-wise fashion:
\begin{equation}
    \begin{aligned} 
    \mathcal{C}_{i}&\gets 
    ([K_{i-1} \| k_{Bi:B(i+1)-1}], [V_{i-1} \| v_{Bi:B(i+1)-1}]) \\
    \mathcal{C}_{i}^\prime & \gets \pi_N(\mathcal{C}_{i}),\quad \mathcal{C}_i \gets \mathcal{C}_{i}^\prime, \quad
    \mathcal{C}_0 = \emptyset,
    \end{aligned}
    \label{eq:block_processing}
\end{equation}
where $k_{Bi:B(i+1)-1}$ and $v_{Bi:B(i+1)-1}$ are the keys and values selected from $X_i$ respectively, and $ \mathcal{C}_{i} =(K_{i}, V_{i})$ is the KV cache after processing the first $i$ prompt blocks. 
As in \cref{eq:eviction_policy}, we concatenate the $B$ new KVs to the current cache, apply 
% the eviction policy 
$\pi_N$ and update the cache in \cref{eq:block_processing}.

We refer to this as \textit{block prompt processing}.
Its main advantage is the ability to control of the compute and memory overhead of KV cache management by adjusting the block size $B$ and cache budget $N$.
Note that, in a decoder-based architecture, applying block prompt processing to $X$ with $B=T$ yields the same result as processing all of $X$ at once and choosing $B=1$ corresponds to the token generation phase of LLM evaluation.

\paragraph{Attention-Based Token Eviction Challenges} Despite its advantages, block prompt processing introduces a challenge for KV cache eviction: eviction decisions in block $X_i$ impact the cache used by $X_{i+1}$, causing eviction errors to compound over time. 
When the model processes $X_i$, attention-based eviction methods retain KVs with high attention weights derived from $X_0, \hdots, X_i$ rather than all of $X$, which may prematurely evict KVs with high weights in upcoming blocks.

%% file: contents/03_method.tex
\begin{figure*}[t!]
     \centering
     \begin{subfigure}[b]{0.49\linewidth}
         \centering
         \includegraphics[width=\linewidth]{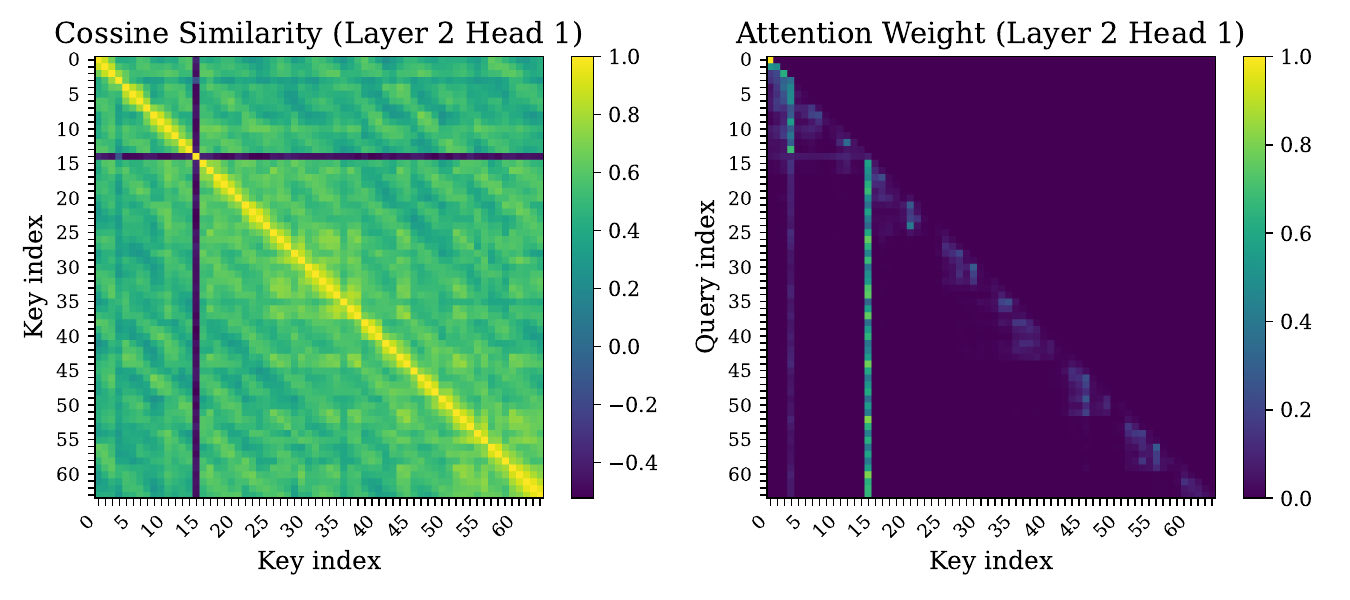}
         \caption{Layer 2 and Head 1}
         \label{fig:cossim-attn-l2-h1}
     \end{subfigure}
     \hfill
     \begin{subfigure}[b]{0.49\linewidth}
         \centering
         \includegraphics[width=\linewidth]{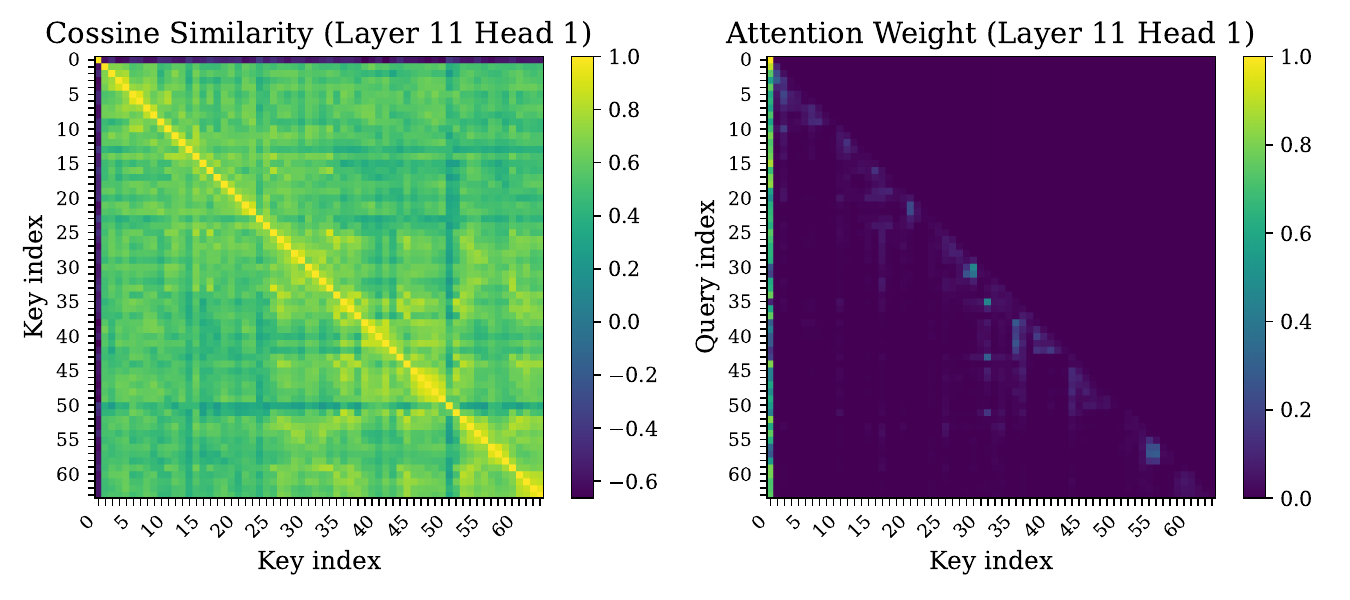}
         \caption{Layer 11 and Head 1}
         \label{fig:cossim-attn-l11-h1}
     \end{subfigure}
     \caption{\textbf{Cosine similarity of the keys and attention weights.} Measured from Llama 3.2-3B-Instruct and the first sample from the NarrativeQA dataset in LongBench. Truncated to the first 64 tokens for visualization.}
    \label{fig:cossim-attn}
    \vspace{-5mm}
\end{figure*}

\section{Method}
\label{sec:method}
We demonstrate a negative correlation between attention scores and the cosine similarity among keys (\cref{subsec:attention_and_dissimilarity}) and leverage this observation to develop \ourmethod{} (\cref{subsec:keydiff}), followed by a theoretical justification of \ourmethod{} (\cref{subsec:keydiff-theory}) and preliminary evidence of its efficacy (\cref{subsec:opt_deriv_keydiff}).

\subsection{Correlation of Attention Scores and Key Dissimilarity}
\label{subsec:attention_and_dissimilarity}

To address the shortcomings of attention-based KV cache eviction in \cref{subsec:llm_in_resource_restricted_setup}, we develop an alternative attention-free scoring metric that retains significant KVs across blocks while being resource efficient. 
We recall the ``attention sink" phenomenon: LLMs often assign high attention weight to the first few tokens, regardless of the input  \cite{xiaoefficient, sun2024massive}; these highly weighted tokens are called \textit{sink tokens}. However, the index of the sink tokens can vary across heads and layers and reside deeper in the sequence than the first few tokens. 
This observation motivates the following hypothesis: \textit{high attention scores can be determined by the intrinsic properties of the keys rather than by any particular combination of keys and queries.} 

\paragraph{Correlation of Key Similarity and Attention Scores}
We evaluate our hypothesis by inspecting the cosine similarities between keys computed inside an attention block.
We visualize the pairwise cosine similarities between keys along with the attention weights in two particular heads and layers in \cref{fig:cossim-attn}. 
We observe that keys with lower cosine similarity with other keys exhibit higher relative attention scores regardless of the choice of query, such as the 4th and 15th keys in  \cref{fig:cossim-attn-l2-h1}, or the 1st key in \cref{fig:cossim-attn-l11-h1}. Pairwise cosine similarity of keys is solely a function of the keys in the cache, which are independent of input queries; the surprising aspect of \cref{fig:cossim-attn} is the negative correlation with attention weights. These distinctive keys essentially recover the attention sink phenomenon \cite{xiaoefficient}.

\subsection{\ourmethod{}}
\label{subsec:keydiff}

Based on the observation in \cref{subsec:attention_and_dissimilarity}, we propose \ourmethod{}, which evicts tokens from the KV cache based on key similarity. 
If the cache $\mathcal{C}$ has intermediate size $n$ and budget $N$ where $n > N$,  $\pi^{\ourmethod{}}_N$ is defined as:
\begin{equation} 
    \begin{aligned}
    S &= \texttt{topk}(-\operatorname{CosSim}(K) \mathbf{1}, N), \\
    K^{\prime} &= \texttt{gather}(K, S),\quad
    V^{\prime} = \texttt{gather}(V, S)
    \end{aligned}
\label{eq:keydiff}
\end{equation}
where $K \in \mathbb{R}^{n \times d}$ and \( V \in \mathbb{R}^{n \times d} \) are the cached keys and values, \( \operatorname{CosSim}(K) \in \mathbb{R}^{n \times n} \) is the pairwise cosine similarity matrix of keys in \( K \) with $\operatorname{CosSim}(K)_{ij} = \frac{ k_i \cdot k_j }{\|k_i\| \|k_j\|}$, and \( \mathbf{1} \in \mathbb{R}^n \) is a vector of ones.

\begin{wrapfigure}{r}{.45\linewidth}
    \includegraphics[width=0.9\linewidth]{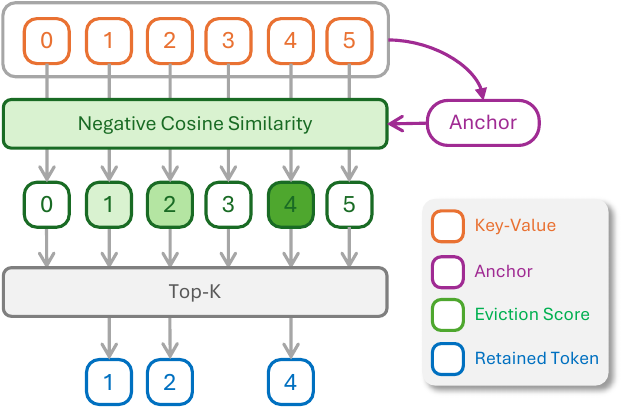}
    \caption{\textbf{An overview of \ourmethod{}}. \textbf{(1)} \ourmethod{} first computes the anchor vector by taking the average of the keys in the KV cache, \textbf{(2)} computes the cosine similarity between the keys and the anchor resulting in eviction scores whose color intensities indicate the score values, and \textbf{(3)} retains the KV pairs with the lowest similarities.}
    \label{fig:keydiff}
    \vspace{-2mm}
\end{wrapfigure}

\paragraph{Efficient Variant of \ourmethod{}} 
\label{subsec:keydiff-eff}
Unlike attention-based eviction policies, \ourmethod{} does not require access to the attention weights $A$, facilitating optimized attention kernels that do not materialize $A$ such as FlashAttention \citep{dao2022flashattention}. However, computing the pairwise cosine similarities runs in \( \mathcal{O}(n^2) \) time. Fortunately, we can compute the score of each token in \cref{eq:keydiff} in $\mathcal{O}(n)$ as follows:
\begin{equation} 
    \begin{aligned}
    S &= \texttt{topk}(-\operatorname{CosSim}(\mu(\hat{K}), \hat{k}_i), N) \\    
    \end{aligned}
\label{eq:efficient-keydiff}
\end{equation}
where $\mu(\hat{K}) = \frac{1}{n} \sum_{i=1}^n \hat{k}_i$ and $\hat{k}_i = \frac{k_i}{||k_i||}$. We refer to $\mu(\hat{K})$ as the \emph{anchor vector}. We show this formulation retains the same KVs of \cref{eq:keydiff} under a mild condition. (see \cref{subsec:opt_deriv_keydiff}). Our experimentation has shown that the anchor vector $\mu(\hat{K})$ can be replaced with $\mu(K)$ without losing accuracy (see \cref{table:anchor_vector_ablation}). We evaluate the efficient \ourmethod{} described in \cref{fig:keydiff} using unnormalized keys $k$ in all subsequent sections. \cref{fig:pca-plot-5-3,fig:pca-plot-27-4,fig:pca-plot-20-0,fig:pca-plot-8-1} visualize the keys retained and evicted by sink attention \cite{xiaoefficient}, TOVA \cite{oren2024transformers} and \ourmethod{} via PCA. \ourmethod{} retains more varied keys. A full complexity and FLOP analysis can be found in \cref{appendix:subsec:runtime-and-memory-complexity,sec:keydiff-flop-count}

\paragraph{\ourmethod{} with Sliding Window} \label{subsec:keydiff-sw}
In tasks such as reasoning and coding, where the most recent tokens are often important, we can augment \ourmethod{} and its efficient variant to use a percentage of the cache budget for a \textit{sliding window} \cite{beltagy2020longformer}, which we call \textit{\ourmethod{} with sliding window}. 
This extension introduces no complexity or memory overhead and we observe better results on certain tasks than vanilla \ourmethod{} (\cref{table:longbench_full_rotated_sliding_window,appendix:math500}).

\subsection{Why \ourmethod{} Works: A Theoretical Perspective}
\label{subsec:keydiff-theory}

To solidify the theoretical foundation of \ourmethod{} and show that \ourmethod{} ultimately selects keys most aligned with queries, we prove the following two results. 
We first validate the relationship between cosine similarity and attention scores observed in \cref{fig:cossim-attn} by bounding the attention score of a new incoming key $k^\star$ in terms cosine similarity with a fixed query $q$:
\begin{lemma}
\label{lemma:key-cossim-reduced}
Suppose that for a fixed query token $q$, there is a set of key tokens $\{k_i\}_{i=1}^n$ such that $||k_i||_2^2 < M, \; \forall \;i$. Without loss of generality suppose $||q||=1$ and assume ${k^*}$ is a key not in $\{k_i\}_{i=1}^n$ with $||{k^*}||_2^2 < M$ that has attention weight $w > 0$. Then, for $n\to \infty$, 
\[\frac{-\log(1-w)}{2M} -1  \leq \textrm{CosSim}({k^*}, q) \]
\end{lemma}

We then establish a relationship between the cosine similarities of $k^\star$, $q$, and the mean of prior keys $\bar{k}$:
\begin{theorem}\label{theorem:keydiff-key-query}
Consider tokens ${k^*}$, $q$ as above, and the average of the keys tokens $\bar{k}$. Suppose $\textrm{CosSim}({k^*}, q) = \beta_q > 0$ and $\textrm{CosSim}(\bar{k}, q) = \alpha_q < 0$. Then 
\begin{equation}
\textrm{CosSim}(\bar{k}, {k^*}) \leq 1 + \alpha_q \beta_q - 0.5\alpha_q^2 -0.5\beta_q^2.
\label{eq:keydiff-key-query-eq}
\end{equation}

\end{theorem}

\begin{wrapfigure}{r}{.45\linewidth}
    \includegraphics[width=0.8\linewidth]{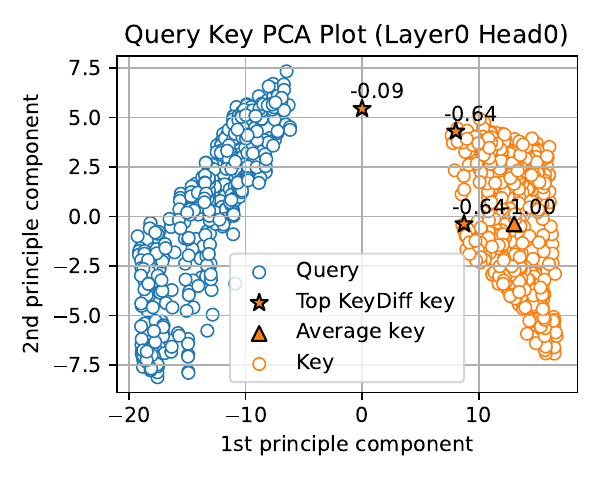}
    \caption{PCA embedding of keys and queries from Llama 3.2 3B 
    }
    \label{fig:key-query-pca}
\end{wrapfigure}

By combining \cref{lemma:key-cossim-reduced,theorem:keydiff-key-query},  
we establish a relationship between the attention weight $w$ and the \ourmethod{} score $\mathrm{CosSim}(\bar{k}, k^*)$.
As $\mathrm{CosSim}(\bar{k}, q)$ decreases and $\mathrm{CosSim}(k^*, q)$ increases (along with the attention weight $w$), then $\mathrm{CosSim}(\bar{k}, k^*)$ tends to $-1$: this means \ourmethod{} selects distinct keys most aligned with $q$. We visualize this in \cref{fig:key-query-pca} with a PCA embedding of keys and queries from a single head of Llama 3.2 3B, highlighting the relationship between top scoring keys via \ourmethod{}, the anchor vector and queries. Similar trends are found from the other layers and heads as shown in \cref{fig:pca-plot}. The proofs of \cref{lemma:key-cossim-reduced,theorem:keydiff-key-query} are in \cref{appendix:theoretical-justification}, along with empirical motivation for the chosen assumptions.

\begin{figure*}[t!]
     \centering
     \begin{subfigure}[b]{0.24\linewidth}
         \centering
         \includegraphics[width=\linewidth]{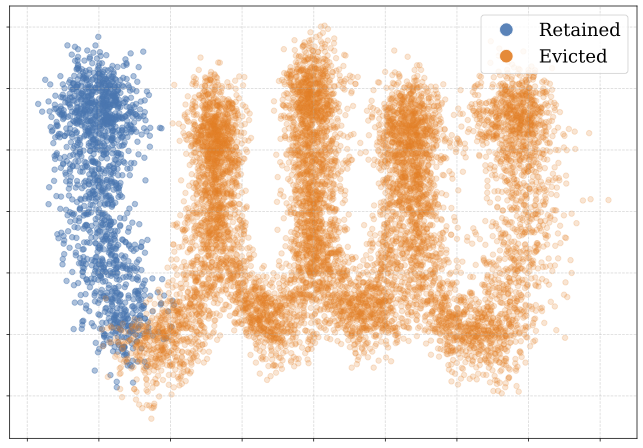}
         \caption{Sink Attention}
     \end{subfigure}
     \begin{subfigure}[b]{0.24\linewidth}
         \centering
         \includegraphics[width=\linewidth]{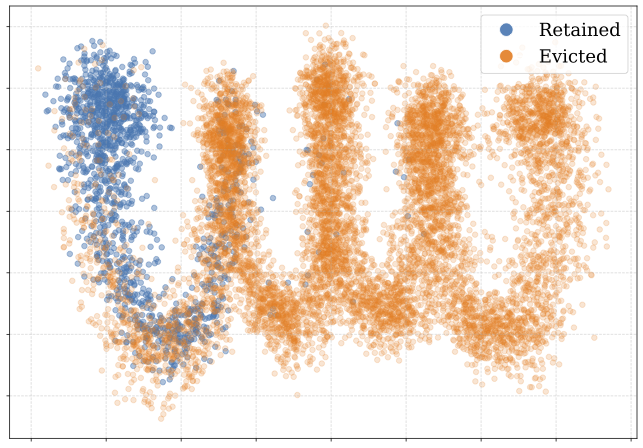}
         \caption{TOVA}
     \end{subfigure}
     \begin{subfigure}[b]{0.24\linewidth}
         \centering
         \includegraphics[width=\linewidth]{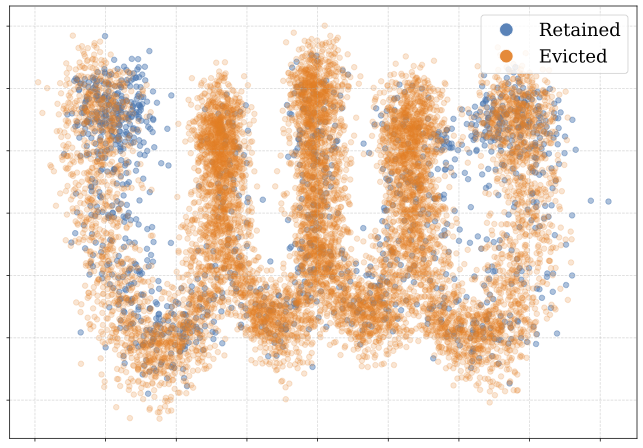}
         \caption{\ourmethod{}}
     \end{subfigure}
     \begin{subfigure}[b]{0.24\linewidth}
         \centering
         \includegraphics[width=\linewidth]{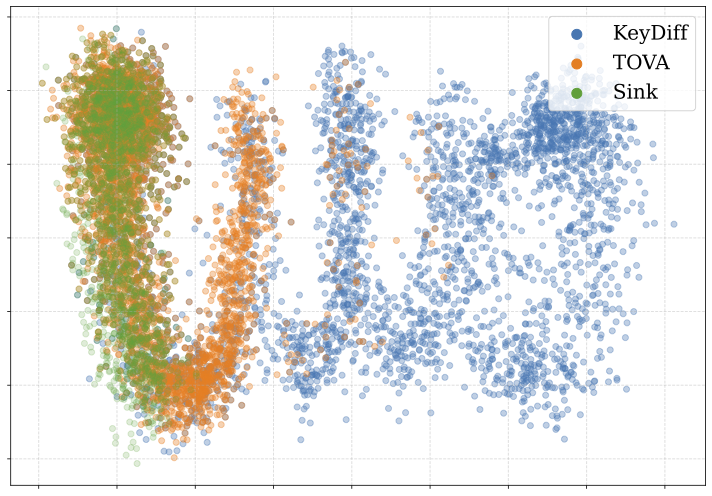}
         \caption{Retained keys only}
     \end{subfigure}
     \hfill
     \caption{\textbf{(a, b, and c)} PCA Visualizations in two dimensions of a key cache managed with Sink, TOVA, and \ourmethod{}. Retained tokens are \textcolor{blue}{blue}, while evicted tokens are \textcolor{orange}{orange}.
     Keys are taken from layer $5$ and head $3$ of Llama3.2-3B-Instruct, and generated using the NarrativeQA dataset. \textbf{(d)} PCA visualization of the retained keys for each KV cache eviction method.
    }
    \label{fig:pca-plot-5-3}
\vspace{-2mm}
\end{figure*}

%% file: contents/04_experiment.tex
\section{Experiments}

\label{sec:experiment}
In this section, we empirically demonstrate the effectiveness of \ourmethod{}. 
We begin with a description of competing, state-of-the-art eviction methods, followed by a detailed description of the evaluation setup, then present our experimental results.
 Our findings can be summarized as follows:
\begin{itemize}[leftmargin=*, itemsep=.02em, topsep=0.5em]    
    \item \textbf{Needle-In-a-Haystack.} \ourmethod{} outperforms competing eviction policies on the Needle-In-A-Haystack benchmark (\cref{subsec:nih}).
    \item \textbf{LongBench.} \ourmethod{} outperforms competing eviction policies with block size $B=128$ on LongBench, achieving an 1.5\% accuracy drop with a 6K cache budget ($\sim$33\% compression rate) and $\leq$ .04\% with a 8k cache budget ($\sim$23\% compression rate) with Llama-3.1-8B-Instruct and Llama-3.2-3B-Instruct (\cref{subsec:longbench}).         
    \item \textbf{Reasoning.} \ourmethod{} performs competitively on the Math-500 reasoning benchmark with other eviction methods using the DeepSeek-R1-Distill-Qwen-7B and Llama-8B, and shows near eviction-free baseline performance when augmented with a sliding window (\cref{subsec:reasoning}) for DeepSeek-R1-Distill-Llama-8B.
    \item \textbf{Ablation Study.} We perform an ablation study on the main parameters of \ourmethod{} and show that utilizing negative cosine similarity as the eviction criteria and the mean of cached keys as the anchor vector  performs best. (\cref{subsec:ablation}).
    \item \textbf{Efficiency.} We compare the end-to-end inference latency of \ourmethod{}, \cite{li2024snapkv} and \cite{oren2024transformers} and observe a 30\% latency improvement with \ourmethod{} (\cref{subsec:inf-latency}).
\end{itemize}
\vspace{-3mm}

\paragraph{Experimental Setup} We apply several cache eviction methods to several decoder-only transformer-based language models, including Llama $3.1$-$8$B-Instruct \cite{dubey2024llama}, Llama $3.2$-$3$B-Instruct \cite{dubey2024llama}, and Qwen $2.5$-$3$B/7B-Instruct \cite{yang2024qwen2technicalreport}. 
We evaluate these models using H2O \cite{zhang2024h2o}, TOVA \cite{oren2024transformers}, SnapKV \cite{li2024snapkv}, and StreamingLLM \cite{xiaoefficient}, (or ``sink attention") cache eviction policies, along with the eviction-free model as a baseline.
We simulate a resource constrained environment by processing prompts and generating responses using \cref{eq:block_processing}, with a block size of $B=128$ for prompt processing and $B=1$ for token generation using greedy decoding for all experiments.
We denote the cache budgets of 2048, 4096, 6144 and 8192 as 2K, 4K, 6K and 8K, respectively.

\begin{figure*}[t!]
     \centering
     \begin{subfigure}[b]{0.32\linewidth}
         \centering
         \includegraphics[width=\linewidth]{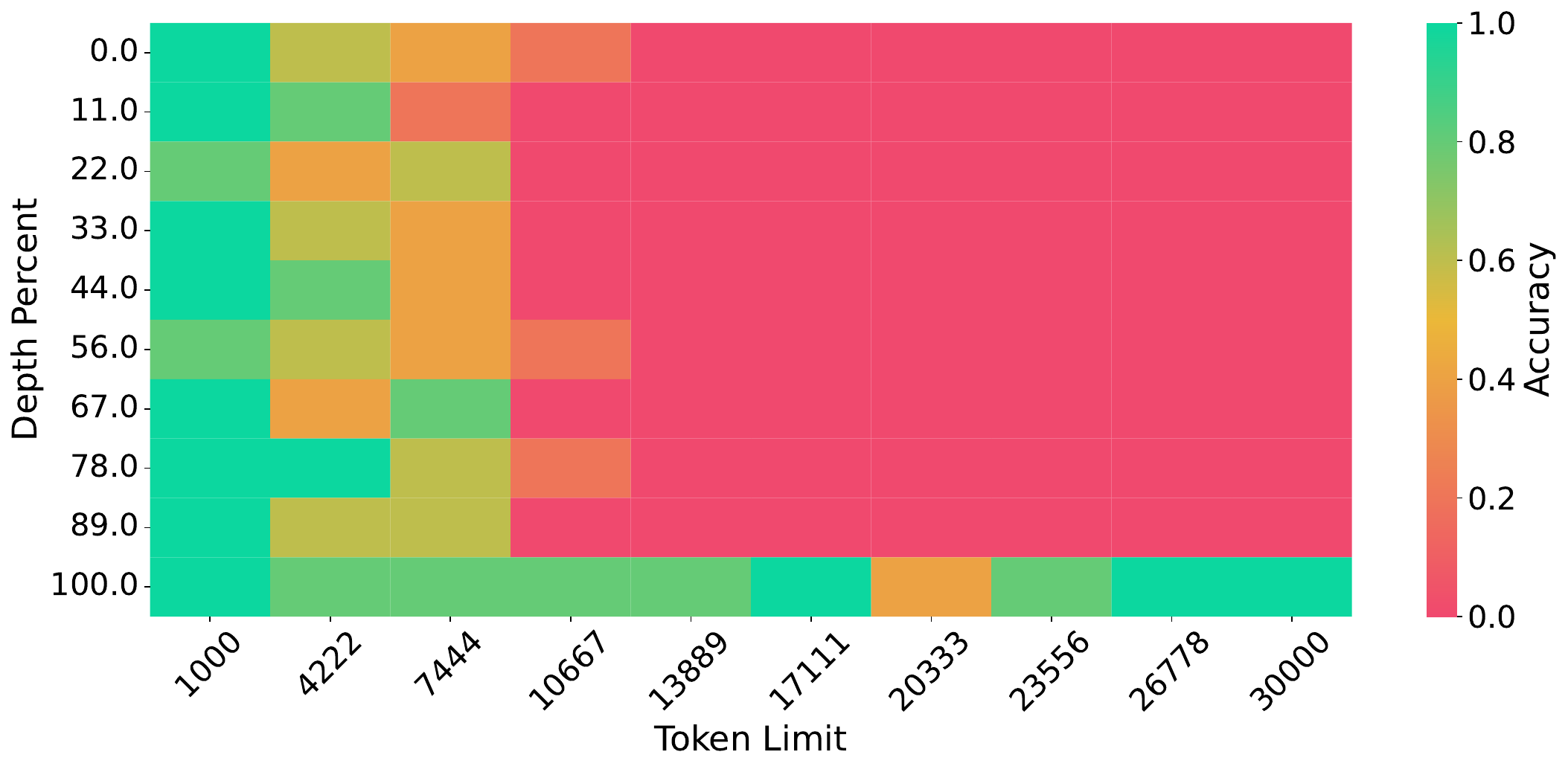}
         \caption{TOVA}
     \end{subfigure}
     \hfill
     \begin{subfigure}[b]{0.32\linewidth}
         \centering
         \includegraphics[width=\linewidth]{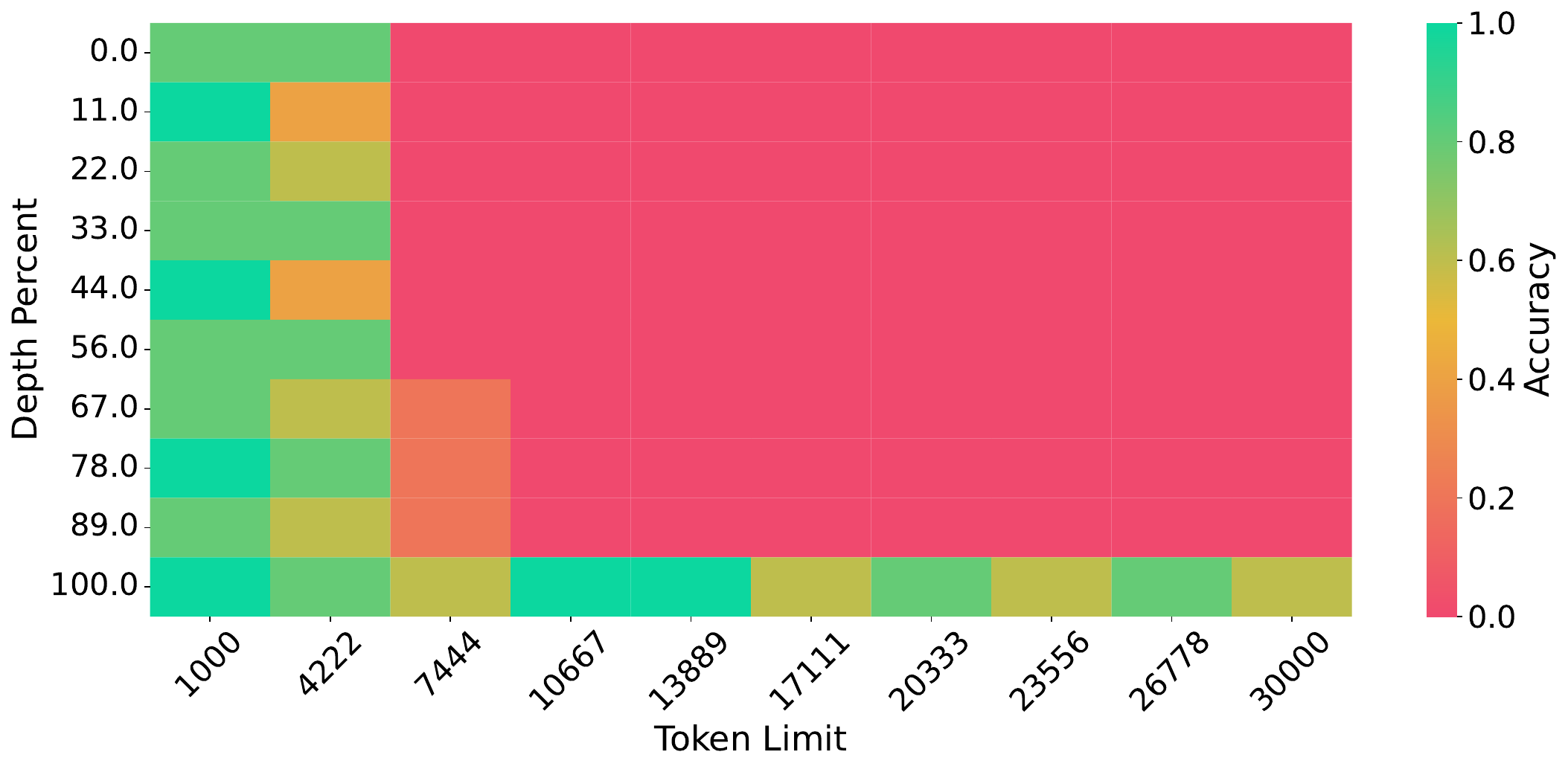}
         \caption{SnapKV}
     \end{subfigure}
     \begin{subfigure}[b]{0.32\linewidth}
         \centering
         \includegraphics[width=\linewidth]{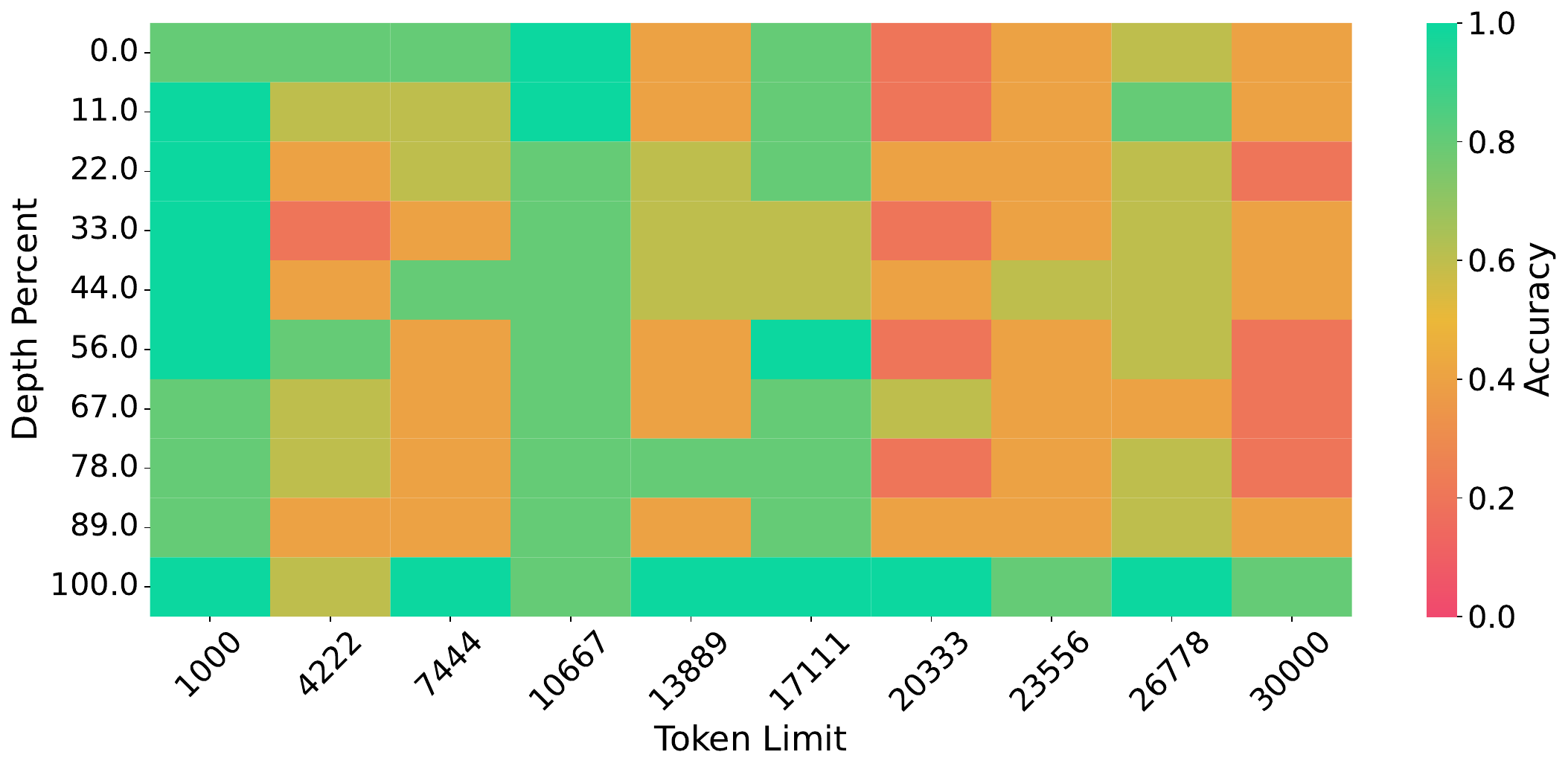}
         \caption{\ourmethod{}}
     \end{subfigure}
     \caption{Accuracy across document length and needle depth  for needle in a haystack test. Cache size is $6$K with $B=128$. }
    \label{fig:needle_haystack}
    \vspace{-3mm}
\end{figure*}

\subsection{Needle In a Haystack}
\label{subsec:nih}
 To compare the impact of various cache eviction policies on fact retrieval, we conduct the ``Needle In a Haystack" test \cite{liu2024lost,kamradt2023needle}.
This test embeds specific information (``needle") at different points within a body of unrelated text (``haystack"); finding and retaining the needle is challenging for eviction policies, which can't know what information must be retained during block prompt processing.
The results are shown in \cref{fig:needle_haystack} and \cref{fig:needle-in-a-haystack-sink}, where we show the recall accuracy of Llama$3.2$-$3$B-Instruct across different document lengths (x-axis) and needle depths (y-axis) with a cache size of $6$K.
\ourmethod{} performs similarly to TOVA, SnapKV and sink attention for shorter documents and outperforms all three methods as the document length increases.

\input{tables/longbench_reduced}
\subsection{LongBench}
\label{subsec:longbench}
LongBench \cite{bai-etal-2024-longbench} is a bilingual, multi-task benchmark suite for LLMs, providing a comprehensive stress test for long prompt inputs. 
LongBench is useful for evaluating cache eviction methods in a resource constrained environments with a fixed memory budget: 51\% of prompts are longer than the largest KV cache size of 8K. For cache budgets of 6k and 8k tokens, prompts in LongBench are compressed by 33\% and 23\% respectively on average, (see \cref{appendix:subsec:longbench-statistics} for more detail.)

\cref{table:longbench_reduced} summarizes the evaluation results of Llama 3.1-8B-Instruct and Llama 3.2-3B-Instruct on the English subset of LongBench with 2K, 4K, 6K, and 8K cache budgets using various eviction policies with block prompt processing enabled with $B=128$. As shown in \cref{table:longbench_reduced}, \ourmethod{} outperforms other eviction strategies across most tasks, even demonstrating better performance with smaller cache budgets.
\ourmethod{} shows significant 
a improvement on the PassageRetrieval-en (PR-en) dataset, which tests whether long-term dependencies within a long prompt can be correctly recognized \citep{bai-etal-2024-longbench}, while achieving near full-context model performance even with the smallest budget. Adding a sliding window to \ourmethod{} improves coding task performance (\cref{table:longbench_full_rotated_sliding_window}).
We observed similar trends in the full LongBench task suite as shown in \cref{table:longbench_full_rotated} and in the additional results in \cref{appendix:longbench-extended}.

\ourmethod{} exhibits similar or better performance compared to competing methods. 
Notably, the attention-based methods (e.g., H2O, TOVA, and SnapKV) show significant performance improvements over the $B=128$ case.
This result supports our hypothesis: an eviction scheme robust to changes in the scope of comparison among tokens is essential in memory constrained environments where token-wise attention weight can't be fully materialized.

\paragraph{Additional Results} We present the full evaluation results in \cref{table:longbench_full_rotated} and more complete comparisons on LongBench in \cref{appendix:longbench-extended}, such as: standard prompt processing with a single large block (i.e. $B=\infty$) in \cref{table:longbench_unrestricted_reduced}; eviction method performance with Qwen 2.5-3B/7B-Instruct in \cref{table:longbench_full_rotated_qwen}; performance behavior with block sizes $B=[64, 256]$ in \cref{table:longbench_full_rotated_block_64_256}; and performance on \ourmethod{} combined with a sliding window as described in \cref{subsec:keydiff-sw}. We also compare against the $L_2$-norm minimizing eviction method of \cite{devoto2024simple} in \cref{table:longbench_full_rotated}.

\subsection{Math-500 Reasoning Benchmark} 
\label{subsec:reasoning}
Reasoning is an important long-context task for LLMs. Unlike other long-context use cases, reasoning typically involves a relatively short prompt followed by a long generation, which presents unique challenges for token eviction methods. To evaluate the effectiveness of token eviction methods, we apply \ourmethod{} and SnapKV to the DeepSeek-R1-Distill-Qwen-7B and Llama-8B distilled models \cite{guo2025deepseek}, and assess their performance on the Math-500 reasoning benchmark \cite{hendrycks2021measuring}. Surprisingly, we found that Llama equipped with \ourmethod{} and a moderate KV cache budget performs comparably to, or slightly better than, the eviction-free baseline, while also outperforming SnapKV. We kindly refer the reader to \cref{appendix:math500} for additional details on the reasoning task evaluation.

\subsection{Ablation Study}
\label{subsec:ablation}

We evaluate the design choices of \ourmethod{}, including the similarity metrics and the choice of the anchor vector, and validate the efficacy of \ourmethod{}.
We provide a full description of the test setup in \cref{appendix:ablation-study} and summarize the findings here:
\begin{itemize}[leftmargin=*]
    \item \ourmethod{} anchor choice does not greatly impact benchmark accuracies (See \cref{table:anchor_vector_ablation}).
    \item \ourmethod{} using cosine similarity as the distance metric outperforms other metrics (See \cref{table:similarity_metric_ablation})
\end{itemize}

\subsection{Latency and Complexity}
\label{subsec:inf-latency}
Additionally, in order to demonstrate that \ourmethod{} decreases end-to-end inference latency, we measured time to first token for the Llama 3.2 3B instruct model using different block prompt processing sizes and cache strategies. These results are visualized in \cref{fig:ttft_eager,fig:ttft_flash}. 
Since \ourmethod{} does not require attention weight materialization, FlashAttention \cite{dao2022flashattention} can be used, resulting in up to 30\% lower latency than TOVA and SnapKV.
We compare the complexity of \ourmethod{} with competitors in \cref{appendix:subsec:runtime-and-memory-complexity} and perform a complete FLOP count in \cref{sec:keydiff-flop-count}.

\begin{figure*}[t!]
     \centering
     \begin{subfigure}[b]{0.99\linewidth}
         \centering
         \includegraphics[width=\linewidth]{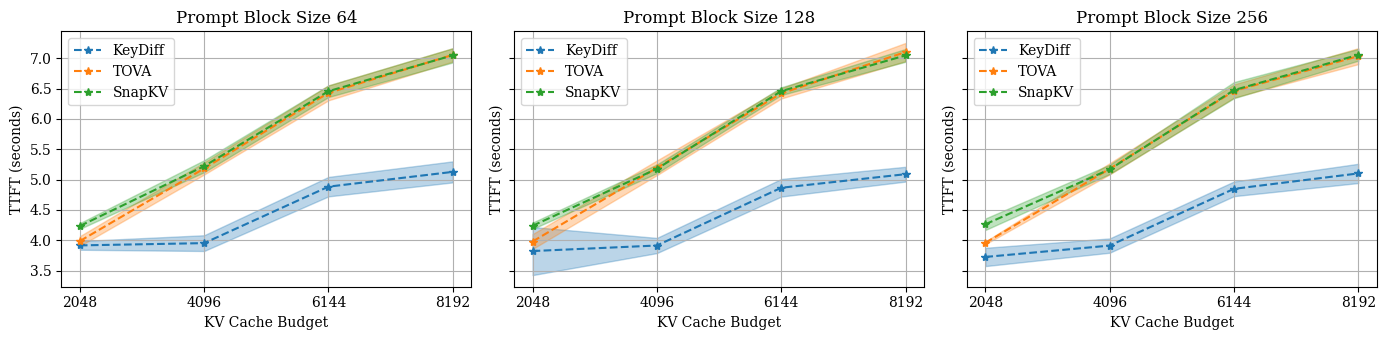}
     \end{subfigure}
     \hfill
     \vspace{-2mm}
     \caption{Time-to-first-token (TTFT) for Llama 3.2-3B using Flash Attention with different eviction strategies with block prompt processing sizes $64$, $128$, and $256$.}
    \label{fig:ttft_flash}
    \vspace{-5mm}
\end{figure*}

%% file: tables/longbench_reduced.tex
\begin{table*}[!t]
\caption{\textbf{Llama-3.1-8B/3.2-3B-Instruct LongBench results with $B=128$ (Higher is better)}. We highlight the best and second best methods within a given budget with \textbf{bold} and \underline{underline}. We omit Chinese dataset results and other model results due to space limit. The full evaluation results are in \cref{table:longbench_full_rotated}. \textdagger: A subset of samples (183/200) were evaluated due to OOM errors.} %(183/200 samples are evaluated).}
\label{table:longbench_reduced}
% \vspace{-5mm}
\begin{center}
\resizebox{\textwidth}{!}{
\begin{tabular}{ccccccccccccccccccc}
\toprule

%%%%%%%%%%%%%%%%%%%%%%%%%%%%%%%%%%%%%%%%%%%%%%%%%%%%%%%%
%%%%%%%%%%%%%%%%%%%%% Header Rows %%%%%%%%%%%%%%%%%%%%%&
%%%%%%%%%%%%%%%%%%%%%%%%%%%%%%%%%%%%%%%%%%%%%%%%%%%%%%%%
&
&
\multicolumn{3}{c}{Single Doc. QA} &
\multicolumn{3}{c}{Multi Doc. QA} &
\multicolumn{3}{c}{Summarization} &
\multicolumn{3}{c}{Few$\-$shot Learning} &
\multicolumn{2}{c}{Synthetic} &
\multicolumn{2}{c}{Code} &
\\

\cmidrule(lr){3-5}
\cmidrule(lr){6-8}
\cmidrule(lr){9-11}
\cmidrule(lr){12-14}
\cmidrule(lr){15-16}
\cmidrule(lr){17-18}

&
& 
{Narrative QA} & {Qasper} & {MF-en} & 
{HotpotQA} & {2WikiMQA} & {Musique} & 
{GovReport} & {QMSum} & {MultiNews} & 
{TREC} & {TriviaQA} & {SAMSum} & 
{PCount} & {PR-en} &
{Lcc} & {RB-P} & 
{Avg.} 
\\

%%%%%%%%%%%%%%%%%%%%%%%%%%%%%%%%%%%%%%%%%%%%%%%%%%%%%%%%
%%%%%%%%%%%%%%%%% Llama 3.1-8B Results %%%%%%%%%%%%%%%%%
%%%%%%%%%%%%%%%%%%%%%%%%%%%%%%%%%%%%%%%%%%%%%%%%%%%%%%%%
\midrule
\midrule
\addlinespace

% Llama 3.1 8b Dense
\multicolumn{2}{c}{Llama3.1-8B} &
30.05\textsuperscript{\textdagger} & 47.00 & 56.12 & 57.33 & 47.81 & 32.25 & 34.86 & 25.32 & 27.02 & 73.00 & 91.61 & 43.37 & 8.33 & 99.50 & 61.66 & 51.94 & 49.20 \\
\midrule

% H2O
\multirow{4}{*}{H2O} &
2K & 1.74 & 21.15 & 25.33 & 26.11 & 24.15 & 8.78 & 2.17 & 2.70 & 16.78 & 44.00 & 29.36 & 7.62 & 2.25 & 5.88 & 40.15 & 12.14 & 16.89 \\

&
4K & 4.07 & 36.16 & 36.00 & 33.52 & 32.87 & 17.78 & 6.66 & 5.95 & 24.09 & 55.00 & 47.65 & 17.41 & 4.00 & 24.50 & 54.85 & 21.43 & 26.37 \\

&
6K & 8.52 & 43.31 & 44.80 & 40.03 & 42.46 & 21.68 & 11.85 & 8.78 & 26.03 & 62.00 & 56.39 & 25.72 & 5.75 & 45.50 & 58.62 & 29.53 & 33.19 \\

&
8K & 13.85 & 44.94 & 47.81 & 43.64 & 44.90 & 23.65 & 18.78 & 11.35 & 26.49 & 69.50 & 69.05 & 33.41 & 5.25 & 62.50 & 59.74 & 36.26 & 38.20 \\

\midrule

% TOVA
\multirow{4}{*}{TOVA} &
2K & 22.57 & 37.26 & 39.43 & 45.74 & 34.48 & 14.77 & 28.87 & 21.17 & 26.95 & 62.50 & 90.73 & 42.74 & 0.00 & 18.00 & 62.68 & 52.48 & 37.52 \\

&
4K & 22.68 & 44.55 & 47.87 & 46.76 & 44.54 & 20.56 & 30.95 & 22.13 & 26.96 & 61.50 & 90.56 & 43.27 & 3.00 & 43.50 & 61.62 & 53.40 & 41.49 \\

&
6K & 24.59 & 45.93 & 53.92 & 55.09 & 47.43 & 25.07 & 32.33 & 24.10 & 27.00 & 68.50 & 90.81 & 43.89 & 4.25 & 67.00 & 61.50 & 52.39 & 45.24 \\

&
8K & 24.86 & 46.78 & 54.83 & 54.52 & 49.00 & 26.40 & 33.44 & 24.76 & 27.00 & 71.00 & 91.11 & 43.29 & 6.25 & 87.00 & 61.49 & 51.79 & 47.09 \\

\midrule

% Sink
\multirow{4}{*}{Sink} &
2K & 21.83 & 34.27 & 29.24 & 38.64 & 29.50 & 12.59 & 28.51 & 20.21 & 26.62 & 65.00 & 89.46 & 42.20 & 2.00 & 25.50 & 64.95 & 59.54 & 36.88 \\

&
4K & 22.94 & 43.01 & 39.08 & 44.04 & 41.39 & 19.09 & 31.08 & 21.57 & 26.78 & 70.00 & 91.53 & 42.29 & 3.00 & 38.50 & 62.12 & 58.84 & 40.95 \\

&
6K & 25.41 & 47.40 & 44.13 & 47.39 & 45.73 & 21.90 & 32.53 & 22.19 & 26.87 & 72.00 & 91.25 & 43.41 & 3.08 & 52.50 & 62.22 & 56.24 & 43.39 \\

&
8K & 23.53 & 46.63 & 48.68 & 49.61 & 47.16 & 21.14 & 33.10 & 23.20 & 26.92 & 72.00 & 91.29 & 43.79 & 3.25 & 66.00 & 62.18 & 56.43 & 44.68 \\

\midrule

% SnapKV
\multirow{4}{*}{SnapKV} &
2K & 21.81 & 37.22 & 37.19 & 46.10 & 35.42 & 16.53 & 29.83 & 21.05 & 26.77 & 61.00 & 88.84 & 42.56 & 4.03 & 51.50 & 62.37 & 51.45 & \underline{39.60} \\

&
4K & 24.79 & 44.22 & 47.30 & 48.49 & 46.73 & 20.55 & 32.19 & 22.68 & 26.95 & 67.50 & 90.98 & 43.14 & 5.17 & 89.50 & 61.44 & 51.20 & \underline{45.18} \\

&
6K & 24.10 & 45.57 & 50.44 & 53.12 & 48.41 & 24.27 & 33.43 & 23.53 & 27.03 & 71.50 & 92.28 & 43.58 & 5.25 & 98.00 & 61.32 & 52.16 & \underline{47.12} \\

&
8K & 25.15 & 46.55 & 53.39 & 56.00 & 48.75 & 27.82 & 33.67 & 24.85 & 27.01 & 72.50 & 91.78 & 43.54 & 5.08 & 100.00 & 61.48 & 51.41 & \underline{48.06} \\

\midrule

% KeyDiff
\multirow{4}{*}{\ourmethod{}} &
2K & 26.64 & 41.73 & 50.99 & 51.59 & 46.47 & 22.84 & 29.02 & 23.86 & 26.76 & 66.50 & 85.92 & 39.26 & 3.17 & 96.00 & 59.17 & 39.42 & \textbf{44.33} \\

&
4K & 28.70 & 45.62 & 56.06 & 54.58 & 49.31 & 28.25 & 32.30 & 25.03 & 27.07 & 70.00 & 90.85 & 42.84 & 4.21 & 99.00 & 60.80 & 48.00 & \textbf{47.66} \\
&
6K & 29.90 & 46.33 & 55.11 & 56.80 & 49.50 & 31.52 & 33.44 & 24.58 & 26.98 & 72.00 & 90.99 & 43.10 & 5.27 & 99.50 & 61.40 & 49.70 & \textbf{48.51} \\

&
8K & 33.57 & 46.77 & 55.48 & 56.87 & 49.37 & 30.88 & 34.17 & 25.12 & 27.01 & 72.50 & 92.28 & 42.81 & 5.83 & 99.50 & 61.48 & 50.90 & \textbf{49.03} \\

\midrule
\midrule
\addlinespace

%%%%%%%%%%%%%%%%%%%%%%%%%%%%%%%%%%%%%%%%%%%%%%%%%%%%%%%%
%%%%%%%%%%%%%%%%% Llama 3.2-3B Results %%%%%%%%%%%%%%%%%
%%%%%%%%%%%%%%%%%%%%%%%%%%%%%%%%%%%%%%%%%%%%%%%%%%%%%%%%

% Dense
\multicolumn{2}{c}{Llama3.2-3B} &
23.76 & 40.23 & 50.09 & 50.69 & 42.29 & 26.84 & 33.09 & 24.30 & 25.21 & 72.50 & 90.11 & 42.58 & 3.00 & 96.50 & 56.22 & 56.52 & 45.87 \\
\midrule

% H2O
\multirow{4}{*}{H2O} &
2K & 1.63 & 19.96 & 20.20 & 18.02 & 19.56 & 2.88 & 0.78 & 1.55 & 15.97 & 41.00 & 21.97 & 9.83 & 0.50 & 0.50 & 39.71 & 13.91 & 14.25 \\

&
4K & 2.92 & 31.94 & 33.23 & 24.49 & 28.08 & 7.55 & 5.44 & 6.30 & 22.77 & 53.00 & 38.85 & 20.33 & 1.50 & 7.50 & 51.23 & 22.94 & 22.38 \\

&
6K & 4.62 & 38.81 & 39.06 & 34.66 & 35.52 & 15.21 & 10.51 & 10.01 & 24.25 & 61.50 & 53.23 & 27.37 & 0.50 & 13.00 & 54.55 & 32.29 & 28.44 \\

&
8K & 9.65 & 39.66 & 43.20 & 38.09 & 40.41 & 21.46 & 17.80 & 13.28 & 24.67 & 70.00 & 64.30 & 32.19 & 2.00 & 24.50 & 55.00 & 39.09 & 33.46 \\
\midrule

% TOVA
\multirow{4}{*}{TOVA} &
2K & 17.14 & 30.14 & 32.44 & 35.96 & 30.05 & 13.08 & 26.15 & 19.70 & 25.04 & 56.50 & 87.81 & 40.48 & 2.50 & 11.50 & 55.51 & 52.36 & 33.52 \\

&
4K & 20.52 & 39.53 & 42.47 & 44.12 & 38.42 & 18.22 & 29.36 & 21.36 & 24.96 & 63.50 & 88.98 & 41.50 & 3.00 & 23.50 & 55.72 & 56.66 & 38.24 \\

&
6K & 20.22 & 39.78 & 45.86 & 49.08 & 41.54 & 20.43 & 30.50 & 22.17 & 25.11 & 66.50 & 89.00 & 42.50 & 4.00 & 46.50 & 55.57 & 57.53 & 41.02 \\

&
8K & 21.08 & 40.67 & 49.07 & 48.69 & 41.93 & 23.05 & 31.64 & 22.85 & 25.21 & 69.00 & 89.25 & 42.19 & 2.50 & 71.00 & 55.77 & 57.47 & 43.21 \\
\midrule

% Sink
\multirow{4}{*}{Sink} &
2K & 16.85 & 30.69 & 26.58 & 33.26 & 25.27 & 13.82 & 26.74 & 19.15 & 25.15 & 65.00 & 86.17 & 40.79 & 1.50 & 19.50 & 56.65 & 52.73 & 33.74 \\

&
4K & 19.46 & 38.61 & 36.22 & 41.97 & 35.84 & 13.37 & 29.34 & 20.19 & 25.06 & 71.00 & 88.06 & 41.31 & 2.50 & 35.50 & 56.48 & 52.43 & 37.96 \\

&
6K & 19.33 & 40.29 & 37.95 & 46.48 & 40.29 & 15.31 & 30.43 & 21.35 & 25.14 & 71.50 & 88.93 & 42.04 & 3.50 & 47.00 & 56.55 & 54.11 & 40.01 \\

&
8K & 20.15 & 40.02 & 41.94 & 48.15 & 42.24 & 16.01 & 31.64 & 22.10 & 25.20 & 73.00 & 89.26 & 42.37 & 3.50 & 62.50 & 56.86 & 56.63 & 41.97 \\

\midrule

% SnapKV
\multirow{4}{*}{SnapKV} &
2K & 17.38 & 31.37 & 31.48 & 37.77 & 30.05 & 11.54 & 27.03 & 19.93 & 24.97 & 59.00 & 88.13 & 40.48 & 3.50 & 32.50 & 56.32 & 55.91 & \underline{35.46} \\

&
4K & 19.85 & 39.22 & 39.86 & 46.70 & 37.98 & 16.64 & 29.79 & 21.21 & 25.01 & 65.50 & 89.35 & 40.95 & 2.50 & 62.50 & 55.74 & 56.88 & \underline{40.60} \\

&
6K & 20.83 & 39.65 & 44.48 & 49.30 & 40.18 & 20.28 & 31.27 & 22.73 & 25.09 & 69.00 & 89.95 & 41.47 & 4.00 & 85.00 & 55.69 & 57.82 & \underline{43.55} \\

&
8K & 20.49 & 40.80 & 48.16 & 48.78 & 41.65 & 24.79 & 31.81 & 23.46 & 25.17 & 70.00 & 90.17 & 41.99 & 5.00 & 94.00 & 55.77 & 57.29 & \underline{44.96} \\

\midrule

% KeyDiff
\multirow{4}{*}{\ourmethod{}} &
2K & 18.29 & 36.65 & 45.44 & 46.09 & 35.41 & 13.79 & 28.16 & 21.45 & 25.01 & 60.00 & 85.24 & 37.00 & 1.00 & 60.50 & 54.13 & 42.01 & \textbf{38.14} \\

&
4K & 22.34 & 40.60 & 49.15 & 50.14 & 40.30 & 21.65 & 31.38 & 23.44 & 25.06 & 66.50 & 87.92 & 41.41 & 2.50 & 88.50 & 55.55 & 52.24 & \textbf{43.67} \\

&
6K & 22.29 & 40.68 & 50.14 & 51.74 & 42.19 & 24.83 & 32.39 & 23.53 & 25.19 & 71.00 & 90.02 & 42.00 & 3.00 & 95.00 & 55.86 & 54.39 & \textbf{45.27} \\

& 
8K & 22.41 & 40.77 & 50.10 & 49.83 & 43.58 & 28.09 & 32.78 & 23.60 & 25.17 & 72.00 & 90.17 & 42.46 & 3.50 & 96.50 & 55.85 & 55.65 & \textbf{45.78} \\

\bottomrule
\end{tabular}
}
\vspace{-5mm}
\end{center}
\end{table*}

%% file: contents/05_related_work.tex
\section{Related Work}
\label{sec:related_work}
\paragraph{Sparse Attention} LLMs often exhibit high attention sparsity, where a small subset of keys receives a significant proportion of attention scores. This characteristic allows sparse approximation techniques to reduce the computational cost of attention. Similar to PagedAttention \cite{kwon2023efficient}, \citet{tang2024quest} estimates the importance of a page (a contiguous set of keys) to a given query, whereas \citet{rehg2024kv} further refined the budgets in a per-head manner. On the contrary, sample-based methods \cite{zhu2024sampleattention, ribarsparq} attempt to approximate token importance by inspecting the attention scores from the last few queries or certain query channel dimensions.
Despite their effectiveness in reducing computational costs, these methods do not address the memory overhead of the KV cache, which typically retains all KVs.

\paragraph{KV Cache Compression}
Different approaches to compress the KV cache include architecture modification such as GQA \cite{ainslie2023gqa}, which shares a KV cache across a small number of heads. Other techniques to compress the KV cache include quantization such as in \cite{hooper2024kvquant,liu2024kivi,zhang2024kv} in which the authors use various techniques to take advantage of existing patterns to efficiently quantize and compress the KV cache. More related to our work \cite{yang2024no} uses a scoring mechanism to determine the precision of the quantization for different tokens. 

\paragraph{Token Eviction Methods} Unlike the sparse attention and KV cache compression methods, eviction methods \textit{evict} KVs from the cache to reduce the size of the KV cache. 
As discussed in \cref{subsec:kv_cache_eviction}, the majority of the token eviction methods employ their own rules to decide the importance of the tokens by manipulating the attention score $A$. For example, by appropriately choosing the aggregation functions $\phi(A)$ of \cref{eq:eviction_policy}, we can obtain existing attention-based eviction methods as discussed in \cref{appendix:subsec:attention-based-eviction-method}.
 Attention-based eviction may be a better choice when the entire prompt is being processed at once, as the eviction can be done by assessing the importance of all tokens simultaneously.
However, computing the full attention score of long prompts could be prohibitively expensive in resource-constrained environments.

%% file: contents/06_conclusion.tex
\section{Conclusion}
\label{sec:conclusion}
Inspired by our observation that distinctive keys tend to have high attention scores, we propose \ourmethod{}, a training-free KV cache eviction method based on key similarity that enables large language models to operate in memory and compute constrained environments. We justify \ourmethod{} by showing that it minimizes the pairwise cosine similarity among keys in the KV cache, maximizing the aforementioned diversity. \ourmethod{} significantly outperforms state-of-the-art KV cache eviction methods under similar memory constraints, with only a 1.5\% and 0.04\% accuracy drop from the non-evicting baseline while achieving 33\% and 23\% KV cache memory reduction on LongBench. Similar to other token eviction methods, \ourmethod{} is primarily designed and evaluated for the GQA attention mechanism used in models such as Llama and Qwen. In future work, we plan to extend \ourmethod{} for seamless integration with other attention variants, such as Multi-Head Latent Attention \cite{guo2025deepseek}.

%% file: contents/07_appendix.tex
\renewcommand \thepart{}
\renewcommand \partname{}

\newpage
\rule[0pt]{\columnwidth}{3pt}
\begin{center}
    \huge{\ourmethod{} \\
    Supplementary Material}
\end{center}
\vspace*{3mm}
\rule[0pt]{\columnwidth}{1pt}%\hline
\vspace*{-.5in}

% Making Table of contents for appendix only.
% https://tex.stackexchange.com/questions/419249/table-of-contents-only-for-the-appendix
\appendix
\addcontentsline{toc}{section}{}
\part{}
\parttoc

% Appendix trick from
% https://ckadapa.wordpress.com/2019/09/24/formatting-equations-in-appendix-in-latex/
\renewcommand{\theequation}{A.\arabic{equation}}
% reset the counter
\setcounter{equation}{0}

\input{contents/A00_extended_related_work}

\input{contents/A01_mathematical_discussion}

\input{contents/A02_TTFT}
\input{contents/A03_Reasoning}
\input{contents/A04_longbench}

\input{contents/A05_ablation_study}
\input{contents/A06_additional_evaluation}

%% file: contents/A00_extended_related_work.tex
\section{Extended Related Work}
\subsection{Attention-based eviction methods}
\label{appendix:subsec:attention-based-eviction-method}
In this section, we provide a unified framework to understand prominent attention-based eviction methods. As mentioned in \cref{eq:attn_based_eviction_policy}, we can specify attention-based eviction methods under the unified framework with proper selection of the aggregation function $\phi(A)$ as follows:
\begin{itemize}[leftmargin=*]
    \item TOVA \cite{oren2024transformers}: $\phi^{\mathrm{TOVA}}(A) = A_{-1, :}$,
    \item H2O \cite{zhang2024h2o}: $\phi^{\mathrm{H2O}}(A) = A_{\mathrm{prev}} + A^{\top}1$,
    \item SnapKV \cite{xiaoefficient}: $\phi^{\mathrm{SnapKV}}(A) = (A^{\top}1) \ast K$, where $K$ is a vector of $\frac{1}{k}$ and $k$ is the kernel size of average smoothing.     
\end{itemize}

\section{Runtime and Memory Complexity}
\label{appendix:subsec:runtime-and-memory-complexity}
We analyze the runtime and memory complexity for the prominent KV cache eviction algorithms in \cref{table:runtime-and-memory}. For a given block size $B$ and cache budget $N$, \ourmethod{} requires $\mathcal{O}(N+B)$ runtime and memory. 
The same holds true for TOVA, since it only requires computing the bottom row of $A$.
Sink attention retains the $k$ first tokens in the input sequence, followed by a sliding window of size $L$, resulting in $\mathcal{O}(k+L) = O(N)$ memory and runtime, since $k+L$ equals the chosen cache budget.
SnapKV computes attention over a sliding window of size $L$ against $N+B$ keys from the incoming block and the key cache, so the memory and runtime complexity is $\mathcal{O}((N+B)L)$. 
H2O accumulates attention weights over all tokens, and computes the attention over the current block, so it will require $\mathcal{O}(NB+B^2)$ memory overhead and runtime.
We summarize these details in \cref{table:runtime-and-memory}

\begin{table}[h]
\caption{\textbf{Runtime and memory complexity of token eviction methods.}}
\label{table:runtime-and-memory}
\begin{center}
\begin{tabular}{lll} 
\toprule
 & Runtime Complexity & Memory Complexity \\
\midrule
\ourmethod{} & $\mathcal{O}(N+B)$ & $\mathcal{O}(N+B)$ \\ 
TOVA & $\mathcal{O}(N+B)$ & $\mathcal{O}(N+B)$ \\ 
H2O & $\mathcal{O}(NB+B^2)$ & $\mathcal{O}(NB+B^2)$ \\ 
SnapKV & $\mathcal{O}((N+B)L)$ & $\mathcal{O}((N+B)L)$ \\ 
Sink & $\mathcal{O}(N)$ & $\mathcal{O}(1)$ \\ 
\bottomrule
\end{tabular}
\end{center}
\end{table}
 \subsection{FLOP count of KeyDiff}
 \label{sec:keydiff-flop-count}
The bulk of the computation in KeyDiff (neglecting the `topk` operator) is the following two expressions:
\begin{itemize}[leftmargin=*]
    \item $\mu(\hat{K}) = \frac{1}{n}\sum_{i=1}^n \frac{k_i}{\|k_i\|}$
    \item $s_i = \textrm{CosSim}(\mu(\hat{K}), k_i) = \frac{\mu(\hat{K})\cdot k_i}{\max(\|\mu(\hat{K})\|\cdot\|k_i\|, \,\epsilon)},\quad i=1,\dots, n$
\end{itemize}
We will count the total number of additions, multiplications, square roots and divisions required by KeyDiff separately, since division and square root implementation are hardware-dependent, then assign weights to each operation at the end for a final count. Norms are assumed to be 2-norms. We count the FLOPs required for each operation as follows:
\begin{itemize}[leftmargin=*]
\item $\|k_i\| = \sqrt{ k_i \cdot k_i}$: since $k_i \in \mathbb{R}^d$: $d$ multiplications, $d-1$ additions, one square root. Repeating for each $i$, this contributes $nd$ multiplications, $n(d-1)$ additions, and $n$ square roots.
\item $\frac{k_i}{c}$: naively, $d$ divisions, but can be rewritten as one division and $d$ multiplications. Repeating for each $i$, this contributes $nd$ multiplications and $n$ divisions.
\item $\frac{1}{n}\sum_{i=1}^n c_i$, for $c\in \mathbb{R}^d$: one division, $(n-1)d$ additions.
\end{itemize}
 
Combining the above, we can compute the anchor vector using $2nd$ multiplications, $2nd-n-d$ additions, $n$ square roots and $n+1$ divisions. 

To compute the cosine similarity score, we have from above that $\mu(\hat{K}) \cdot k_i$ requires $nd$ multiplications and $n(d-1)$ additions. Also from above, we have that computing $\|\mu(\hat{K})\|$ requires $d$ multiplications, $d-1$ additions and one square root. We reuse the computation of $\|k_i\|$ from the previous step and compute $\|k_i\|\|\mu(\hat{K})\|$ in $n$ multiplications and $\max(\|k_i\|\|\mu(\hat{K})\|, \,\epsilon)$ with more $n$ additions (assuming boolean comparison equals addition in cost). We can then divide through to compute $\frac{\mu(\hat{K}) \cdot k_i}{\max(\|k_i\|\|\mu(\hat{K})\|, \, \epsilon)}$ with $n$ divisions. Therefore, computing the cosine similarity between the anchor and each key requires $nd+d+n$ multiplications, $nd+d-1$ additions, $n+1$ divisions and one square root.

Adding everything up, KeyDiff requires:
\begin{enumerate}[leftmargin=*]
\item $3nd+d+n$ multiplications,
\item $3nd-n-1$ additions,
\item $2n+2$ divisions,
\item $n+1$ square roots,
\end{enumerate}

If, based on x86 instruction tables, we declare additions and square roots cost one FLOP (i.e. can be computed in one cycle), multiplications cost three FLOPs, division is roughly 47 FLOPs, we arrive at a final FLOP count of:
\begin{equation}
3(3nd+d+n)+ (3nd-n-1)+ 47*(2n+2)+ (n+1) =
(12d+97)n+3d+94.
\end{equation}

This is linear in the number of keys $n$ with a small constant, relative to the quadratic complexity of the attention operator.

%% file: contents/A01_mathematical_discussion.tex
\section{\ourmethod{}: A Theoretical Perspective}
\subsection{Additional PCA Visualizations}
\label{appendix:subsec:pca-visualizations}
In order to demonstrate the phenomena in \cref{fig:pca-plot} persists across all heads and layers, we have included several more visualizations as seen in \cref{fig:pca-plot-27-4}, \cref{fig:pca-plot-20-0}, and \cref{fig:pca-plot-8-1}. 
\begin{figure*}[ht!]
     % \vspace{-2mm}
     \centering
     \begin{subfigure}[b]{0.24\linewidth}
         \centering
         \includegraphics[width=\linewidth]{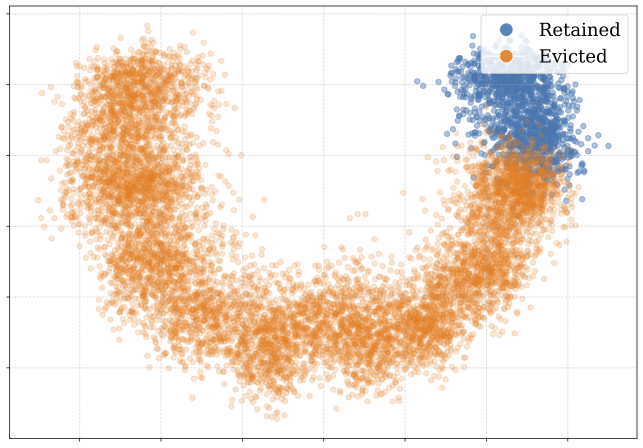}
         \caption{Sink Attention}
     \end{subfigure}
     \begin{subfigure}[b]{0.24\linewidth}
         \centering
         \includegraphics[width=\linewidth]{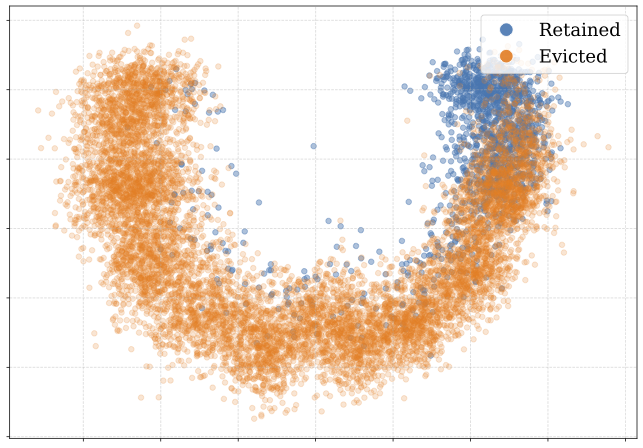}
         \caption{TOVA}
     \end{subfigure}
     \begin{subfigure}[b]{0.24\linewidth}
         \centering
         \includegraphics[width=\linewidth]{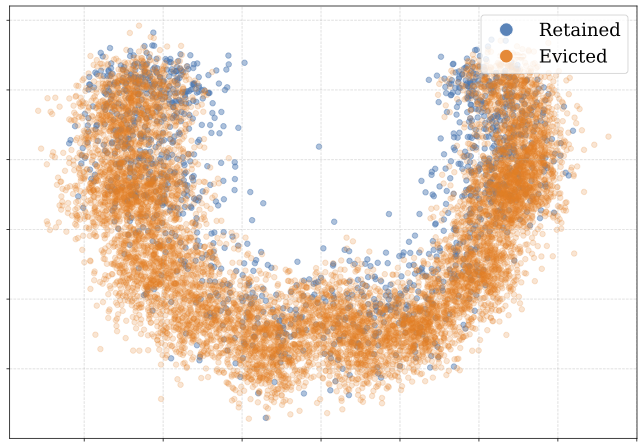}
         \caption{\ourmethod{}}
     \end{subfigure}
     \begin{subfigure}[b]{0.24\linewidth}
         \centering
         \includegraphics[width=\linewidth]{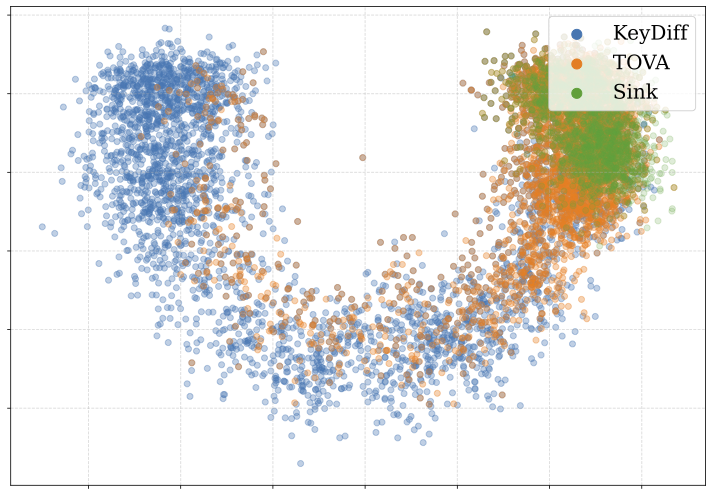}
         \caption{Retained keys only}
     \end{subfigure}
     \hfill
     % \vspace{-2mm}
     \caption{\textbf{(a, b, and c)} PCA Visualizations in two dimensions of a key cache managed with Sink, TOVA, and \ourmethod{}. Retained tokens are \textcolor{blue}{blue}, while evicted tokens are \textcolor{orange}{orange}. 
     Keys are taken from layer $27$ and head $4$ of Llama3.2-3B-Instruct, and generated using the NarrativeQA dataset. \textbf{(d)} PCA visualization of the retained keys for each KV cache eviction method}
    \label{fig:pca-plot-27-4}
% \vspace{-4mm}
\end{figure*}

\begin{figure*}[ht!]
     % \vspace{-2mm}
     \centering
     \begin{subfigure}[b]{0.24\linewidth}
         \centering
         \includegraphics[width=\linewidth]{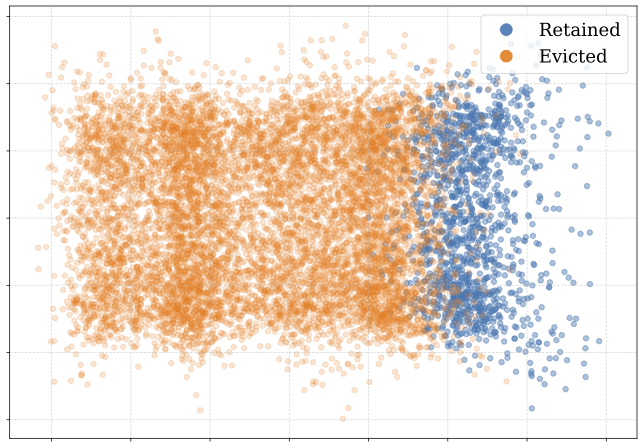}
         \caption{Sink Attention}
     \end{subfigure}
     \begin{subfigure}[b]{0.24\linewidth}
         \centering
         \includegraphics[width=\linewidth]{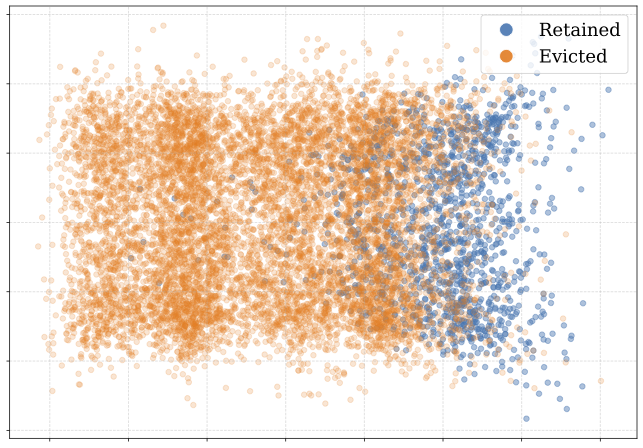}
         \caption{TOVA}
     \end{subfigure}
     \begin{subfigure}[b]{0.24\linewidth}
         \centering
         \includegraphics[width=\linewidth]{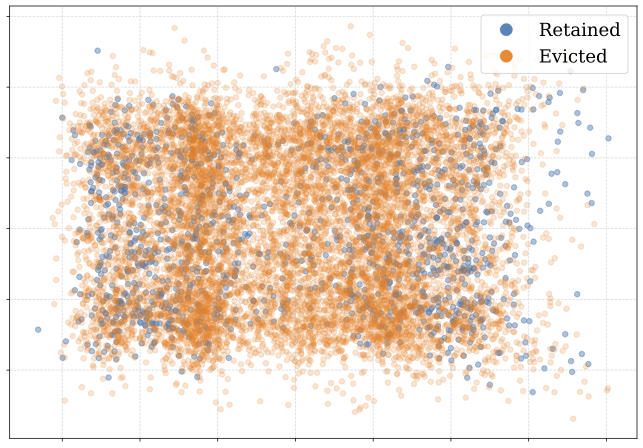}
         \caption{\ourmethod{}}
     \end{subfigure}
     \begin{subfigure}[b]{0.24\linewidth}
         \centering
         \includegraphics[width=\linewidth]{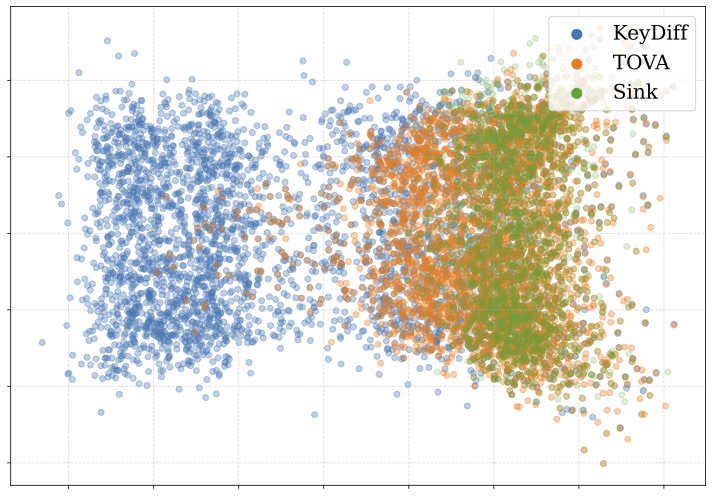}
         \caption{Retained keys only}
     \end{subfigure}
     \hfill
     % \vspace{-2mm}
     \caption{Keys taken from layer $20$ and head $0$ of Llama3.2-3B-Instruct}
    \label{fig:pca-plot-20-0}
\end{figure*}

\begin{figure*}[ht!]
     % \vspace{-2mm}
     \centering
     \begin{subfigure}[b]{0.24\linewidth}
         \centering
         \includegraphics[width=\linewidth]{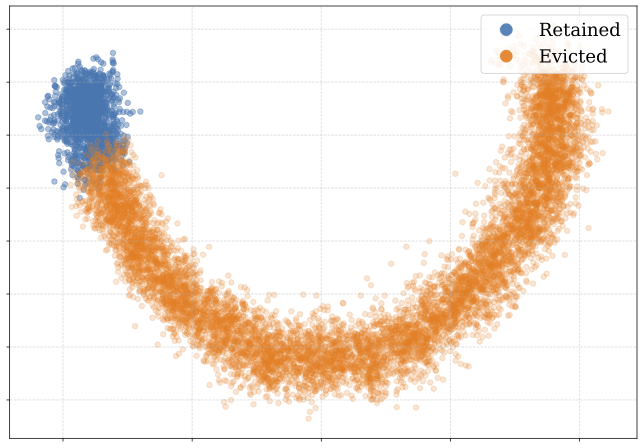}
         \caption{Sink Attention}
     \end{subfigure}
     \begin{subfigure}[b]{0.24\linewidth}
         \centering
         \includegraphics[width=\linewidth]{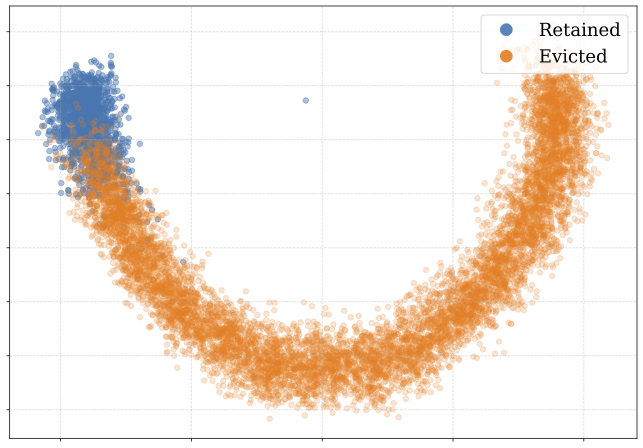}
         \caption{TOVA}
     \end{subfigure}
     \begin{subfigure}[b]{0.24\linewidth}
         \centering
         \includegraphics[width=\linewidth]{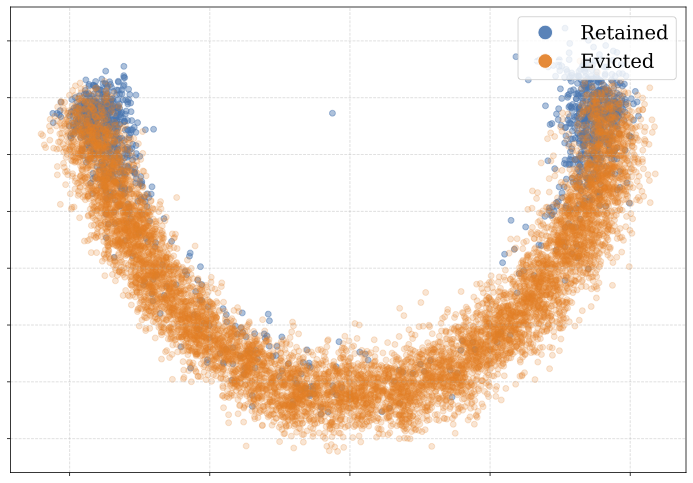}
         \caption{\ourmethod{}}
     \end{subfigure}
     \begin{subfigure}[b]{0.24\linewidth}
         \centering
         \includegraphics[width=\linewidth]{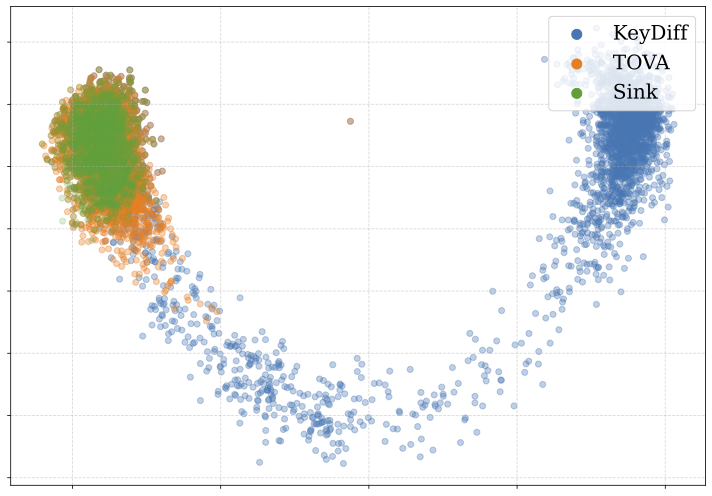}
         \caption{Retained keys only}
     \end{subfigure}
     \hfill
     % \vspace{-2mm}
     \caption{Keys taken from layer $8$ and head $1$ of Llama3.2-3B-Instruct}
    \label{fig:pca-plot-8-1}
\end{figure*}

\begin{figure*}[ht!]
     \centering
     \begin{subfigure}[b]{0.24\linewidth}
         \centering
         \includegraphics[width=\linewidth]{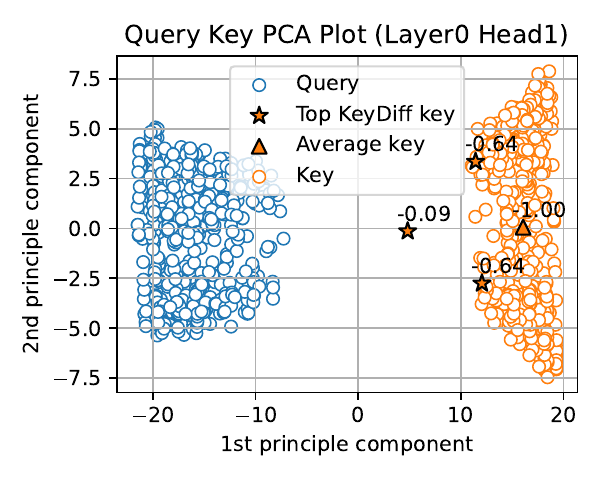}
         \caption{Layer 0 Head 1}
     \end{subfigure}
     \begin{subfigure}[b]{0.24\linewidth}
         \centering
         \includegraphics[width=\linewidth]{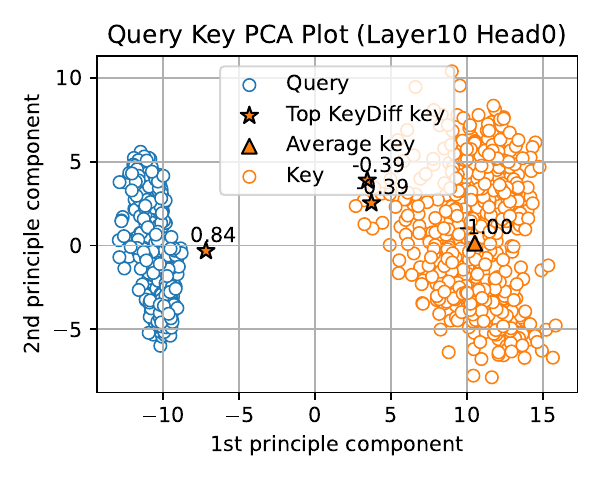}
         \caption{Layer 10 Head 1}
     \end{subfigure}
     \begin{subfigure}[b]{0.24\linewidth}
         \centering
         \includegraphics[width=\linewidth]{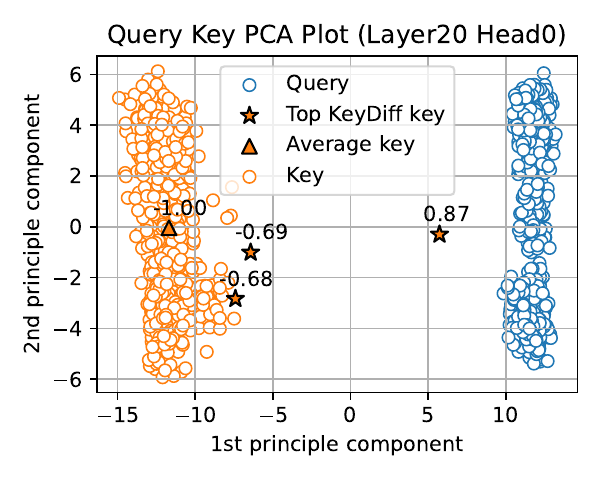}
         \caption{Layer 20 Head 1}
     \end{subfigure}
     \begin{subfigure}[b]{0.24\linewidth}
         \centering
         \includegraphics[width=\linewidth]{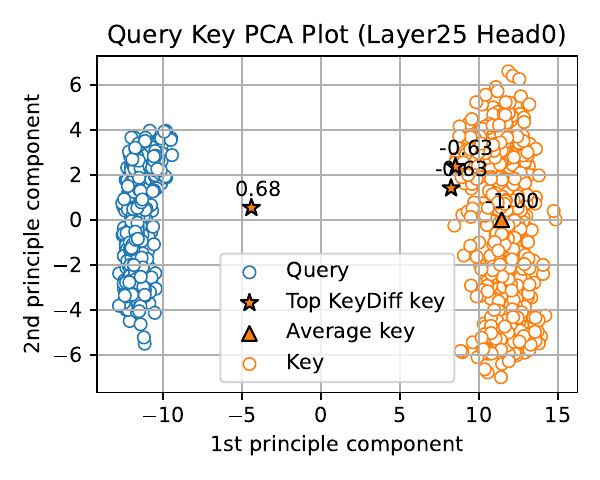}
         \caption{Layer 25 Head 1}
     \end{subfigure}
     \hfill
     \caption{PCA plots of Query and Keys from Llama-3.2-3B-Instruct}
    \label{fig:pca-plot}
\end{figure*}

\subsection{Derivation of \ourmethod{} from an Optimization Perspective}
\label{subsec:opt_deriv_keydiff}
To leverage the observation in \cref{subsec:attention_and_dissimilarity}, we minimize the sum of pairwise cosine similarities of each key retained in the cache. 
This can be formulated as a constrained optimization problem with the keys $K \in \mathbb{R}^{n \times d}$ whose element is $k_i$, and budget $N$ smaller than $n$:
\begin{mini}
    {S}{\sum\limits_{i \in S}\sum\limits_{j\in S} \frac{ k_i \cdot k_j }{\|k_i\| \|k_j\|} }{}{}
    \addConstraint{S}{\subseteq \{1,\hdots,n\} }
    \addConstraint{|S|}{=N}
    \label{eq:min_cossim_opt}
\end{mini}

This is a combinatorial optimization problem, which is difficult to solve efficiently, particularly during inference. 
However, we can relax \cref{eq:min_cossim_opt} to produce a more tractable approximation to the original problem. First, we rewrite \cref{eq:min_cossim_opt} by normalizing keys $k_i$ such that $\hat{k}_i = \frac{k_i}{\|k_i\|}$ resulting in: 
\begin{align*} \sum\limits_{i \in S}\sum\limits_{j\in S}  \hat{k}_i \cdot \hat{k}_j &= \sum\limits_{i \in S} \hat{k}_i  \cdot \left(\sum\limits_{j \in S} \hat{k}_j \right)
= \sum\limits_{i \in S} \langle \hat{k}_i , N\mu (\hat{K}_S) \rangle
\end{align*}
where $\mu(\hat{K}_S)=\frac{1}{N}\sum\limits_{j \in S} \hat{k}_j$ is the empirical mean of normalized keys in $S$. This objective requires recomputing $\mu( \hat{K}_S)$ for each candidate subset $S$. The sub-sampled mean tends to converge to the mean over the entire set, the problem can be further relaxed by replacing $\mu(\hat{K}_S)$ with $\mu(\hat{K}) = \frac{1}{n} \sum_{i=1}^{n} \hat{k}_i$. Dropping $N$ from the objective (since it doesn't affect the solution) yields: 
\begin{mini}
    {S}{\sum\limits_{i \in S}  \hat{k}_i \cdot \mu(\hat{K})  }{}{}
    \addConstraint{S}{\subseteq \{1,\hdots,n\} }
    \addConstraint{|S|}{= N}
    \label{eq:relaxed_min_cossim_opt}
\end{mini}
The optimal solution of \cref{eq:relaxed_min_cossim_opt} can be found by sorting tokens using their cosine similarity with $\mu(\hat{K})$ and selecting the smallest $N$, leading to the algorithm described in \cref{eq:efficient-keydiff}.

\paragraph{Key Cache Diversity and \ourmethod{}}
To empirically verify \ourmethod{}'s ability to retain diverse keys, we apply PCA to the keys computed in an attention block of Llama 3.2-3B-Instruct after evaluating a long context prompt.
We visualize the distribution of  retained and evicted keys by applying sink attention \cite{xiaoefficient}, TOVA \cite{oren2024transformers}, and \ourmethod{} as eviction policies in  \cref{fig:pca-plot}.
Visual observation reveals the tokens retained by sink attention and TOVA tend to tightly cluster together while \ourmethod{}'s retained tokens are more evenly distributed.
 This observation generalizes across heads and layers, as shown in \cref{appendix:subsec:pca-visualizations}.

We also visualize the log determinant of the Gram matrix of the key cache $\log(\det(KK^T))$ generated using different eviction policies in \cref{fig:gram_matrix}. This quantity corresponds to the volume of space spanned by the keys in $\mathcal{C}$. 
The distribution of volumes for \ourmethod{} attains 
higher values, indicating that the retained keys are more distinctive than TOVA and sink, corroborating the results of \cref{fig:pca-plot}. Details of how these plots were generated are discussed in \cref{appendix:subsec:sec3-setup}.

\begin{figure}[ht!]
     \vspace{-2mm}
     \centering
     \begin{subfigure}[b]{.5\linewidth}
         \centering
         \includegraphics[width=\linewidth]{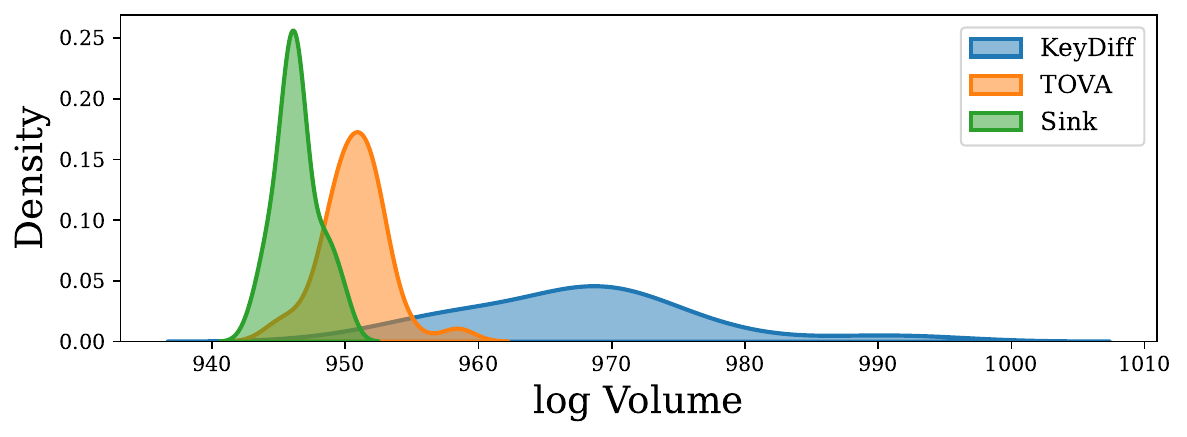}
     \end{subfigure}
     \vspace{-1mm}
     \caption{Distribution of
     $\log\left(\textrm{det}\left(K K^T\right)\right)$
     from the Qasper dataset in LongBench. Larger values mean more of the key space is spanned by the key cache. \ourmethod{} retains keys that span a greater volume of the ambient space than TOVA or sink attention.}
    \label{fig:gram_matrix}
\vspace{-4mm}
\end{figure}

\subsection{Attention Sinks and Approximate Collinearity}
\label{appendix:theoretical-justification}
\begin{figure*}[ht!]
     \vspace{-2mm}
     \centering
     \begin{subfigure}[b]{0.49\linewidth}
         \centering
         \includegraphics[width=\linewidth]{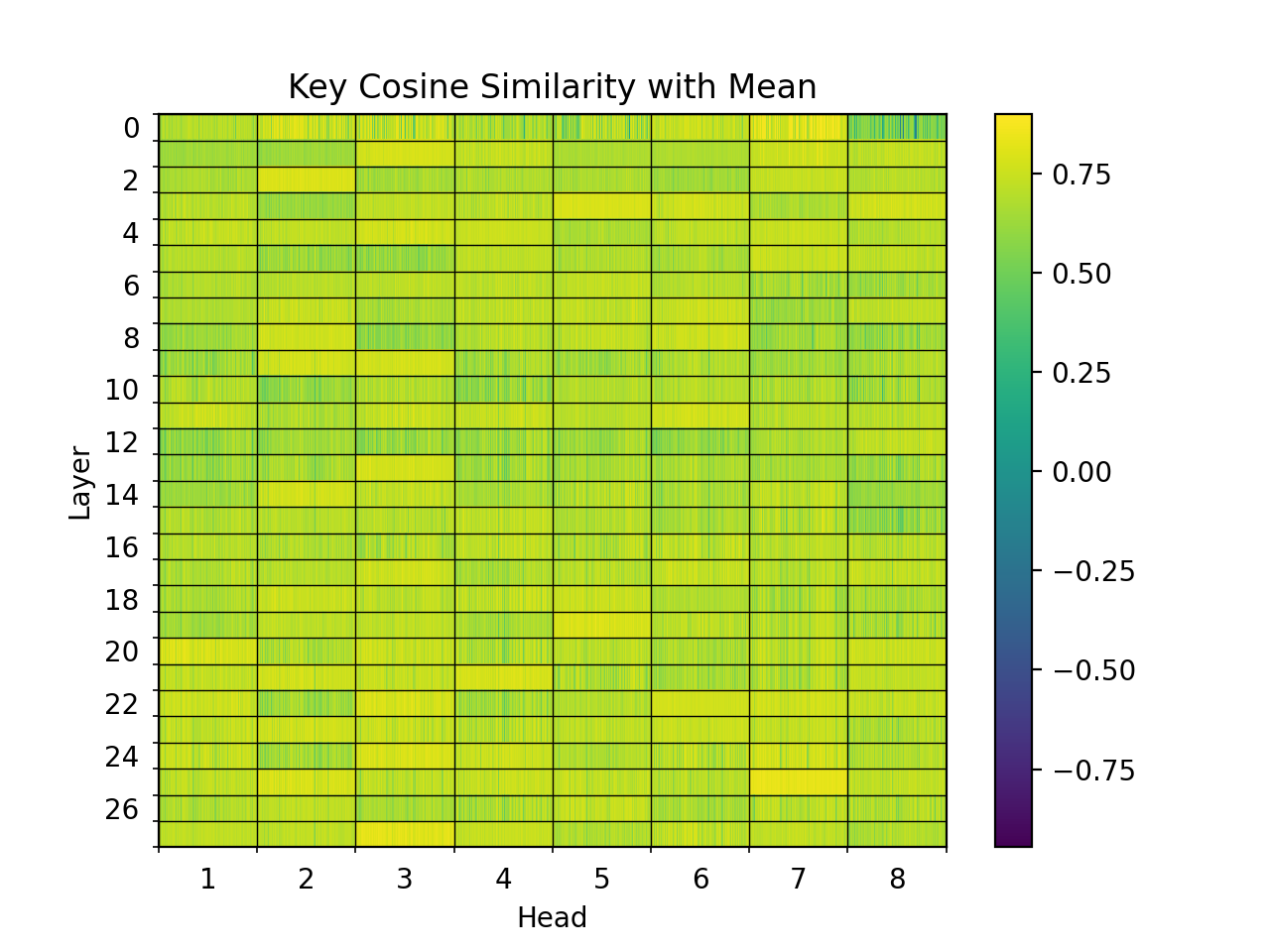}
         \caption{Keys}
     \end{subfigure}
     \begin{subfigure}[b]{0.49\linewidth}
         \centering
         \includegraphics[width=\linewidth]{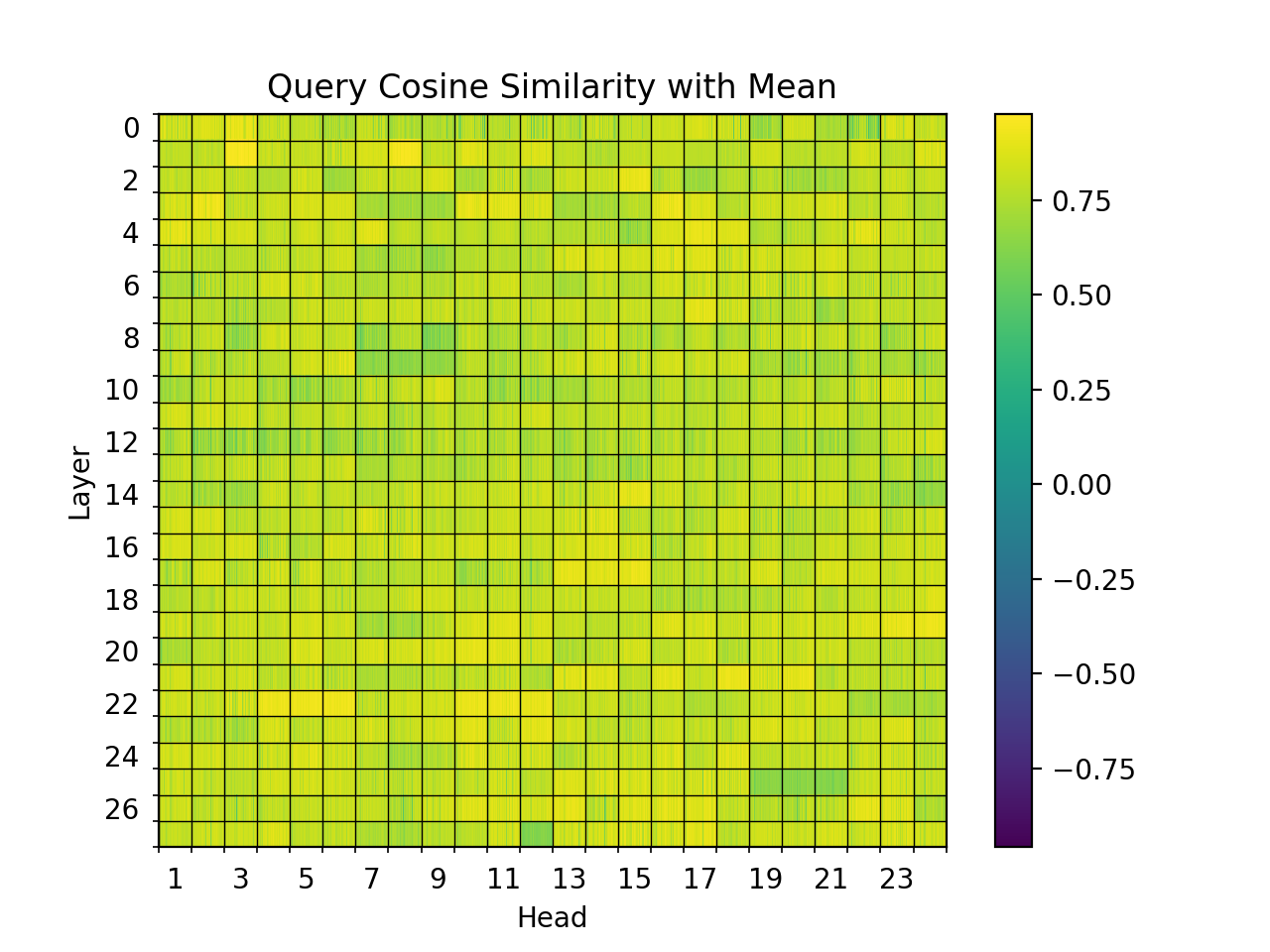}
         \caption{Queries}
     \end{subfigure}
     \hfill
     \vspace{-2mm}
     \caption{Cosine similarity between keys and their mean (left) and queries and their mean (right) across heads and layers}
    \label{fig:key_query_dist_plot}
\end{figure*}

\begin{figure*}[ht!]
     \vspace{-2mm}
     \centering
         \includegraphics[width=.7\linewidth]{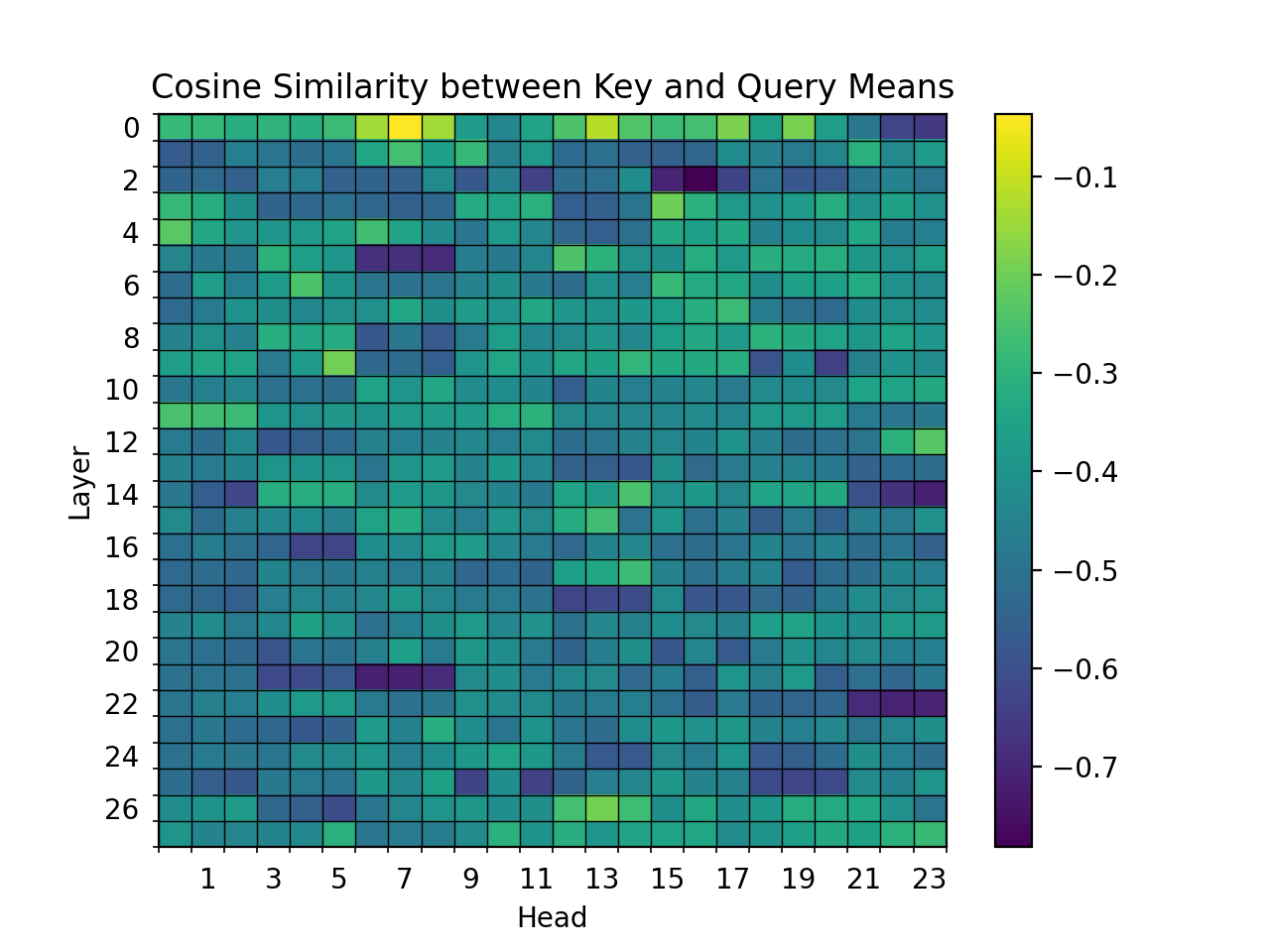}
     \caption{Cosine similarity of mean key and mean query for each head and layer.}
    \label{fig:mean_key_query_cossim}
\end{figure*}

To better understand why there exists a negative correlation between cosine similarity between keys and attention scores, we look to recent research that seeks to the importance of attention sinks in decoder-based LLMs. The authors in \cite{barbero2025llms} show that attention sinks emerge from training decoder-based LLMs since they can denoise the model and prevent rank collapse by limiting over mixing in attention heads. Moreover, attention patterns in decoder based models demonstrate that most attention logits are quite small (and almost always negative) for most keys and queries. This allows the attention sink to have high attention activation, preventing over mixing in addition to allowing the heads to specialize and selectively identify important tokens.

At the same time, the results in \cite{godey2024anisotropy} suggest that hidden states, keys and queries are all approximately collinear in the sense that $\textrm{CosSim}(x_i, x_j) \gg 0$. In geometric terms, this means that most key tokens and query tokens lie in the same direction in Euclidean space. Our own results, seen in \cref{fig:key_query_dist_plot} demonstrate that this is the case for both keys and queries. Our experiments show that most keys and queries lie within a small angular distance from the mean key and query. More than this, we see that across all heads, the mean key and mean query have negative cosine similarity \cref{fig:mean_key_query_cossim}. Moreover, as seen in \cref{fig:key_query_l2_dist}, we find that the norms of keys and queries are tightly clustered around a relatively fixed value. This means that variations in the norm of key and query tokens have less impact on the magnitude of attention scores than their direction. These three observations indicate that most keys and queries combine to create uniformly small attention logits, and that larger attention weights are constructed by projecting keys closer to the direction of the mean query. This appears to be the fundamental mechanism through which over mixing is prevented: if most attention activations are very small, each head can increase the activations of a small number of keys across most queries selectively projecting them to be more aligned with the distribution of queries. This hypothesis is further supported by the fact that sink tokens themselves often have small norm, which results in an approximate no-op in the attention head as in \cite{barbero2025llms}, however in this case, the only way for a key corresponding to a sink token to have high attention scores is if it is as parallel as possible to the set of query tokens.

\begin{figure*}[ht!]
     \vspace{-2mm}
     \centering
     \begin{subfigure}[b]{0.49\linewidth}
         \centering
         \includegraphics[width=\linewidth]{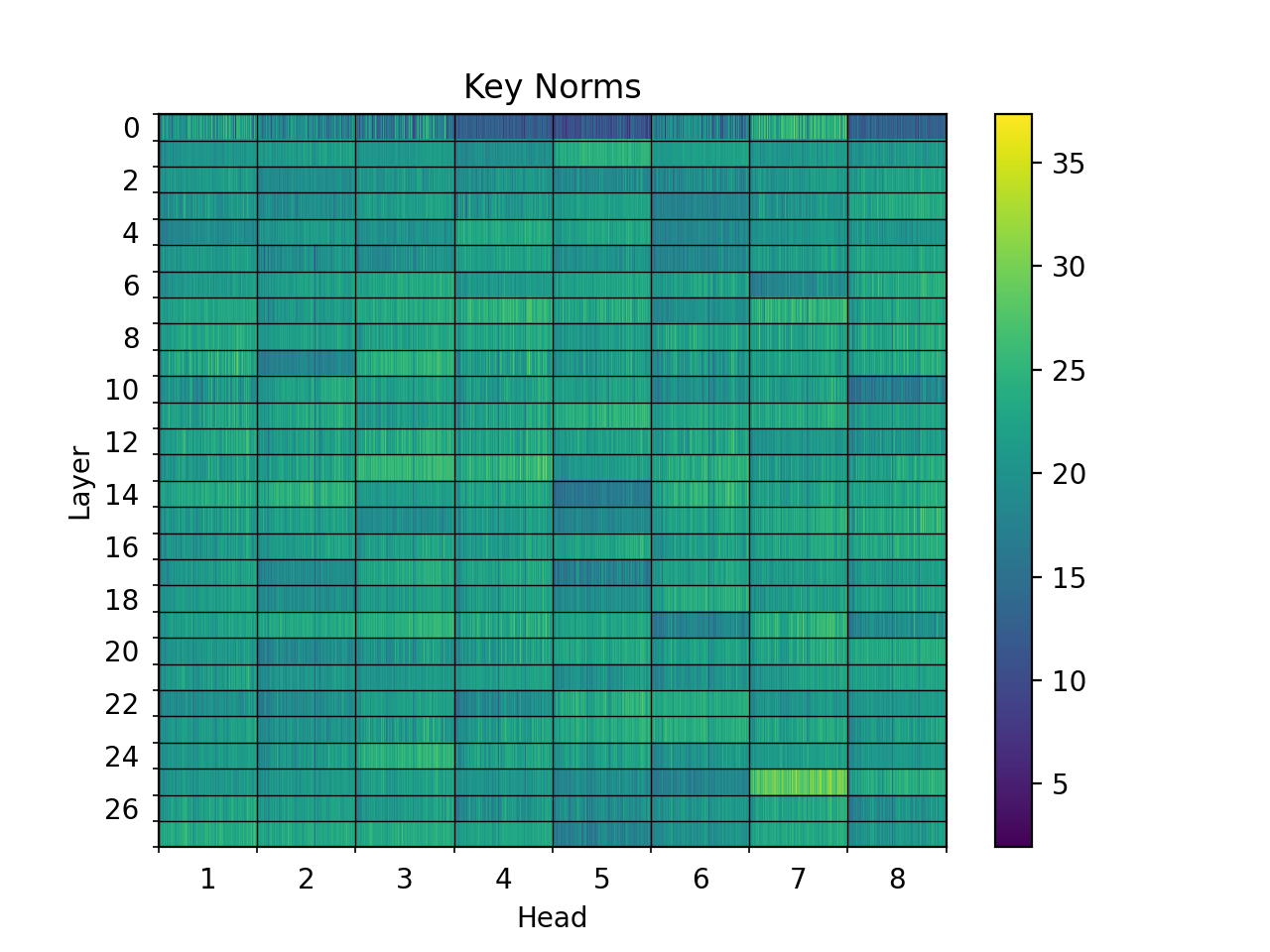}
     \end{subfigure}
     \begin{subfigure}[b]{0.49\linewidth}
         \centering
         \includegraphics[width=\linewidth]{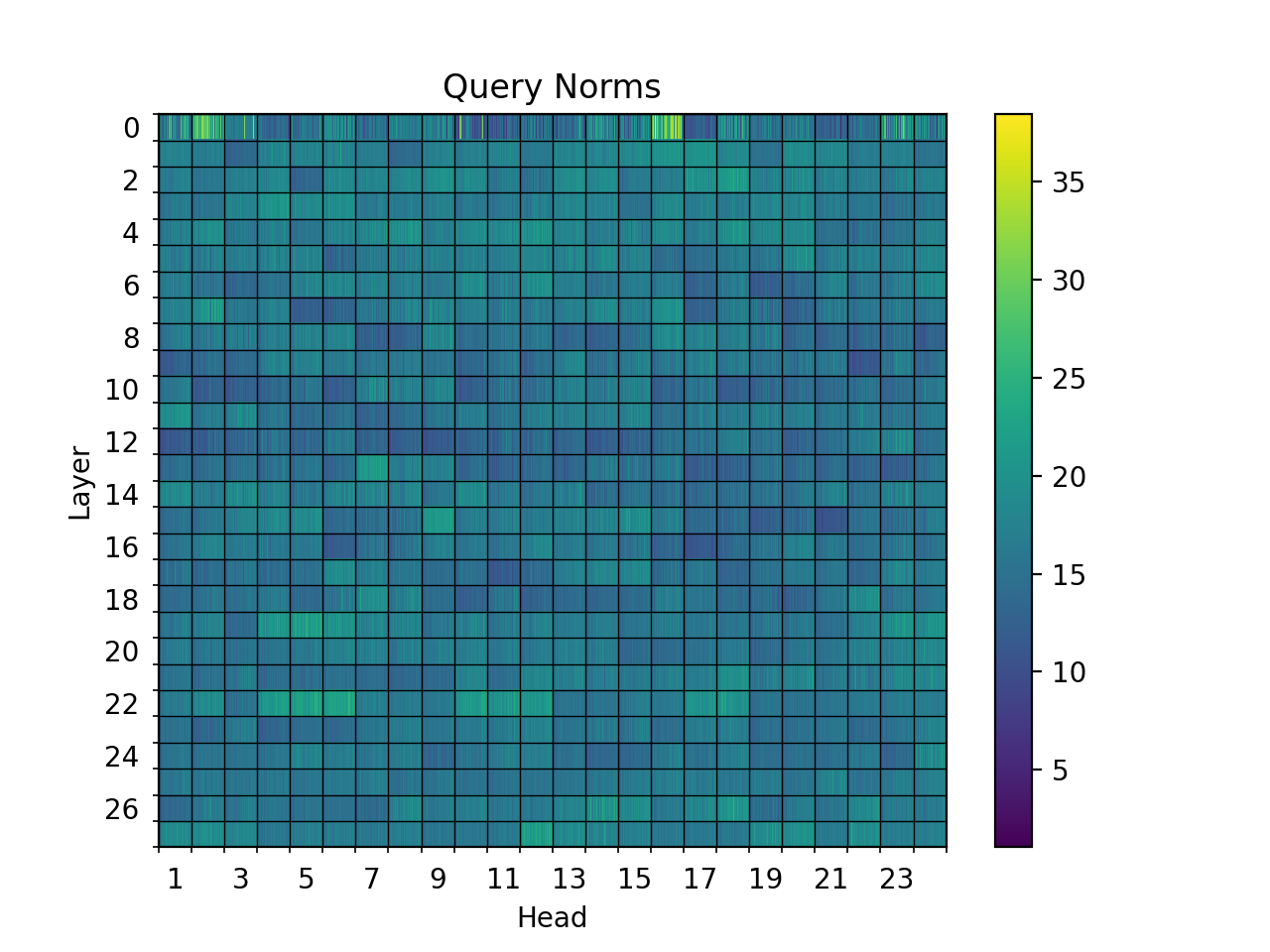}
     \end{subfigure}
     \hfill
     \caption{Distribution of $L^2$ norms for keys and queries across heads and layers.}
    \label{fig:key_query_l2_dist}
\end{figure*}

\begin{figure*}[ht!]
     \vspace{-2mm}
     \centering
         \includegraphics[width=.8\linewidth]{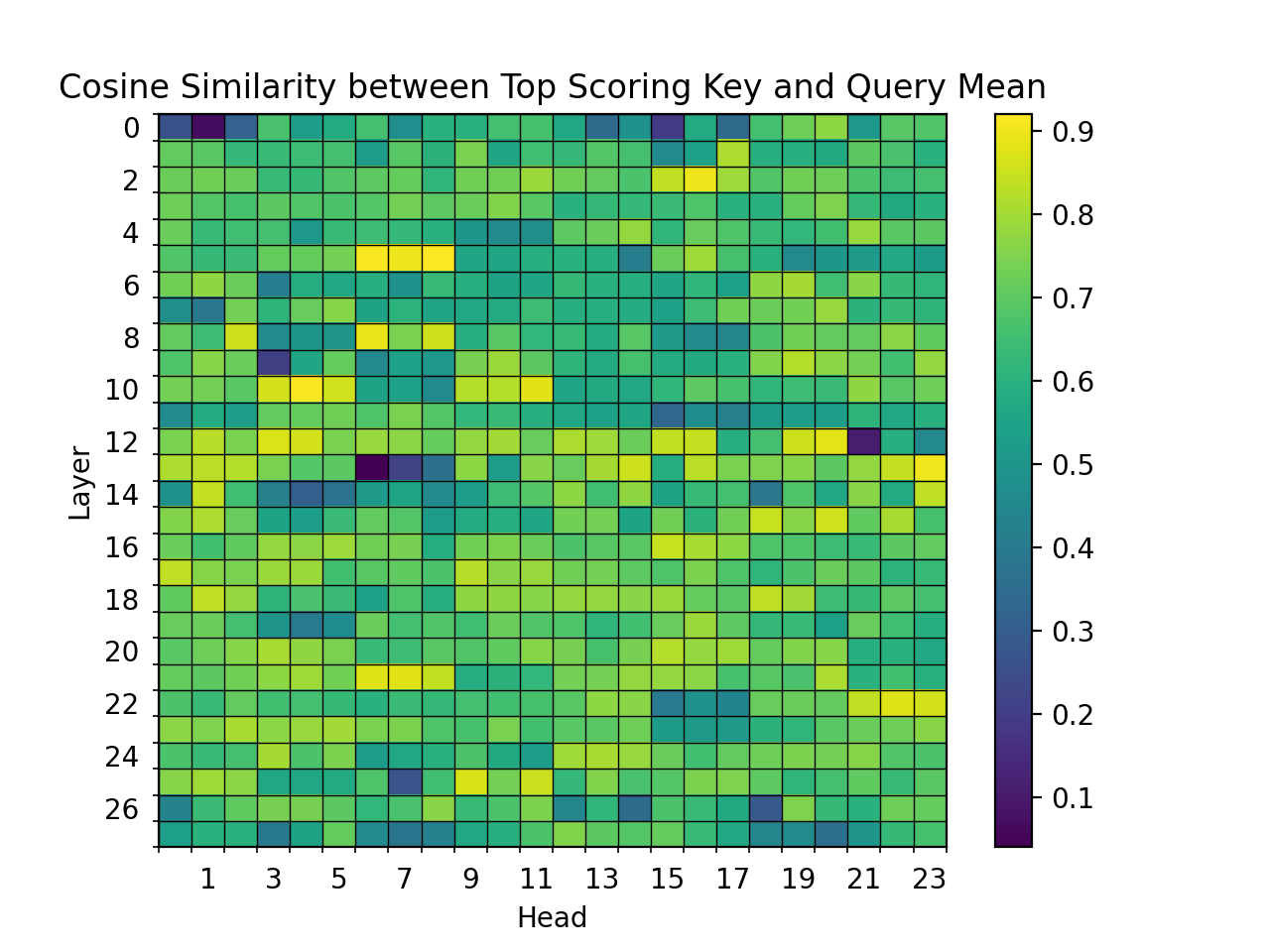}
     \vspace{-2mm}
     \caption{Cosine similarity between highest \ourmethod{} scoring key token with mean query.}
    \label{fig:keydiff_query_scores}
\end{figure*}

To verify the above hypothesis, we show \cref{fig:keydiff_query_scores} that keys which have maximum angular difference from the mean key are aligned with the mean query, resulting in very large attention weights. This demonstrates how LLMs exploit the geometry of the hidden states and projections to limit over mixing, and selectively identify important tokens. 

To summarize, we have for all attention heads in decoder based transformer models:
\begin{itemize}[leftmargin=*]
    \item \emph{The majority of keys and queries are approximately collinear with their mean.}
    \item \emph{Mean keys and mean queries have negative cosine similarity across all heads.}
    \item \emph{Most keys and queries have a similar L2 norm.}
    \item \emph{Decoder based attention heads can selectively increase attention weights for a fixed key by aligning it with the mean query.}
    \item \emph{Key token importance can hence be measured by the angular distance between a key and the mean key.}
\end{itemize}

We can show mathematically with some reasonable assumptions based on the above observations that key tokens with persistently high attention scores must be geometrically aligned with queries. 

\begin{theorem}Suppose that for a fixed query token $q$, there is a set of key tokens $\{k_i\}_{i=1}^n$ such that $||k_i||_2^2 < M, \; \forall \;i$. Without loss of generality suppose $||q||=1$, the scaling parameter is $1$ and assume ${k^*}$ is a key not in $\{k_i\}_{i=1}^n$ with $||{k^*}||_2^2 < M$ that has attention weight $w > 0$: 
\[w = \frac{\exp({k^*}^\top q)}{\exp({k^*}^\top q) + \sum\limits_{i=1}^n \exp(k_i^\top q)}.\] Then
\[\frac{ \log(\frac{n}{n+1}) -  \log(1-w)}{2M} -1  \leq \textrm{CosSim}({k^*}, q) \]
\end{theorem}
\begin{proof}
To show this we have that
\begin{align*}
    w &= \frac{\exp({k^*}^\top q)}{\exp({k^*}^\top q) + \sum\limits_{i=1}^n \exp(k_i^\top q)} \\
    w\left(\exp({k^*}^\top q) + \sum\limits_{i=1}^n \exp(k_i^\top q) \right)  &= \exp({k^*}^\top q) \\
    w\sum\limits_{i=1}^n \exp(k_i^\top q) &= (1-w) \exp({k^*}^\top q) \\
\intertext{Note that $-M \leq k_i^\top q \leq M$ and hence}
    w\sum\limits_{i=1}^n \exp(-M) &\leq (1-w) \exp({k^*}^\top q)\\
    wn\exp(-M) &\leq (1-w) \exp({k^*}^\top q) \\
    \frac{wn\exp(-M)}{1-w} & \leq \exp({k^*}^\top q)\\\
    \log(wn) - M - \log(1-w) & \leq M  \textrm{CosSim}({k^*}, q) \\
    \frac{\log(wn) - M -  \log(1-w)}{M} & \leq \textrm{CosSim}({k^*}, q) \\
    \frac{\log(\frac{\exp({k^*}^\top q)}{\exp({k^*}^\top q) + \sum\limits_{i=1}^n \exp(k_i^\top q)}n) - M -  \log(1-w)}{M} & \leq \textrm{CosSim}({k^*}, q) \\
    \frac{\log(\frac{\exp({k^*}^\top q)}{(n+1)\exp(M)}n) - M -  \log(1-w)}{M} & \leq \textrm{CosSim}({k^*}, q) \\
    \frac{-M\textrm{CosSim}(k^*, q) + \log(\frac{n}{n+1}) - 2M -  \log(1-w)}{M} & \leq \textrm{CosSim}({k^*}, q) \\
    \frac{ \log(\frac{n}{n+1}) - 2M -  \log(1-w)}{M} & \leq 2\textrm{CosSim}({k^*}, q) \\
    \frac{ \log(\frac{n}{n+1}) -  \log(1-w)}{2M} -1 & \leq \textrm{CosSim}({k^*}, q) \\
\end{align*}
Taking the limit as $n\to \infty$ produces \cref{lemma:key-cossim-reduced}
\end{proof}

The above proof demonstrates that as long as the norms of the keys are bounded, in order for an attention head to be able to freely allocate $w$ attention weight to a given key ${k^*}$, the cosine similarity between ${k^*}$ and $q$ must be high, therefore ${k^*}$ and $q$ must be approximately collinear.

Using this result, we can also show that as long as the cosine similarity between ${k^*}$ and $q$ is high, while the cosine similarity between $\bar{k}$ and $q$ is low, $\textrm{CosSim}({k^*},\bar{k})$ is small. Note that, since empirical results demonstrate that most keys have high cosine similarity with their mean $\bar{k}$, a key with high importance, approximately collinear to $q$, will also have low cosine similarity to $\bar{k}$. Generally, this also suggests that key tokens with low cosine similarity to $\bar{k}$ have greater importance.

In order to show this, we need the following auxiliary result.
\begin{lemma}\label{lemma:orthsum}
    Suppose $\{x_1, ..., x_n\}$ is an orthonormal basis of $\mathbb{R}^n$ and $y \in \mathbb{R}^n$. Define $\alpha_i = \textrm{CosSim}(x_i, y)$. Then $\sum\limits_{i=1}^n \alpha_i^2 = 1$.
\end{lemma}
\begin{proof}
    Note that $\alpha_i = \frac{y^\top x_i}{||y|| ||x_i||} = \frac{y^\top x_i}{||y||}$. If we expand $\frac{y}{||y||}$ in the basis $\{x_1, ..., x_n\}$ we see that 
    \begin{align*}
        \frac{y}{||y||} &= \sum\limits_{i=1}^n \left({\frac{y}{||y||}}^\top x_i\right) x_i\\
        &= \sum\limits_{i=1}^n \alpha_i x_i \\
    \end{align*}
    But then, we know that since $\left|\left|\frac{y}{||y||} \right|\right|_2^2 = 1$, then $\left\langle \sum\limits_{i=1}^n \alpha_i x_i ,  \sum\limits_{i=1}^n \alpha_i x_i \right\rangle = 1$. But we have
    \begin{align*}
    \left\langle \sum\limits_{i=1}^n \alpha_i x_i ,  \sum\limits_{i=1}^n \alpha_i x_i \right\rangle & = \sum\limits_{i=1}^n\sum\limits_{j=1}^n \langle\alpha_ix_i, \alpha_j x_j \rangle \\
    &= \sum\limits_{i=1}^n \langle\alpha_ix_i, \alpha_i x_i \rangle \\
    & = \sum\limits_{i=1}^n \alpha_i^2
    \end{align*}
    proving the result.
\end{proof}

\begin{theorem}\label{theorem:keydiff}
Consider tokens ${k^*}$, $q$, $\bar{k}$ as above where $\bar{k}$ is the average of the keys tokens. Suppose $\textrm{CosSim}({k^*}, q) = \beta_q > 0$ and $\textrm{CosSim}(\bar{k}, q) = \alpha_q < 0$. Then $\textrm{CosSim}(\bar{k}, {k^*}) \leq 1 + \alpha_q \beta_q - \frac{1}{2}\alpha_q^2 -\frac{1}{2}\beta_q^2$.
\end{theorem}
\begin{proof}
Consider the cosine similarity of $\bar{k}$ and ${k^*}$:
\begin{align*}
    \textrm{CosSim}(\bar{k}, {k^*}) &= \frac{{k^*}^\top\bar{k}}{||{k^*}|| ||\bar{k}||}
\end{align*}
expand $\bar{k}$ in an orthonormal basis which contains $q$, $\{q,r_1, ..., r_{n-1}\}$ such that $$\bar{k} = ||\bar{k}|| \left(\alpha_q  q +  \sum\limits_{i=1}^{n-1} \alpha_i r_i \right)$$ where $\alpha_i = \textrm{CosSim}(\bar{k}, r_i)$. Additionally, define $\beta_i = \textrm{CosSim}({k^*}, r_i)$ and note that by the definition of an orthonormal basis and the cosine similarity operation, using the result from \cref{lemma:orthsum} we have that $\alpha_q^2 + \sum\limits_{i=1}^{n-1}\alpha_i^2 = 1$ and that $\beta_q^2 + \sum\limits_{i=1}^{n-1}\beta_i^2 = 1$. Now we have that 

\begin{align*}
    \frac{{k^*}^\top\bar{k}}{||{k^*}|| ||\bar{k}||} & = \frac{{k^*}^\top \left(||\bar{k}||\alpha_q  q +  \sum\limits_{i=1}^{n-1} ||\bar{k}|| \alpha_i r_i \right)}{||{k^*}|| ||\bar{k}||} \\
    &= \alpha_q \beta_q + \frac{1}{||{k^*}||} \sum\limits_{i=1}^{n-1} \alpha_i {k^*}^\top r_i \\
    & = \alpha_q \beta_q + \sum\limits_{i=1}^{n-1} \alpha_i \beta_i \\
    & \leq \alpha_q \beta_q + \sum\limits_{i=1}^{n-1} |\alpha_i| |\beta_i| \\
\intertext{Applying Young's inequality we obtain}
    & \leq \alpha_q \beta_q + \frac{1}{2}\sum\limits_{i=1}^{n-1} \alpha_q^2 + \beta_q^2 \\
    & = \alpha_q \beta_q +  \frac{1}{2}(1-\alpha_q^2) + \frac{1}{2}(1-\beta_q^2)\\
    & = 1 + \alpha_q \beta_q - \frac{1}{2}\alpha_q^2 -\frac{1}{2}\beta_q^2 \\
\end{align*}
\end{proof}
Note that on the domain $\alpha_q \in [-1,0)$, $\beta_q \in (0,1]$ the function $1 + \alpha_q \beta_q - \frac{1}{2}\alpha_q^2 -\frac{1}{2}\beta_q^2$ is bounded above by $1$ and decreasing to $-\frac{1}{2}$ as $\alpha \rightarrow -1$. Hence, the smaller $\textrm{CosSim}(\bar{k}, q) = \alpha_q$ is, the smaller $\textrm{CosSim}(\bar{k}, {k^*})$ must be. 

\subsection{Correlation Analysis}
\label{appendix:subsec:correlation-analysis}

We report the Spearman rank correlation between key cosine similarity and attention scores for each layer of the Llama-3.2-3B-Instruct model. Correlations are averaged over randomly sampled LongBench-Musique prompts. From \cref{tab:spearman_correlation}, we observe a consistently high correlation ($\rho \approx 0.94$ on average)
across all layers, indicating that geometrically distinctive keys
(i.e., those with low pairwise cosine similarity) are strongly aligned with
tokens receiving higher attention scores.
This empirical evidence supports our theoretical claim in \cref{subsec:keydiff-theory}
that key diversity serves as a reliable proxy for token importance.

\begin{table}[t]
\centering
\caption{Spearman correlation ($\rho$) between negative key cosine similarity and attention scores for each transformer layer of Llama-3.2-3B-Instruct.}
\label{tab:spearman_correlation}
\scriptsize
\begin{tabular}{cccccccc}
\toprule
Layer & 1 & 2 & 3 & 4 & 5 & 6 & 7 \\
\midrule
$\rho$ & 0.8997 & 0.9561 & 0.9332 & 0.9496 & 0.9517 & 0.9484 & 0.9587 \\
\midrule
Layer & 8 & 9 & 10 & 11 & 12 & 13 & 14 \\
\midrule
$\rho$ & 0.9578 & 0.9534 & 0.9570 & 0.9628 & 0.9554 & 0.9658 & 0.9578 \\
\midrule
Layer & 15 & 16 & 17 & 18 & 19 & 20 & 21 \\
\midrule
$\rho$ & 0.9477 & 0.9538 & 0.9280 & 0.9373 & 0.9328 & 0.9340 & 0.9311 \\
\midrule
Layer & 22 & 23 & 24 & \multicolumn{4}{c}{Mean $\pm$ Std} \\
\midrule
$\rho$ & 0.9140 & 0.9060 & 0.8950 & \multicolumn{4}{c}{$0.94 \pm 0.02$} \\
\bottomrule
\end{tabular}
\end{table}

%% file: contents/A02_TTFT.tex
\section{TTFT Analysis}
We have measured end-to-end inference latency (measured as time to first token (TTFT)) for Llama 3.2-3B via the standard Huggingface API when the model is using eager attention and FlashAttention as in \cref{fig:ttft_eager} and \cref{fig:ttft_flash}.  Tests are performed on NVIDIA A100 80GB GPUs. We test various block sizes and KV cache budgets. KeyDiff outperforms TOVA and SnapKV with FlashAttention as well as with eager attention for large cache budgets. We can see inference latency with KeyDiff is independent of block size when using FlashAttention because of KeyDiff's linear complexity and its lack of required materialized attention weights.

\begin{figure*}[ht!]
     \centering
     \begin{subfigure}[b]{0.99\linewidth}
         \centering
         \includegraphics[width=\linewidth]{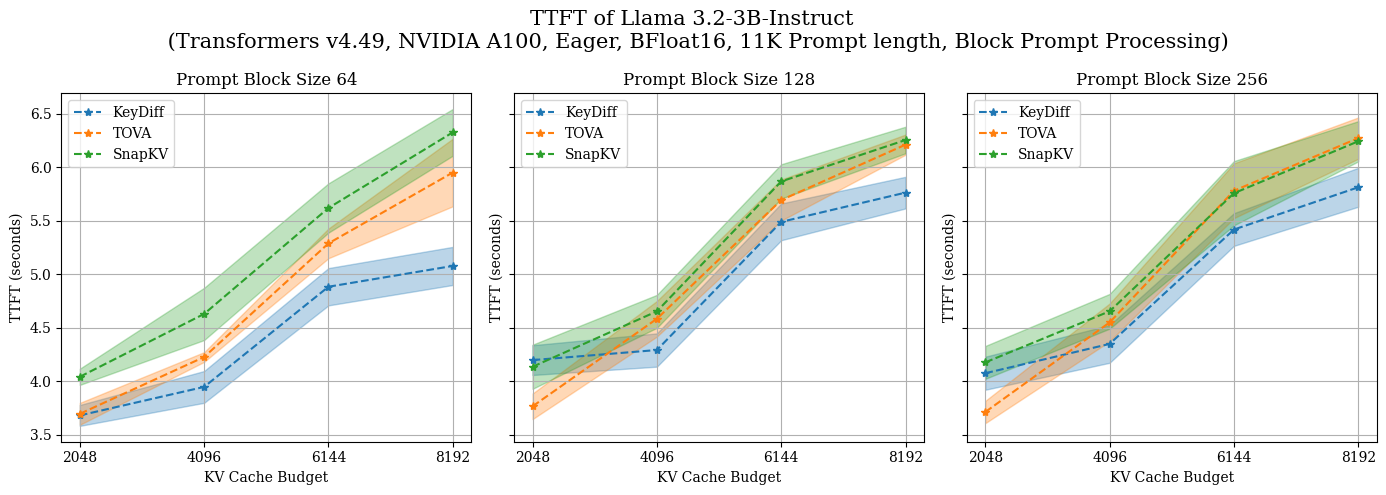}
     \end{subfigure}
     \hfill
     \vspace{-2mm}
     \caption{TTFT for eager attention with different cache eviction strategies using block size $64$, $128$, $256$ for block prompt processing.}
    \label{fig:ttft_eager}
\end{figure*}

%% file: contents/A03_Reasoning.tex
\section{Math-500 reasoning benchmark}
\label{appendix:math500}
In order to measure the effectiveness of different caching methods on reasoning tasks, we tested several different model using various caching algorithms on the Math 500 reasoning benchmark. Specifically, we test \ourmethod{} with a sliding window whose window size is $20\%$ of the KV cache budget, and SnapKV on the DeepSeek-R1-Distill-Llama-8B and DeepSeek-R1-Distill-Qwen-7B models. We randomly sample 5 responses with $\textrm{TopP}=0.95$, $\textrm{Temperature}=0.95$ with the 4096, 6144, and 32,786 max generation lengths. The reported scores are the average of accuracies over the random samples. 

\cref{table:math500-llama-8b} summarizes the Math-500 evaluation results for Llama-8B. As shown in the table, token eviction methods generally perform well even with KV cache budgets that are strictly smaller than the maximum generation length. Surprisingly, we found that \ourmethod{} slightly outperforms baseline methods in certain configurations (e.g., \ourmethod{} with a 2K budget for 4K generation length, and \ourmethod{} with a 4K budget for 8K generation length).

To further analyze the effectiveness of token eviction, we measure accuracy in cases where the context length of the baseline method (i.e., prompt length + generation length) exceeds the available KV cache budget. As shown in \cref{table:math500-llama-8b-split}, for samples where eviction is actively triggered, \ourmethod{} continues to outperform the token eviction baseline (SnapKV), and often achieves accuracy close to or slightly better than the non-evicting baseline.

We also conducted a similar evaluation on DeepSeek-R1-Distill-Qwen-7B and observed a slight performance degradation for token eviction methods compared to full KV cache baselines (See \cref{table:math500-qwen2.5-7b,table:math500-qwen-2.5-7b-split}.) However, \ourmethod{} still demonstrates comparable performance to SnapKV overall. This discrepancy may stem from architectural differences that Llama uses a lower GQA \cite{ainslie2023gqa} ratio than Qwen, which results in more information compression in the KV cache. We hypothesize that models with more compression like Qwen are more sensitive to eviction since each evicted token contains more information in Qwen than Llama by design.

\input{tables/math500/ds-r1-llama-8b}

\input{tables/math500/ds-r1-llama-8b-split}
\input{tables/math500/ds-r1-qwen2.5-7b}
\input{tables/math500/ds-r1-qwen2.5-7b-split}

%% file: tables/math500/ds-r1-llama-8b.tex
\begin{table*}[!t]
\caption{\textbf{Math 500 results on DeepSeek-R1-Llama-8B distilled model (Higher is better)}. We highlight the methods showing the best performance within a given budget with \textbf{boldface}. }
\label{table:math500-llama-8b}
\begin{center}
\resizebox{.8\textwidth}{0.18\textheight}{%
\begin{tabular}{cccccc}
\toprule

%%%%%%%%%%%%%%%%%%%%%%%%%%%%%%%%%%%%%%%%%%%%%%%%%%%%%%%%
%%%%%%%%%%%%%%%%%%%%% Header Rows %%%%%%%%%%%%%%%%%%%%%&
%%%%%%%%%%%%%%%%%%%%%%%%%%%%%%%%%%%%%%%%%%%%%%%%%%%%%%%%

Method & Max Gen. Length & Budget & Flex Match & Exact Match & Avg. Gen. Length \\
\midrule
\midrule
\addlinespace
Full & 4K& N/A & 0.711 & 0.537 & 2769 \\
\midrule
\multirow{2}{*}{KeyDiff} &
\multirow{2}{*}{4K} 
  & 1024 & \bf{0.695} & 0.531 & 2753 \\
& & 2048 & \bf{0.720} & 0.546 & 2740 \\
\midrule
\multirow{2}{*}{SnapKV}  &
\multirow{2}{*}{4K}  
   & 1024 & 0.689 & 0.529 & 2759 \\
 & & 2048 & 0.714 & 0.544 & 2757 \\
\midrule
\midrule
\addlinespace

Full & 8K& N/A & 0.840 & 0.628 & 3812 \\
\midrule
\multirow{2}{*}{KeyDiff} &
\multirow{2}{*}{8K} 
  & 2048 & \bf{0.819} & 0.617 & 3888\\
& & 4096 & \bf{0.844} & 0.634 & 3805\\
\midrule
\multirow{2}{*}{SnapKV}  &
\multirow{2}{*}{8K}  
   & 2048 & 0.805 & 0.610 & 3898\\
 & & 4096 & 0.828 & 0.627 & 3898 \\
\midrule
\midrule
\addlinespace

Full & 32K& N/A & 0.898 & 0.668 & 6869\\
\midrule
\multirow{3}{*}{KeyDiff} &
\multirow{3}{*}{32K} 
   & 2048 & \bf{0.883} & 0.662 & 7678 \\
 & & 4096 & \bf{0.894} & 0.668 & 7312 \\
 & & 8192 & \bf{0.894} & 0.667 & 7096 \\
\midrule
\multirow{3}{*}{SnapKV}  &
\multirow{3}{*}{32K}  
   & 2048 & 0.849 & 0.641 & 7509 \\
 & & 4096 & 0.884 & 0.661 & 7218 \\
 & & 8192 & 0.893 & 0.665 & 7005 \\

\bottomrule
\end{tabular}
}
\end{center}
\end{table*}

%% file: tables/math500/ds-r1-llama-8b-split.tex
\begin{table*}[!t]
\caption{\textbf{Math 500 results on DeepSeek-R1-Llama-8B distilled model (Higher is better)}. We highlight the methods showing the best performance within a given budget with \textbf{boldface}. }
\label{table:math500-llama-8b-split}
\begin{center}
\begin{tabular}{cccc}
\toprule

\multicolumn{4}{c}{Max Gen. Length $=$ 4K} \\
\midrule
\midrule
\multicolumn{4}{c}{Num Tokens $>$ 1K (497/500 samples)} \\
\midrule
 &  Budget & Flex & Exact   \\
\midrule
Full &  N/A & 0.709 & 0.534  \\ 
\midrule
   KeyDiff & 1024 & \bf{0.693} & \bf{0.528} \\
   SnapKV  & 1024 &      0.687 &      0.526 \\
\midrule
\midrule
\multicolumn{4}{c}{Num Tokens $>$ 2K (240/500 samples)} \\
\midrule
 Full    &   N/A &      0.604 &      0.449 \\
\midrule
 KeyDiff &  2048 & \bf{0.618} & \bf{0.463} \\
 SnapKV  &  2048 &      0.610 &      0.458 \\

\bottomrule
\end{tabular}
\begin{tabular}{cccc}
\toprule

\multicolumn{4}{c}{Max Gen. Length $=$ 8K} \\
\midrule
\midrule
\multicolumn{4}{c}{Num Tokens $>$ 2K (353/500 samples)} \\
\midrule
 &  Budget & Flex & Exact   \\
\midrule
Full &  N/A & 0.783 & 0.584  \\ 
\midrule
   KeyDiff & 2048 & \bf{0.756} & \bf{0.570} \\
   SnapKV  & 2048 &      0.736 &      0.560 \\
\midrule
\midrule
\multicolumn{4}{c}{Num Tokens $>$ 4K (195/500 samples)} \\
\midrule
 Full    &   N/A &      0.650 &      0.455 \\
\midrule
 KeyDiff &  4096 & \bf{0.662} & \bf{0.470} \\
 SnapKV  &  4096 &      0.637 &      0.453 \\

\bottomrule
\end{tabular}
\begin{tabular}{cccc}
\toprule

\multicolumn{4}{c}{Max Gen. Length $=$ 32K} \\
\midrule
\midrule
\multicolumn{4}{c}{Num Tokens $>$ 2K (353/500 samples)} \\
\midrule
 &  Budget & Flex & Exact   \\
\midrule
Full &  N/A & 0.783 & 0.584  \\ 
\midrule
   KeyDiff & 2048 & \bf{0.756} & \bf{0.570} \\
   SnapKV  & 2048 &      0.736 &      0.560 \\
\midrule
\midrule
\multicolumn{4}{c}{Num Tokens $>$ 4K (195/500 samples)} \\
\midrule
 Full    &   N/A &      0.650 &      0.455 \\
\midrule
 KeyDiff &  4096 & \bf{0.662} & \bf{0.470} \\
 SnapKV  &  4096 &      0.637 &      0.453 \\

\midrule
\midrule
\multicolumn{4}{c}{Num Tokens $>$ 8K (162/500 samples)} \\
\midrule
 Full    &   N/A &      0.793 &      0.550 \\
\midrule
 KeyDiff &  8192 & \bf{0.786} & \bf{0.545} \\
 SnapKV  &  8192 &      0.782 &      0.541 \\

\bottomrule
\end{tabular}
\end{center}
\end{table*}

%% file: tables/math500/ds-r1-qwen2.5-7b.tex
\begin{table*}[!t]
\caption{\textbf{Math 500 results on DeepSeek-R1-Qwen-7B distilled model (Higher is better)}. We highlight the methods showing the best performance within a given budget with \textbf{boldface}. }
\label{table:math500-qwen2.5-7b}
\begin{center}
\resizebox{.8\textwidth}{0.18\textheight}{%
\begin{tabular}{cccccc}
\toprule

%%%%%%%%%%%%%%%%%%%%%%%%%%%%%%%%%%%%%%%%%%%%%%%%%%%%%%%%
%%%%%%%%%%%%%%%%%%%%% Header Rows %%%%%%%%%%%%%%%%%%%%%&
%%%%%%%%%%%%%%%%%%%%%%%%%%%%%%%%%%%%%%%%%%%%%%%%%%%%%%%%

Method & Max Gen. Length & Budget & Flex Match & Exact Match & Avg. Gen. Length \\
\midrule
\midrule
\addlinespace
Full & 4K& N/A & 0.764 & 0.579 & 2630 \\
\midrule
\multirow{2}{*}{KeyDiff} &
\multirow{2}{*}{4K} 
  & 1024 & 0.666 & 0.512 & 2692 \\
& & 2048 & \bf{0.749} & 0.570 & 2629 \\
\midrule
\multirow{2}{*}{SnapKV}  &
\multirow{2}{*}{4K}  
   & 1024 & \bf{0.692} & 0.533 & 2655 \\
 & & 2048 & \bf{0.749} & 0.566 & 2637 \\
\midrule
\midrule
\addlinespace

Full & & N/A & 0.877 & 0.658 & 3287 \\
\midrule
\multirow{2}{*}{KeyDiff} &
\multirow{2}{*}{8K} 
  & 2048 & 0.811 & 0.613 & 3348 \\
& & 4096 & 0.867 & 0.647 & 3208 \\
\midrule
\multirow{2}{*}{SnapKV}  &
\multirow{2}{*}{8K}  
   & 2048 & \bf{0.812} & 0.612 & 3328 \\
 & & 4096 & \bf{0.868} & 0.647 & 3214 \\
\midrule
\midrule
\addlinespace

Full & 32K& N/A & 0.923 & 0.682 & 4051\\
\midrule
\multirow{3}{*}{KeyDiff} &
\multirow{3}{*}{32K} 
   & 2048 & 0.811 & 0.613 & 4322 \\
 & & 4096 & \bf{0.897} & 0.647 & 3800 \\
 & & 8192 & \bf{0.891} & 0.663 & 3741 \\
\midrule
\multirow{3}{*}{SnapKV}  &
\multirow{3}{*}{32K}  
   & 2048 & \bf{0.812} & 0.612 & 4279 \\
 & & 4096 & 0.868 & 0.647 & 3828 \\
 & & 8192 & \bf{0.891} & 0.662 & 3808 \\

\bottomrule
\end{tabular}
}
\end{center}
\end{table*}

%% file: tables/math500/ds-r1-qwen2.5-7b-split.tex
\begin{table*}[!t]
\caption{\textbf{Math 500 results on DeepSeek-R1-Qwen-7B distilled model (Higher is better)}. We highlight the methods showing the best performance within a given budget with \textbf{boldface}. }
\label{table:math500-qwen-2.5-7b-split}
\begin{center}
\begin{tabular}{cccc}
\toprule

\multicolumn{4}{c}{Max Gen. Length $=$ 4K} \\
\midrule
\midrule
\multicolumn{4}{c}{Num Tokens $>$ 1K (497/500 samples)} \\
\midrule
 &  Budget & Flex & Exact   \\
\midrule
Full &  N/A &                0.762 &     0.578  \\ 
\midrule
   KeyDiff & 1024 &          0.664 &      0.511 \\
   SnapKV  & 1024 &     \bf{0.690} &      0.532 \\
\midrule
\midrule
\multicolumn{4}{c}{Num Tokens $>$ 2K (328/500 samples)} \\
\midrule
 Full    &   N/A &      0.650 &      0.478 \\
\midrule
 KeyDiff &  2048 &      0.628 &      0.464 \\
 SnapKV  &  2048 &   \bf{0.629} &      0.459 \\

\bottomrule
\end{tabular}
\begin{tabular}{cccc}
\toprule

\multicolumn{4}{c}{Max Gen. Length $=$ 8K} \\
\midrule
\midrule
\multicolumn{4}{c}{Num Tokens $>$ 2K (317/500 samples)} \\
\midrule
 &  Budget & Flex & Exact   \\
\midrule
Full &  N/A & 0.816 & 0.603  \\ 
\midrule
   KeyDiff & 2048 &         0.713 &      0.533 \\
   SnapKV  & 2048 &    \bf{0.718} &      0.533 \\
\midrule
\midrule
\multicolumn{4}{c}{Num Tokens $>$ 4K (179/500 samples)} \\
\midrule
 Full    &   N/A &           0.626 &      0.441 \\
\midrule                    
 KeyDiff &  4096 &           0.611 &      0.419 \\
 SnapKV  &  4096 &      \bf{0.615} &      0.422 \\

\bottomrule
\end{tabular}
\begin{tabular}{cccc}
\toprule

\multicolumn{4}{c}{Max Gen. Length $=$ 32K} \\
\midrule
\midrule
\multicolumn{4}{c}{Num Tokens $>$ 2K (360/500 samples)} \\
\midrule
 &  Budget & Flex & Exact   \\
\midrule
Full &  N/A & 0.889 & 0.642   \\ 
\midrule
   KeyDiff & 2048 &      0.713 & 0.533  \\
   SnapKV  & 2048 & \bf{0.718} & 0.533  \\
\midrule
\midrule
\multicolumn{4}{c}{Num Tokens $>$ 4K (141/500 samples)} \\
\midrule
 Full    &   N/A &      0.787 &      0.524 \\
\midrule
 KeyDiff &  4096 &      0.611 &      0.419 \\
 SnapKV  &  4096 & \bf{0.615} &      0.422 \\

\midrule
\midrule
\multicolumn{4}{c}{Num Tokens $>$ 8K (53/500 samples)} \\
\midrule
 Full    &   N/A &      0.550 &      0.275 \\
\midrule
 KeyDiff &  8192 & \bf{0.392} &      0.222 \\
 SnapKV  &  8192 &      0.381 &      0.215 \\

\bottomrule
\end{tabular}
\end{center}
\end{table*}

%% file: contents/A04_longbench.tex
\section{Additional discussion for LongBench}
\label{appendix:longbench-extended}

\subsection{Empirical Motivation for \ourmethod{} Setup}
\label{appendix:subsec:sec3-setup}

To generate \cref{fig:pca-plot}, we used the first sample from the test split of the narrativeqa task in LongBench to prefill the KV cache of Llama3.2-3B-Instruct with a block size of $B=128$.
The KV cache had a maximum size of $4096$ while the sample was much longer, requiring KV eviction.  
We applied PCA to the key cache and repeated the process for sink attention, TOVA and \ourmethod{}. 

To construct \cref{fig:gram_matrix}, we sample 100 prompts from the \texttt{Qasper} dataset in LongBench \cite{bai-etal-2024-longbench}, compute the log determinant of $KK^T$ of the keys in the KV caches of each head and layer of Llama 3.2-3B-Instruct using a cache budget of $N=2048$ and a block size of $B=128$, and plot the distribution in \cref{fig:gram_matrix}.
We show this key distribution for sink attention, TOVA and \ourmethod{}. 
\subsection{Experiment Setup}
\label{appendix:subsec:longbench-setup}

In this subsection, we provide the experimental setup for \ourmethod{} and the baselines for the LongBench experiments. The LongBench evaluation is conducted using the default parameters of the LongBench evaluator with predefined prompt templates. Tests are performed on NVIDIA A100 80GB GPUs. 

For TOVA, H2O, and SnapKV, the set of attention weights computed from a single key cache due to grouped query attention \cite{ainslie2023gqa} is aggregated by taking the average over the attention weights. Additionally, only for SnapKV, we apply average smoothing to the attention score with a kernel size of 7 and keep the most recent 32 tokens in the cache, following the suggestion of the original paper. For Sink, we used the first four tokens as the attention sink, following the suggestion of the original paper.

\subsection{Longbench dataset statistics}
\label{appendix:subsec:longbench-statistics}

In this section, we provide the length statistics of the Longbench Benchmark and in-depth analysis of compression ratios for the given KV cache budgets, such as 2K, 4K, 6K, and 8K. 

\paragraph{Prompt lengths} We measure the number of tokens in the samples using LLama tokenizer \cite{touvron2023llama}. As shown in \cref{fig:longbench-statistics}, LongBench exhibits variability in sample length from the datasets.

\begin{figure}[h]
    \centering
    \includegraphics[width=1.0\linewidth]{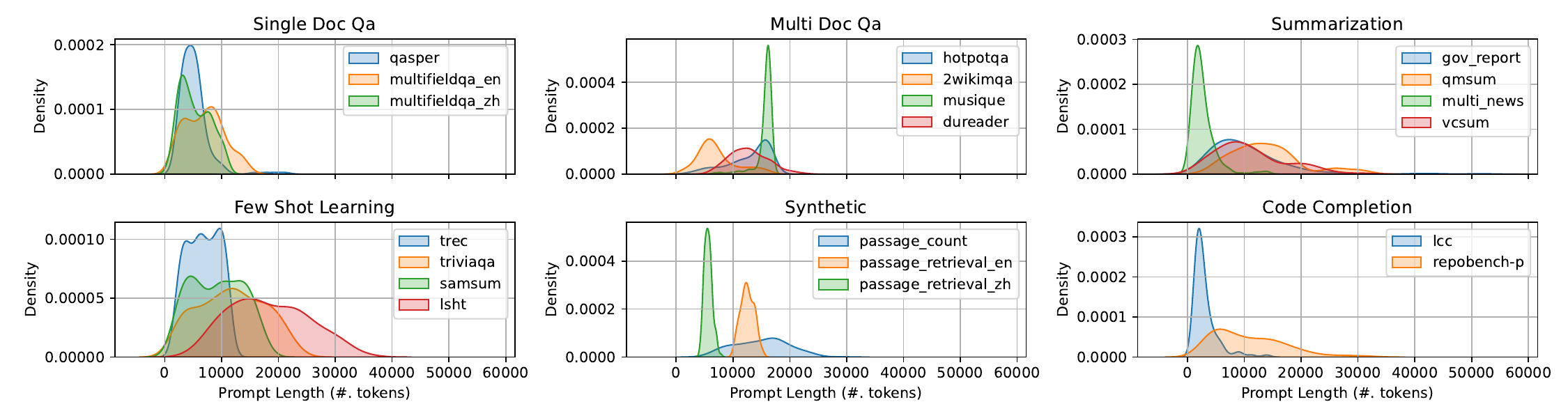}
    \caption{\textbf{Histograms of sample lengths measured by number of tokens}}
    \label{fig:longbench-statistics}
\end{figure}

\paragraph{Compression ratio} The majority of other KV cache eviction studies assume an unconstrained memory footprint. Before they compress the cache by applying an eviction policy, they first set the target compression ratio and evict the appropriate number of KV pairs to satisfy the compression ratio \cite{zhang2024h2o, oren2024transformers, li2024snapkv}. On the other hand, we fix the KV cache size and ensure the number of cached tokens is less than or equal to the predefined cache size. Due to these differences, it is less straightforward to set appropriate KV cache budgets to satisfy the target compression ratios. Instead, we provide the average compression ratio, which is defined as: 
\begin{equation*} 
\text{Average Compression Ratio} = \frac{1}{I}\sum_{i=1}^{I} \frac{N}{L_i},
\end{equation*} 
where $N$ is the KV cache budget, and $L_i$ is the length of the $i$-th prompt (sample). We replace the summand with 1 whenever $N \geq L_{i}$, as compression doesn't occur in that setting.

As summarized in \cref{table:compression_ratio}, 2K cache budgets have a 0.31 average compression ratio, which indicates 69\% of input prompts are compressed. Our largest cache budget, 8K, exhibits a 0.77 average compression ratio.

\input{tables/longbench_lengths}
\input{tables/longbench_compression_ratios}

\subsection{Additional Results}
\label{appendix:subsec:longbench-result}

\input{tables/longbench_unrestricted}
\input{tables/longbench_full_rotated}

\input{tables/longbench_full_rotated_qwen}

\input{tables/longbench_full_rotated_llama3.2-3b_blocksize_ablation}

\input{tables/longbench_full_rotated_llama3.x_keydiff-recent-tokens}

%% file: tables/longbench_lengths.tex
\begin{table}[ht!]
    \centering
    \resizebox{0.75\textwidth}{!}{
    \begin{tabular}{r|ccccc|c}
        \toprule
         & $\leq$ 2K & 2K $\leq L \leq$ 4K & 4K $\leq L \leq$ 6K & 6K $\leq L \leq$ 8K & $\geq$ 8K & Total \\ 
        \hline
        NarrativeQA           & 0   & 0   & 0   & 8   & 192 & 200  \\
        Qasper                & 1   & 77  & 83  & 25  & 14  & 200  \\
        MultifidelityQA-En    & 9   & 31  & 21  & 32  & 57  & 150  \\
        MultifidelityQA-Zh    & 14  & 69  & 47  & 33  & 37  & 200  \\
        \hline
        HotPotQA              & 1   & 4   & 12  & 12  & 171 & 200  \\
        2wikimqa              & 8   & 17  & 68  & 54  & 53 & 200  \\
        musique               & 0   & 0   & 0   & 3   & 197 & 200  \\
        dureader              & 0   & 0   & 0   & 16  & 184 & 200  \\
        \hline
        gov report            & 0   & 20  & 29  & 45  & 106 & 200  \\
        qmsum                 & 0   & 1   & 17  & 15  & 167 & 200  \\
        multi news            & 99  & 71  & 19  & 6   & 5   & 200  \\
        vcsum                 & 4   & 20  & 18  & 32  & 126 & 200  \\
        \hline
        trec                  & 4  & 43  & 41   & 39  & 73  & 200  \\
        triviaqa              & 4  & 21  & 15   & 21  & 139 & 200  \\
        samsum                & 6  & 29  & 34   & 17  & 114 & 200  \\
        lsht                  & 0  & 0   & 3    & 8   & 189 & 200  \\
        \hline
        passage count         & 0  & 0   & 3    & 13  & 184 & 200  \\
        passage-retrieval-En  & 0  & 0   & 0    & 0   & 200 & 200  \\
        passage-retrieval-Zh  & 0  & 0   & 160  & 10  & 0   & 170  \\
        \hline
        lcc                   & 80 & 86  & 21   & 4   & 9   & 200  \\
        repobench-p           & 0  & 25  & 37   & 25  & 113 & 200  \\
        \bottomrule
        
    \end{tabular}
    }
    \caption{\textbf{Distribution of sample length measured by Llama2 tokenizer}}
    \label{tab:five_column_example}
\end{table}

%% file: tables/longbench_compression_ratios.tex
\begin{table}[ht!]
    \centering
    \begin{tabular}{r|cccc}
        \toprule
        & 2K & 4K & 6K & 8K \\ 
        \hline
        NarrativeQA           & 0.10  & 0.20  & 0.30  & 0.40  \\
        Qasper                & 0.47  & 0.82  & 0.96  & 0.98  \\
        MultifidelityQA-En    & 0.40  & 0.66  & 0.82  & 0.93  \\
        MultifidelityQA-Zh    & 0.49  & 0.78  & 0.91  & 0.98  \\
        \hline
        HotPotQA              & 0.19  & 0.37  & 0.52  & 0.66  \\
        2wikimqa              & 0.36  & 0.65  & 0.85  & 0.92  \\
        musique               & 0.13  & 0.27  & 0.40  & 0.53  \\
        dureader              & 0.17  & 0.34  & 0.52  & 0.68  \\
        \hline
        gov report            & 0.28  & 0.52  & 0.70  & 0.81  \\
        qmsum                 & 0.18  & 0.36  & 0.52  & 0.66  \\
        multi news            & 0.81  & 0.96  & 0.98  & 0.53  \\
        vcsum                 & 0.27  & 0.49  & 0.65  & 0.68  \\
        \hline
        trec                  & 0.38  & 0.66  & 0.84  & 0.94  \\
        triviaqa              & 0.26  & 0.46  & 0.60  & 0.72  \\
        samsum                & 0.32  & 0.55  & 0.71  & 0.82  \\
        lsht                  & 0.13  & 0.27  & 0.39  & 0.52  \\
        \hline
        passage count         & 0.15  & 0.31  & 0.46  & 0.60  \\
        passage-retrieval-En  & 0.16  & 0.33  & 0.49  & 0.66  \\
        passage-retrieval-Zh  & 0.37  & 0.74  & 0.99  & 1.00  \\
        \hline
        lcc                   & 0.78  & 0.95  & 0.98  & 0.99  \\
        repobench-p           & 0.27  & 0.52  & 0.67  & 0.78  \\
        \hline
        average               & 0.31  & 0.53  & 0.67  & 0.77  \\
        \bottomrule
        
    \end{tabular}
    \caption{\textbf{Compression ratio of prompts w.r.t. various KV cache budgets}}
    \label{table:compression_ratio}
\end{table}

%% file: tables/longbench_unrestricted.tex
\begin{table*}[!t]
\caption{Resource unrestricted LongBench Results (Higher is better). All methods processes the input prompt in parallel (i.e., block size $=\infty$) and make an eviction decision with all token information in the input. The token eviction is made at every step of generation if the budget exceeds. We highlight the methods showing the best performance within a given budget with \textbf{boldface}. We omit NarrativeQA from evaluation due to higher chance of OOM errors.}
\label{table:longbench_unrestricted_reduced}
\begin{center}
\resizebox{0.8\textwidth}{0.9\textheight}{%
\begin{adjustbox}{angle=90}
\begin{tabular}{cccccccccccccccccccccccc}
\toprule

%%%%%%%%%%%%%%%%%%%%%%%%%%%%%%%%%%%%%%%%%%%%%%%%%%%%%%%%
%%%%%%%%%%%%%%%%%%%%% Header Rows %%%%%%%%%%%%%%%%%%%%%&
%%%%%%%%%%%%%%%%%%%%%%%%%%%%%%%%%%%%%%%%%%%%%%%%%%%%%%%%
&
&
\multicolumn{3}{c}{Single Doc. QA} &
\multicolumn{4}{c}{Multi Doc. QA} &
\multicolumn{4}{c}{Summarization} &
\multicolumn{4}{c}{Few$\-$shot Learning} &
\multicolumn{3}{c}{Synthetic} &
\multicolumn{2}{c}{Code} &
\\

\cmidrule(lr){3-6}
\cmidrule(lr){7-10}
\cmidrule(lr){11-14}
\cmidrule(lr){15-18}
\cmidrule(lr){19-21}
\cmidrule(lr){22-23}

&
& 
Qasper & MF-en & MF-zh &
HotpotQA & 2WikiMQA & Musique & DuReader &
GovReport & QMSum & MultiNews & VCSum &
TREC & TriviaQA & SAMSum & LSHT &
PCount & PR-en & PR-zh &
Lcc & RB-P & 
Avg. 
\\

%%%%%%%%%%%%%%%%%%%%%%%%%%%%%%%%%%%%%%%%%%%%%%%%%%%%%%%%
%%%%%%%%%%%%%%%%% Llama 3.1-8B Results %%%%%%%%%%%%%%%%%
%%%%%%%%%%%%%%%%%%%%%%%%%%%%%%%%%%%%%%%%%%%%%%%%%%%%%%%%
\midrule
\midrule
\addlinespace

% Llama 3.1 8b Dense
\multicolumn{2}{c}{Llama3.1-8B} & 47.00 & 56.12 & 59.86 & 57.33 & 47.81 & 32.25 & 35.64 & 34.86 & 25.32 & 27.02 & 17.28 & 73.00 & 91.61 & 43.37 & 45.50 & 8.33 & 99.50 & 99.00 & 61.66 & 51.94 & 50.72 \\

\midrule

% H2O
\multirow{4}{*}{H2O} &
2K & 22.75 & 32.73 & 25.93 & 43.56 & 29.49 & 0.00 & 5.35 & 3.70 & 4.73 & 17.42 & 4.44 & 46.19 & 54.88 & 10.39 & 12.20 & 16.13 & 100.00 & 37.75 & 42.44 & 16.44 & 26.33 \\

&
4K & 35.49 & 43.14 & 39.80 & 52.82 & 36.60 & 10.00 & 6.11 & 9.34 & 7.54 & 24.61 & 7.38 & 54.31 & 65.83 & 23.18 & 16.67 & 16.39 & 100.00 & 61.50 & 55.96 & 28.68 & 34.77 \\

&
6K & 44.05 & 47.74 & 47.88 & 52.67 & 45.90 & 11.11 & 7.97 & 15.55 & 13.77 & 26.58 & 9.13 & 60.11 & 78.42 & 32.35 & 13.64 & 12.04 & 100.00 & 94.00 & 60.23 & 39.82 & 40.65 \\

&
8K & 45.93 & 51.12 & 55.72 & 53.09 & 48.83 & 13.58 & 15.63 & 25.89 & 17.39 & 27.18 & 12.39 & 67.71 & 86.91 & 40.05 & 13.33 & 13.73 & 100.00 & 99.00 & 60.82 & 45.17 & 44.67 \\

\midrule

% TOVA
\multirow{4}{*}{TOVA} &
2K & 44.37 & 55.47 & 58.07 & 59.16 & 48.26 & 16.67 & 26.58 & 30.54 & 24.37 & 26.81 & 16.66 & 71.00 & 91.93 & 45.29 & 28.89 & 9.84 & 100.00 & 96.58 & 61.65 & 51.83 & \textbf{48.20} \\

&
4K & 46.45 & 56.23 & 58.82 & 60.72 & 49.96 & 15.28 & 29.93 & 33.52 & 25.56 & 27.18 & 17.40 & 72.08 & 91.40 & 44.49 & 29.79 & 13.06 & 100.00 & 99.00 & 61.95 & 52.19 & \textbf{49.25} \\

&
6K & 46.66 & 54.26 & 59.31 & 55.99 & 50.26 & 16.22 & 34.54 & 34.35 & 25.21 & 27.22 & 17.34 & 72.96 & 91.56 & 44.82 & 30.43 & 14.72 & 100.00 & 99.00 & 62.18 & 53.33 & 49.52 \\

&
8K & 46.57 & 55.44 & 59.16 & 59.87 & 51.77 & 13.58 & 35.28 & 35.21 & 25.98 & 27.27 & 16.80 & 73.23 & 90.70 & 44.71 & 34.78 & 13.59 & 100.00 & 99.00 & 61.96 & 54.32 & \textbf{49.96} \\

\midrule

% Sink
\multirow{4}{*}{Sink} &
2K & 33.33 & 33.74 & 34.73 & 45.37 & 38.46 & 20.97 & 18.31 & 26.08 & 21.41 & 24.98 & 16.08 & 67.50 & 90.00 & 40.99 & 21.25 & 2.50 & 36.00 & 18.00 & 57.08 & 53.81 & 35.03 \\

&
4K & 38.57 & 41.15 & 46.38 & 48.07 & 40.87 & 22.51 & 17.50 & 29.28 & 21.89 & 25.12 & 16.84 & 71.50 & 90.52 & 41.32 & 24.75 & 2.50 & 49.00 & 22.50 & 57.94 & 53.64 & 38.09 \\

&
6K & 40.41 & 43.74 & 52.60 & 50.08 & 42.57 & 22.32 & 17.71 & 30.54 & 22.38 & 25.25 & 17.37 & 72.50 & 90.77 & 41.94 & 26.25 & 2.50 & 57.00 & 21.00 & 57.14 & 54.07 & 39.41 \\

&
8K & 40.53 & 44.58 & 54.77 & 49.13 & 42.38 & 23.75 & 20.35 & 31.35 & 22.61 & 25.28 & 17.45 & 72.50 & 90.77 & 42.27 & 27.75 & 3.00 & 70.00 & 20.50 & 57.13 & 55.35 & 40.57 \\
\midrule

% SnapKV
\multirow{4}{*}{SnapKV} &
2K & 45.01 & 52.53 & 56.15 & 57.54 & 50.17 & 25.00 & 32.35 & 32.99 & 25.38 & 27.24 & 18.14 & 70.85 & 89.48 & 39.71 & 26.53 & 15.86 & 98.57 & 87.04 & 60.56 & 51.92 & 48.15 \\

&
4K & 46.28 & 55.05 & 59.57 & 55.49 & 49.70 & 24.69 & 35.09 & 33.78 & 26.63 & 27.15 & 17.48 & 72.68 & 90.28 & 42.59 & 28.57 & 13.38 & 98.08 & 98.50 & 61.14 & 52.05 & 49.41 \\

&
6K & 46.82 & 55.83 & 59.07 & 59.68 & 50.27 & 22.22 & 35.82 & 34.89 & 25.17 & 27.24 & 17.51 & 72.31 & 90.76 & 44.31 & 23.26 & 14.00 & 100.00 & 99.00 & 62.13 & 54.43 & \textbf{49.74} \\

&
8K & 46.72 & 55.33 & 59.79 & 57.18 & 51.51 & 15.28 & 34.83 & 35.28 & 25.71 & 27.24 & 17.40 & 71.96 & 90.46 & 45.08 & 23.26 & 10.92 & 100.00 & 98.99 & 61.76 & 54.80 & 49.17 \\

\midrule

% KeyDiff
\multirow{4}{*}{\ourmethod{}} &
2K & 44.58 & 53.88 & 53.65 & 57.40 & 47.66 & 14.89 & 33.44 & 29.97 & 25.88 & 26.95 & 16.10 & 72.00 & 92.27 & 43.51 & 31.82 & 11.38 & 100.00 & 95.92 & 59.26 & 45.66 & 47.81 \\

&
4K & 46.05 & 54.87 & 57.52 & 59.04 & 49.97 & 13.58 & 36.51 & 33.17 & 25.92 & 27.12 & 16.85 & 73.00 & 90.13 & 43.87 & 31.82 & 13.28 & 100.00 & 97.67 & 60.57 & 50.49 & 49.07 \\

&
6K & 46.61 & 54.49 & 59.19 & 59.70 & 50.03 & 12.99 & 36.04 & 34.42 & 26.67 & 27.27 & 17.47 & 73.33 & 90.94 & 44.62 & 33.33 & 11.83 & 100.00 & 99.00 & 61.29 & 53.20 & 49.62 \\

&
8K & 46.55 & 55.11 & 59.28 & 57.48 & 50.51 & 13.58 & 35.01 & 34.81 & 25.87 & 27.22 & 17.15 & 72.59 & 91.22 & 44.69 & 27.27 & 11.20 & 100.00 & 99.00 & 61.95 & 53.62 & 49.21 \\
\midrule
\midrule
\addlinespace

%%%%%%%%%%%%%%%%%%%%%%%%%%%%%%%%%%%%%%%%%%%%%%%%%%%%%%%%
%%%%%%%%%%%%%%%%% Llama 3.2-3B Results %%%%%%%%%%%%%%%%%
%%%%%%%%%%%%%%%%%%%%%%%%%%%%%%%%%%%%%%%%%%%%%%%%%%%%%%%%

% Dense
\multicolumn{2}{c}{Llama3.2-3B} &
40.23 & 50.09 & 55.84 & 50.69 & 42.29 & 26.84 & 36.24 & 33.09 & 24.30 & 25.21 & 16.41 & 72.50 & 90.11 & 42.58 & 34.00 & 3.00 & 96.50 & 20.50 & 56.22 & 56.52 & 43.66 \\

\midrule

% H2O
\multirow{4}{*}{H2O} &
2K & 20.21 & 24.23 & 18.98 & 28.28 & 23.51 & 9.88 & 10.40 & 1.43 & 4.13 & 16.14 & 3.05 & 48.00 & 41.36 & 12.23 & 16.00 & 4.81 & 1.52 & 0.50 & 40.16 & 17.58 & 17.12 \\

&
4K & 32.58 & 33.87 & 35.46 & 35.29 & 27.46 & 17.84 & 12.25 & 6.96 & 9.26 & 23.10 & 5.17 & 56.00 & 59.11 & 21.77 & 13.70 & 4.95 & 13.71 & 7.25 & 51.49 & 29.58 & 24.84 \\

&
6K & 38.61 & 44.28 & 46.68 & 42.34 & 36.93 & 11.61 & 15.57 & 13.52 & 13.03 & 24.44 & 7.74 & 63.50 & 72.81 & 30.46 & 12.86 & 4.76 & 63.00 & 19.00 & 54.28 & 39.85 & 32.76 \\

&
8K & 40.23 & 44.83 & 52.90 & 46.13 & 39.78 & 14.74 & 20.79 & 22.65 & 17.06 & 24.77 & 10.55 & 70.50 & 83.55 & 35.08 & 15.07 & 4.04 & 81.91 & 20.00 & 55.35 & 45.83 & 37.29 \\
\midrule

% TOVA
\multirow{4}{*}{TOVA} &
2K & 38.22 & 48.83 & 54.09 & 48.08 & 42.21 & 14.81 & 27.89 & 28.71 & 23.27 & 24.94 & 15.64 & 70.50 & 89.47 & 42.97 & 22.22 & 5.88 & 95.92 & 18.50 & 56.14 & 55.68 & \textbf{41.20} \\

&
4K & 40.55 & 50.77 & 56.10 & 54.26 & 42.68 & 17.08 & 31.66 & 31.09 & 23.44 & 25.37 & 16.00 & 71.50 & 89.77 & 43.00 & 21.92 & 5.88 & 95.50 & 18.50 & 56.33 & 56.52 & 42.40 \\

&
6K & 40.57 & 50.12 & 56.62 & 52.67 & 43.12 & 19.44 & 34.78 & 32.91 & 23.88 & 25.30 & 15.87 & 72.50 & 89.07 & 42.73 & 24.00 & 4.35 & 94.97 & 20.00 & 56.30 & 57.38 & \textbf{42.83} \\

&
8K & 40.86 & 49.79 & 56.31 & 52.91 & 42.03 & 15.34 & 36.74 & 33.34 & 24.11 & 25.30 & 16.15 & 72.50 & 89.40 & 42.39 & 24.32 & 5.77 & 95.50 & 20.00 & 56.27 & 56.85 & 42.79 \\
\midrule

% Sink
\multirow{4}{*}{Sink} &
2K & 33.33 & 33.74 & 34.73 & 45.37 & 38.46 & 20.97 & 18.31 & 26.08 & 21.41 & 24.98 & 16.08 & 67.50 & 90.00 & 40.99 & 21.25 & 2.50 & 36.00 & 18.00 & 57.08 & 53.81 & 35.03 \\

&
4K & 38.57 & 41.15 & 46.38 & 48.07 & 40.87 & 22.51 & 17.50 & 29.28 & 21.89 & 25.12 & 16.84 & 71.50 & 90.52 & 41.32 & 24.75 & 2.50 & 49.00 & 22.50 & 57.94 & 53.64 & 38.09 \\

&
6K & 40.41 & 43.74 & 52.60 & 50.08 & 42.57 & 22.32 & 17.71 & 30.54 & 22.38 & 25.25 & 17.37 & 72.50 & 90.77 & 41.94 & 26.25 & 2.50 & 57.00 & 21.00 & 57.14 & 54.07 & 39.41 \\

&
8K & 40.53 & 44.58 & 54.77 & 49.13 & 42.38 & 23.75 & 20.35 & 31.35 & 22.61 & 25.28 & 17.45 & 72.50 & 90.77 & 42.27 & 27.75 & 3.00 & 70.00 & 20.50 & 57.13 & 55.35 & 40.57 \\

\midrule

% KeyDiff
\multirow{4}{*}{\ourmethod{}} &
2K & 38.15 & 49.58 & 51.15 & 48.73 & 42.04 & 20.24 & 33.47 & 29.13 & 23.91 & 25.10 & 14.47 & 70.50 & 88.31 & 42.28 & 20.55 & 4.90 & 89.00 & 17.50 & 55.21 & 48.73 & 40.65 \\

& 
4K & 40.55 & 51.38 & 55.88 & 51.56 & 41.88 & 18.78 & 35.93 & 31.63 & 24.15 & 25.38 & 15.80 & 71.50 & 89.98 & 42.55 & 19.18 & 4.95 & 95.85 & 21.00 & 55.79 & 55.55 & \textbf{42.46} \\

&
6K & 40.54 & 51.04 & 55.42 & 51.78 & 41.71 & 15.26 & 36.60 & 32.80 & 24.67 & 25.34 & 16.32 & 72.50 & 90.58 & 43.19 & 22.22 & 5.21 & 97.42 & 20.50 & 55.87 & 56.28 & 42.76 \\

& 
8K & 40.66 & 50.93 & 56.16 & 53.52 & 42.07 & 17.79 & 37.08 & 33.31 & 24.27 & 25.31 & 16.00 & 72.50 & 89.48 & 42.52 & 25.33 & 5.00 & 96.46 & 20.00 & 56.20 & 56.80 & \textbf{43.07} \\

\bottomrule
\end{tabular}
\end{adjustbox}
}
\end{center}
\end{table*}

%% file: tables/longbench_full_rotated.tex
\begin{table*}[!t]
\caption{\textbf{Full Llama-3.1-8B/3.2-3B-Instruct LongBench Results with $B=128$ (Higher is better)}. We highlight the methods showing the best performance within a given budget with \textbf{boldface}. \textdagger: A subset of samples were evaluated due to OOM errors (183/200 samples are evaluated).}
\label{table:longbench_full_rotated}
\begin{center}
\resizebox{1.0\textwidth}{0.9\textheight}{%
\begin{adjustbox}{angle=90}
\begin{tabular}{cccccccccccccccccccccccc}
\toprule

%%%%%%%%%%%%%%%%%%%%%%%%%%%%%%%%%%%%%%%%%%%%%%%%%%%%%%%%
%%%%%%%%%%%%%%%%%%%%% Header Rows %%%%%%%%%%%%%%%%%%%%%&
%%%%%%%%%%%%%%%%%%%%%%%%%%%%%%%%%%%%%%%%%%%%%%%%%%%%%%%%
&
&
\multicolumn{4}{c}{Single Doc. QA} &
\multicolumn{4}{c}{Multi Doc. QA} &
\multicolumn{4}{c}{Summarization} &
\multicolumn{4}{c}{Few$\-$shot Learning} &
\multicolumn{3}{c}{Synthetic} &
\multicolumn{2}{c}{Code} &
\\

\cmidrule(lr){3-6}
\cmidrule(lr){7-10}
\cmidrule(lr){11-14}
\cmidrule(lr){15-18}
\cmidrule(lr){19-21}
\cmidrule(lr){22-23}

&
&
{Narrative QA} & {Qasper} & {MF-en} & {MF-zh} &
{HotpotQA} & {2WikiMQA} & {Musique} & {DuReader} &
{GovReport} & {QMSum} & {MultiNews} & {VCSum} &
{TREC} & {TriviaQA} & {SAMSum} & {LSHT} &
{PCount} & {PR-en} & {PR-zh} &
{Lcc} & {RB-P} & 
{Avg.} 
\\

%%%%%%%%%%%%%%%%%%%%%%%%%%%%%%%%%%%%%%%%%%%%%%%%%%%%%%%%
%%%%%%%%%%%%%%%%% Llama 3.1-8B Results %%%%%%%%%%%%%%%%%
%%%%%%%%%%%%%%%%%%%%%%%%%%%%%%%%%%%%%%%%%%%%%%%%%%%%%%%%
\midrule
\midrule
\addlinespace

% Llama 3.1 8b Dense
\multicolumn{2}{c}{Llama3.1-8B} &
% 30.05\textsuperscript{\textdagger} & 47.00 & 56.12 & 57.33 & 47.81 & 32.25 & 34.86 & 25.32 & 27.02 & 73.00 & 91.61 & 43.37 & 8.33 & 99.50 & 61.66 & 51.94 & 49.74 \\
30.05\textsuperscript{\textdagger} & 47.00 & 56.12 & 59.86 & 57.33 & 47.81 & 32.25 & 35.64 & 34.86 & 25.32 & 27.02 & 17.28 & 73.00 & 91.61 & 43.37 & 45.50 & 8.33 & 99.50 & 99.00 & 61.66 & 51.94 & 49.74 \\
\midrule

% H2O
\multirow{4}{*}{H2O} &
2K & 1.74 & 21.15 & 25.33 & 21.65 & 26.11 & 24.15 & 8.78 & 5.90 & 2.17 & 2.70 & 16.78 & 3.97 & 44.00 & 29.36 & 7.62 & 14.50 & 2.25 & 5.88 & 4.00 & 40.15 & 12.14 & 15.25 \\

&
4K & 4.07 & 36.16 & 36.00 & 38.02 & 33.52 & 32.87 & 17.78 & 5.68 & 6.66 & 5.95 & 24.09 & 6.03 & 55.00 & 47.65 & 17.41 & 18.50 & 4.00 & 24.50 & 31.25 & 54.85 & 21.43 & 24.83 \\

&
6K & 8.52 & 43.31 & 44.80 & 46.24 & 40.03 & 42.46 & 21.68 & 7.33 & 11.85 & 8.78 & 26.03 & 7.82 & 62.00 & 56.39 & 25.72 & 18.00 & 5.75 & 45.50 & 90.00 & 58.62 & 29.53 & 33.35 \\

&
8K & 13.85 & 44.94 & 47.81 & 56.14 & 43.64 & 44.90 & 23.65 & 11.01 & 18.78 & 11.35 & 26.49 & 9.96 & 69.50 & 69.05 & 33.41 & 19.50 & 5.25 & 62.50 & 98.67 & 59.74 & 36.26 & 38.40 \\

\midrule

% TOVA
\multirow{4}{*}{TOVA} &
2K & 22.57 & 37.26 & 39.43 & 36.96 & 45.74 & 34.48 & 14.77 & 16.98 & 28.87 & 21.17 & 26.95 & 16.21 & 62.50 & 90.73 & 42.74 & 18.75 & 0.00 & 18.00 & 32.00 & 62.68 & 52.48 & 34.35 \\

&
4K & 22.68 & 44.55 & 47.87 & 51.16 & 46.76 & 44.54 & 20.56 & 22.50 & 30.95 & 22.13 & 26.96 & 16.75 & 61.50 & 90.56 & 43.27 & 25.25 & 3.00 & 43.50 & 84.00 & 61.62 & 53.40 & 41.12 \\

&
6K & 24.59 & 45.93 & 53.92 & 55.45 & 55.09 & 47.43 & 25.07 & 27.68 & 32.33 & 24.10 & 27.00 & 16.91 & 68.50 & 90.81 & 43.89 & 29.00 & 4.25 & 67.00 & 98.67 & 61.50 & 52.39 & 45.31 \\
&
8K & 24.86 & 46.78 & 54.83 & 57.95 & 54.52 & 49.00 & 26.40 & 31.15 & 33.44 & 24.76 & 27.00 & 17.33 & 71.00 & 91.11 & 43.29 & 33.25 & 6.25 & 87.00 & 98.67 & 61.49 & 51.79 & 47.23 \\

\midrule

% Sink
\multirow{4}{*}{Sink} &
2K & 21.83 & 34.27 & 29.24 & 32.82 & 38.64 & 29.50 & 12.59 & 16.18 & 28.51 & 20.21 & 26.62 & 15.54 & 65.00 & 89.46 & 42.20 & 22.25 & 2.00 & 25.50 & 32.50 & 64.95 & 59.54 & 33.78 \\

&
4K & 22.94 & 43.01 & 39.08 & 46.16 & 44.04 & 41.39 & 19.09 & 16.54 & 31.08 & 21.57 & 26.78 & 16.73 & 70.00 & 91.53 & 42.29 & 29.25 & 3.00 & 38.50 & 71.00 & 62.12 & 58.84 & 39.76 \\

&
6K & 25.41 & 47.40 & 44.13 & 52.78 & 47.39 & 45.73 & 21.90 & 17.55 & 32.53 & 22.19 & 26.87 & 17.05 & 72.00 & 91.25 & 43.41 & 33.75 & 3.08 & 52.50 & 98.00 & 62.22 & 56.24 & 43.49 \\

&
8K & 23.53 & 46.63 & 48.68 & 55.77 & 49.61 & 47.16 & 21.14 & 19.54 & 33.10 & 23.20 & 26.92 & 16.91 & 72.00 & 91.29 & 43.79 & 37.00 & 3.25 & 66.00 & 99.00 & 62.18 & 56.43 & 44.91 \\

\midrule

% SnapKV
\multirow{4}{*}{SnapKV} &
2K & 21.81 & 37.22 & 37.19 & 38.29 & 46.10 & 35.42 & 16.53 & 16.37 & 29.83 & 21.05 & 26.77 & 16.16 & 61.00 & 88.84 & 42.56 & 21.75 & 4.03 & 51.50 & 81.17 & 62.37 & 51.45 & 38.45 \\

&
4K & 24.79 & 44.22 & 47.30 & 50.27 & 48.49 & 46.73 & 20.55 & 22.04 & 32.19 & 22.68 & 26.95 & 16.95 & 67.50 & 90.98 & 43.14 & 25.00 & 5.17 & 89.50 & 96.67 & 61.44 & 51.20 & 44.46 \\

&
6K & 24.10 & 45.57 & 50.44 & 55.27 & 53.12 & 48.41 & 24.27 & 27.46 & 33.43 & 23.53 & 27.03 & 16.84 & 71.50 & 92.28 & 43.58 & 27.00 & 5.25 & 98.00 & 99.00 & 61.32 & 52.16 & 46.65 \\

&
8K & 25.15 & 46.55 & 53.39 & 57.65 & 56.00 & 48.75 & 27.82 & 32.66 & 33.67 & 24.85 & 27.01 & 17.37 & 72.50 & 91.78 & 43.54 & 33.75 & 5.08 & 100.00 & 98.67 & 61.48 & 51.41 & 48.05 \\

\midrule

% KeyL2Norm
\multirow{4}{*}{KeyL2Norm\cite{devoto2024simple}} &
2K & 8.66 & 36.63 & 41.70 & 37.70 & 33.75 & 32.25 & 5.39 & 17.73 & 19.64 & 14.96 & 26.69 & 11.02 & 63.00 & 58.94 & 28.45 & 22.50 & 3.05 & 17.75 & 20.13 & 52.40 & 25.63 & 27.52 \\

&
4K & 15.38 & 44.06 & 51.75 & 47.52 & 50.22 & 45.56 & 18.44 & 27.61 & 29.50 & 22.27 & 26.93 & 13.44 & 69.50 & 79.41 & 37.50 & 27.00 & 4.50 & 58.00 & 80.50 & 58.82 & 35.08 & 40.14 \\

&
6K & 21.75 & 45.63 & 55.06 & 53.77 & 52.93 & 47.70 & 25.63 & 32.17 & 32.66 & 24.85 & 26.94 & 15.12 & 70.00 & 86.89 & 40.51 & 34.50 & 5.25 & 75.00 & 98.00 & 60.98 & 43.14 & 45.17 \\

&
8K & 25.12 & 45.70 & 56.02 & 56.57 & 58.14 & 47.77 & 30.29 & 34.09 & 33.81 & 24.89 & 26.94 & 15.89 & 71.50 & 89.26 & 41.34 & 39.00 & 7.25 & 87.00 & 99.00 & 62.05 & 48.28 & 47.61 \\

\midrule

% KeyDiff
\multirow{4}{*}{\ourmethod{}} &
2K & 26.64 & 41.73 & 50.99 & 51.18 & 51.59 & 46.47 & 22.84 & 32.37 & 29.02 & 23.86 & 26.76 & 14.81 & 66.50 & 85.92 & 39.26 & 42.25 & 3.17 & 96.00 & 96.25 & 59.17 & 39.42 & \textbf{45.06} \\

&
4K & 28.70 & 45.62 & 56.06 & 56.83 & 54.58 & 49.31 & 28.25 & 33.06 & 32.30 & 25.03 & 27.07 & 16.32 & 70.00 & 90.85 & 42.84 & 44.50 & 4.21 & 99.00 & 97.67 & 60.80 & 48.00 & \textbf{48.14} \\

&
6K & 29.90 & 46.33 & 55.11 & 59.00 & 56.80 & 49.50 & 31.52 & 34.97 & 33.44 & 24.58 & 26.98 & 16.80 & 72.00 & 90.99 & 43.10 & 47.00 & 5.27 & 99.50 & 99.00 & 61.40 & 49.70 & \textbf{49.19} \\
&
8K & 33.57 & 46.77 & 55.48 & 59.16 & 56.87 & 49.37 & 30.88 & 34.54 & 34.17 & 25.12 & 27.01 & 17.13 & 72.50 & 92.28 & 42.81 & 46.50 & 5.83 & 99.50 & 98.67 & 61.48 & 50.90 & \textbf{49.55} \\

\midrule
\midrule
\addlinespace

%%%%%%%%%%%%%%%%%%%%%%%%%%%%%%%%%%%%%%%%%%%%%%%%%%%%%%%%
%%%%%%%%%%%%%%%%% Llama 3.2-3B Results %%%%%%%%%%%%%%%%%
%%%%%%%%%%%%%%%%%%%%%%%%%%%%%%%%%%%%%%%%%%%%%%%%%%%%%%%%

% Dense
\multicolumn{2}{c}{Llama3.2-3B} &
23.76 & 40.23 & 50.09 & 55.84 & 50.69 & 42.29 & 26.84 & 36.24 & 33.09 & 24.30 & 25.21 & 16.41 & 72.50 & 90.11 & 42.58 & 34.00 & 3.00 & 96.50 & 20.50 & 56.22 & 56.52 & 42.71 \\
\midrule

% H2O
\multirow{4}{*}{H2O} &
2K & 1.63 & 19.96 & 20.20 & 15.20 & 18.02 & 19.56 & 2.88 & 6.47 & 0.78 & 1.55 & 15.97 & 3.11 & 41.00 & 21.97 & 9.83 & 11.75 & 0.50 & 0.50 & 0.00 & 39.71 & 13.91 & 12.60 \\

&
4K & 2.92 & 31.94 & 33.23 & 33.25 & 24.49 & 28.08 & 7.55 & 10.10 & 5.44 & 6.30 & 22.77 & 4.81 & 53.00 & 38.85 & 20.33 & 15.50 & 1.50 & 7.50 & 6.25 & 51.23 & 22.94 & 20.38 \\

&
6K & 4.62 & 38.81 & 39.06 & 45.17 & 34.66 & 35.52 & 15.21 & 13.36 & 10.51 & 10.01 & 24.25 & 6.66 & 61.50 & 53.23 & 27.37 & 15.25 & 0.50 & 13.00 & 19.50 & 54.55 & 32.29 & 26.43 \\

&
8K & 9.65 & 39.66 & 43.20 & 52.60 & 38.09 & 40.41 & 21.46 & 18.55 & 17.80 & 13.28 & 24.67 & 9.12 & 70.00 & 64.30 & 32.19 & 17.00 & 2.00 & 24.50 & 21.50 & 55.00 & 39.09 & 31.15 \\
\midrule

% TOVA
\multirow{4}{*}{TOVA} &
2K & 17.14 & 30.14 & 32.44 & 31.64 & 35.96 & 30.05 & 13.08 & 9.62 & 26.15 & 19.70 & 25.04 & 15.47 & 56.50 & 87.81 & 40.48 & 16.75 & 2.50 & 11.50 & 6.50 & 55.51 & 52.36 & 29.35 \\

&
4K & 20.52 & 39.53 & 42.47 & 45.80 & 44.12 & 38.42 & 18.22 & 17.76 & 29.36 & 21.36 & 24.96 & 16.60 & 63.50 & 88.98 & 41.50 & 18.75 & 3.00 & 23.50 & 15.00 & 55.72 & 56.66 & 34.56 \\

&
6K & 20.22 & 39.78 & 45.86 & 52.93 & 49.08 & 41.54 & 20.43 & 24.78 & 30.50 & 22.17 & 25.11 & 16.37 & 66.50 & 89.00 & 42.50 & 21.00 & 4.00 & 46.50 & 20.50 & 55.57 & 57.53 & 37.71 \\

&
8K & 21.08 & 40.67 & 49.07 & 55.17 & 48.69 & 41.93 & 23.05 & 31.02 & 31.64 & 22.85 & 25.21 & 16.55 & 69.00 & 89.25 & 42.19 & 22.50 & 2.50 & 71.00 & 21.50 & 55.77 & 57.47 & 39.91 \\
\midrule

% Sink
\multirow{4}{*}{Sink} &
2K & 16.85 & 30.69 & 26.58 & 27.32 & 33.26 & 25.27 & 13.82 & 9.38 & 26.74 & 19.15 & 25.15 & 15.88 & 65.00 & 86.17 & 40.79 & 19.50 & 1.50 & 19.50 & 8.50 & 56.65 & 52.73 & 29.54 \\

&
4K & 19.46 & 38.61 & 36.22 & 41.68 & 41.97 & 35.84 & 13.37 & 9.86 & 29.34 & 20.19 & 25.06 & 16.44 & 71.00 & 88.06 & 41.31 & 21.75 & 2.50 & 35.50 & 16.00 & 56.48 & 52.43 & 33.96 \\

&
6K & 19.33 & 40.29 & 37.95 & 49.68 & 46.48 & 40.29 & 15.31 & 11.10 & 30.43 & 21.35 & 25.14 & 16.64 & 71.50 & 88.93 & 42.04 & 23.50 & 3.50 & 47.00 & 19.50 & 56.55 & 54.11 & 36.22 \\

&
8K & 20.15 & 40.02 & 41.94 & 53.57 & 48.15 & 42.24 & 16.01 & 14.76 & 31.64 & 22.10 & 25.20 & 16.50 & 73.00 & 89.26 & 42.37 & 26.25 & 3.50 & 62.50 & 20.50 & 56.86 & 56.63 & 38.25 \\

\midrule

% SnapKV
\multirow{4}{*}{SnapKV} &
2K & 17.38 & 31.37 & 31.48 & 29.65 & 37.77 & 30.05 & 11.54 & 9.66 & 27.03 & 19.93 & 24.97 & 15.97 & 59.00 & 88.13 & 40.48 & 16.25 & 3.50 & 32.50 & 9.00 & 56.32 & 55.91 & 30.85 \\

&
4K & 19.85 & 39.22 & 39.86 & 47.33 & 46.70 & 37.98 & 16.64 & 16.88 & 29.79 & 21.21 & 25.01 & 16.74 & 65.50 & 89.35 & 40.95 & 18.25 & 2.50 & 62.50 & 22.50 & 55.74 & 56.88 & 36.73 \\

&
6K & 20.83 & 39.65 & 44.48 & 51.84 & 49.30 & 40.18 & 20.28 & 25.32 & 31.27 & 22.73 & 25.09 & 16.81 & 69.00 & 89.95 & 41.47 & 18.75 & 4.00 & 85.00 & 20.50 & 55.69 & 57.82 & 39.52 \\

&
8K & 20.49 & 40.80 & 48.16 & 55.44 & 48.78 & 41.65 & 24.79 & 30.40 & 31.81 & 23.46 & 25.17 & 16.44 & 70.00 & 90.17 & 41.99 & 22.00 & 5.00 & 94.00 & 21.50 & 55.77 & 57.29 & 41.20 \\

\midrule

% KeyL2Norm
\multirow{4}{*}{KeyL2Norm \cite{devoto2024simple}} &
2K & 7.67 & 30.39 & 31.85 & 30.64 & 29.47 & 25.76 & 7.41 & 14.17 & 15.36 & 12.42 & 24.20 & 7.91 & 48.00 & 50.99 & 23.09 & 17.50 & 2.00 & 7.50 & 5.00 & 48.92 & 26.32 & 22.22 \\

&
4K & 12.92 & 37.59 & 44.71 & 43.85 & 38.89 & 33.42 & 12.41 & 22.42 & 24.63 & 19.27 & 24.77 & 11.37 & 63.00 & 72.51 & 31.75 & 19.00 & 3.87 & 9.50 & 11.00 & 55.82 & 40.08 & 30.13 \\

&
6K & 13.02 & 40.55 & 48.17 & 50.87 & 43.18 & 40.17 & 17.10 & 31.29 & 28.99 & 21.47 & 25.08 & 13.30 & 65.00 & 79.61 & 37.16 & 21.50 & 2.50 & 46.50 & 18.50 & 56.49 & 47.05 & 35.60 \\

&
8K & 15.72 & 40.54 & 47.88 & 54.29 & 49.29 & 43.79 & 22.22 & 33.18 & 31.86 & 22.50 & 25.19 & 14.28 & 70.00 & 84.92 & 39.45 & 23.00 & 1.50 & 69.00 & 20.50 & 56.82 & 50.73 & 38.89 \\

\midrule

% KeyDiff
\multirow{4}{*}{\ourmethod{}} &
2K & 18.29 & 36.65 & 45.44 & 47.47 & 46.09 & 35.41 & 13.79 & 28.89 & 28.16 & 21.45 & 25.01 & 13.56 & 60.00 & 85.24 & 37.00 & 24.88 & 1.00 & 60.50 & 12.00 & 54.13 & 42.01 & \textbf{35.09} \\

&
4K & 22.34 & 40.60 & 49.15 & 52.56 & 50.14 & 40.30 & 21.65 & 32.46 & 31.38 & 23.44 & 25.06 & 15.28 & 66.50 & 87.92 & 41.41 & 27.50 & 2.50 & 88.50 & 19.50 & 55.55 & 52.24 & \textbf{40.28} \\

&
6K & 22.29 & 40.68 & 50.14 & 54.51 & 51.74 & 42.19 & 24.83 & 34.64 & 32.39 & 23.53 & 25.19 & 15.88 & 71.00 & 90.02 & 42.00 & 28.75 & 3.00 & 95.00 & 21.50 & 55.86 & 54.39 & \textbf{41.88} \\

& 
8K & 22.41 & 40.77 & 50.10 & 55.62 & 49.83 & 43.58 & 28.09 & 34.30 & 32.78 & 23.60 & 25.17 & 15.77 & 72.00 & 90.17 & 42.46 & 30.75 & 3.50 & 96.50 & 21.50 & 55.85 & 55.65 & \textbf{42.40} \\

\bottomrule
\end{tabular}
\end{adjustbox}
}
\end{center}
\end{table*}

%% file: tables/longbench_full_rotated_qwen.tex
\begin{table*}[!t]
\caption{\textbf{Full Qwen-2.5-7B/3B-Instruct LongBench Results with $B=128$ (Higher is better)}. We highlight the best and second best methods within a given budget with \textbf{bold} and \underline{underline}.}
\label{table:longbench_full_rotated_qwen}
\begin{center}
\resizebox{0.75\textwidth}{0.9\textheight}{%
\begin{adjustbox}{angle=90}
\begin{tabular}{cccccccccccccccccccccccc}
\toprule

%%%%%%%%%%%%%%%%%%%%%%%%%%%%%%%%%%%%%%%%%%%%%%%%%%%%%%%%
%%%%%%%%%%%%%%%%%%%%% Header Rows %%%%%%%%%%%%%%%%%%%%%&
%%%%%%%%%%%%%%%%%%%%%%%%%%%%%%%%%%%%%%%%%%%%%%%%%%%%%%%%
&
&
\multicolumn{4}{c}{Single Doc. QA} &
\multicolumn{4}{c}{Multi Doc. QA} &
\multicolumn{4}{c}{Summarization} &
\multicolumn{4}{c}{Few$\-$shot Learning} &
\multicolumn{3}{c}{Synthetic} &
\multicolumn{2}{c}{Code} &
\\

\cmidrule(lr){3-6}
\cmidrule(lr){7-10}
\cmidrule(lr){11-14}
\cmidrule(lr){15-18}
\cmidrule(lr){19-21}
\cmidrule(lr){22-23}

&
&
{Narrative QA} & {Qasper} & {MF-en} & {MF-zh} &
{HotpotQA} & {2WikiMQA} & {Musique} & {DuReader} &
{GovReport} & {QMSum} & {MultiNews} & {VCSum} &
{TREC} & {TriviaQA} & {SAMSum} & {LSHT} &
{PCount} & {PR-en} & {PR-zh} &
{Lcc} & {RB-P} & 
{Avg.} 
\\

%%%%%%%%%%%%%%%%%%%%%%%%%%%%%%%%%%%%%%%%%%%%%%%%%%%%%%%%
%%%%%%%%%%%%%%%%% Qwen 2.5-7b Results %%%%%%%%%%%%%%%%%
%%%%%%%%%%%%%%%%%%%%%%%%%%%%%%%%%%%%%%%%%%%%%%%%%%%%%%%%
\midrule
\midrule
\addlinespace

% Dense
\multicolumn{2}{c}{Qwen2.5-7B} &
15.75 & 16.94 & 32.38 & 14.87 & 11.89 & 11.88 & 7.95 & 30.56 & 34.33 & 19.91 & 22.67 & 15.28 & 65.50 & 87.05 & 44.75 & 39.47 & 4.22 & 93.08 & 68.79 & 57.74 & 61.84 & 36.04 \\
\midrule

% H2O
\multirow{4}{*}{H2O} &
2K & 2.39 & 7.29 & 12.42 & 11.73 & 8.55 & 11.06 & 2.73 & 6.07 & 3.62 & 6.60 & 15.69 & 3.44 & 42.50 & 28.21 & 10.63 & 16.00 & 0.65 & 0.00 & 1.50 & 35.10 & 18.77 & 11.66 \\

&
4K & 1.99 & 11.92 & 19.88 & 14.72 & 10.24 & 10.12 & 4.73 & 7.51 & 9.08 & 10.14 & 20.85 & 6.15 & 51.00 & 37.37 & 20.57 & 15.75 & 3.16 & 6.43 & 27.62 & 52.14 & 29.09 & 17.64 \\

&
6K & 3.34 & 14.79 & 23.94 & 15.33 & 11.45 & 11.30 & 5.52 & 9.30 & 14.63 & 14.27 & 22.06 & 8.68 & 55.75 & 51.99 & 28.01 & 18.50 & 1.39 & 9.41 & 54.53 & 54.68 & 38.32 & 22.25 \\

&
8K & 6.10 & 15.55 & 28.29 & 14.99 & 12.37 & 14.65 & 6.24 & 16.10 & 20.78 & 17.22 & 22.44 & 11.12 & 59.00 & 58.74 & 33.05 & 24.92 & 1.82 & 15.73 & 55.16 & 55.63 & 44.56 & 25.45 \\

\midrule

% TOVA
\multirow{4}{*}{TOVA} &
2K & 8.49 & 14.01 & 21.04 & 11.55 & 14.00 & 11.51 & 5.09 & 14.45 & 27.43 & 17.84 & 22.83 & 15.75 & 56.50 & 79.56 & 40.55 & 20.50 & 2.43 & 9.29 & 20.45 & 55.99 & 56.15 & \underline{25.02} \\

&
4K & 12.83 & 17.03 & 27.01 & 14.14 & 16.80 & 13.37 & 8.05 & 21.15 & 29.21 & 19.05 & 22.73 & 15.81 & 58.50 & 82.67 & 42.71 & 27.75 & 1.67 & 15.00 & 43.53 & 56.69 & 56.59 & 28.68 \\

&
6K & 15.77 & 15.33 & 30.31 & 14.58 & 19.30 & 13.78 & 9.11 & 25.70 & 30.40 & 19.95 & 22.91 & 15.10 & 61.50 & 83.47 & 42.90 & 27.60 & 1.15 & 21.75 & 55.16 & 57.68 & 57.99 & 30.54 \\

&
8K & 15.69 & 15.55 & 33.09 & 14.78 & 18.37 & 13.99 & 11.26 & 27.92 & 31.33 & 20.17 & 22.82 & 15.27 & 62.00 & 84.49 & 43.01 & 33.21 & 2.78 & 30.33 & 55.16 & 57.45 & 58.96 & 31.79 \\

\midrule

% Sink
\multirow{4}{*}{Sink} &
2K & 7.68 & 14.68 & 19.36 & 12.98 & 8.58 & 9.34 & 3.97 & 10.66 & 27.75 & 17.96 & 22.33 & 14.26 & 62.00 & 75.26 & 42.76 & 23.00 & 1.07 & 7.50 & 21.70 & 50.11 & 49.57 & 23.93 \\

&
4K & 7.68 & 17.18 & 23.46 & 14.65 & 9.09 & 9.38 & 4.39 & 10.57 & 30.23 & 18.62 & 22.79 & 15.48 & 64.50 & 83.39 & 44.19 & 29.81 & 2.74 & 18.08 & 64.95 & 55.23 & 51.30 & 28.46 \\

&
6K & 7.37 & 16.61 & 25.73 & 14.74 & 11.29 & 11.27 & 5.69 & 11.49 & 31.47 & 18.72 & 22.86 & 15.62 & 64.50 & 84.86 & 44.47 & 31.07 & 3.59 & 41.48 & 71.21 & 55.89 & 55.99 & 30.76 \\

&
8K & 8.22 & 16.15 & 28.63 & 15.52 & 11.59 & 11.11 & 6.44 & 17.29 & 32.56 & 18.49 & 22.91 & 15.45 & 65.00 & 83.95 & 44.15 & 35.75 & 4.14 & 48.72 & 71.71 & 56.82 & 56.42 & 31.95 \\

\midrule

% SnapKV
\multirow{4}{*}{SnapKV} &
2K & 11.60 & 12.45 & 23.66 & 12.54 & 12.38 & 10.64 & 7.03 & 14.40 & 27.57 & 18.27 & 22.85 & 15.23 & 58.00 & 81.78 & 41.13 & 23.67 & 3.76 & 19.42 & 35.09 & 55.83 & 56.53 & \textbf{26.85} \\

&
4K & 14.35 & 13.45 & 28.28 & 13.76 & 16.33 & 11.74 & 8.12 & 21.96 & 29.71 & 19.18 & 22.82 & 15.20 & 57.00 & 83.80 & 43.27 & 25.51 & 2.41 & 39.83 & 55.28 & 58.12 & 58.67 & \textbf{30.42} \\

&
6K & 14.34 & 16.35 & 31.12 & 14.16 & 17.56 & 14.10 & 8.74 & 25.56 & 31.09 & 20.16 & 22.84 & 15.04 & 60.00 & 83.80 & 42.99 & 30.81 & 2.91 & 54.17 & 55.16 & 57.48 & 60.26 & 32.32 \\

&
8K & 15.60 & 15.81 & 33.47 & 14.77 & 18.02 & 14.49 & 10.53 & 27.45 & 31.99 & 20.09 & 22.84 & 15.20 & 61.00 & 84.08 & 43.01 & 34.28 & 4.58 & 64.25 & 55.16 & 57.46 & 60.59 & 33.56 \\
\midrule

% KeyDiff
\multirow{4}{*}{\ourmethod{}} &
2K & 7.17 & 10.06 & 24.28 & 12.96 & 10.03 & 10.81 & 5.71 & 23.59 & 17.09 & 18.03 & 22.71 & 11.73 & 52.00 & 53.98 & 32.22 & 30.00 & 3.52 & 33.33 & 34.37 & 53.13 & 32.05 & 23.75 \\

&
4K & 13.16 & 12.00 & 32.08 & 14.15 & 13.04 & 13.68 & 5.39 & 27.83 & 25.61 & 20.42 & 22.76 & 13.37 & 54.00 & 70.90 & 40.23 & 39.62 & 3.37 & 58.42 & 54.86 & 55.95 & 42.27 & \underline{30.15} \\

&
6K & 13.42 & 14.90 & 35.11 & 14.62 & 18.70 & 14.09 & 8.34 & 30.74 & 29.83 & 21.08 & 22.86 & 14.39 & 60.50 & 77.03 & 42.00 & 38.08 & 4.13 & 69.83 & 54.33 & 56.76 & 51.50 & \textbf{32.96} \\

&
8K & 14.90 & 15.77 & 34.32 & 14.90 & 19.02 & 13.93 & 9.27 & 30.65 & 31.29 & 20.90 & 22.79 & 14.69 & 60.00 & 83.01 & 43.65 & 40.70 & 3.87 & 74.13 & 55.16 & 57.33 & 52.80 & \textbf{33.96} \\

\midrule
\midrule
\addlinespace

%%%%%%%%%%%%%%%%%%%%%%%%%%%%%%%%%%%%%%%%%%%%%%%%%%%%%%%%
%%%%%%%%%%%%%%%%% Llama 2.5-3B Results %%%%%%%%%%%%%%%%%
%%%%%%%%%%%%%%%%%%%%%%%%%%%%%%%%%%%%%%%%%%%%%%%%%%%%%%%%

% Dense
\multicolumn{2}{c}{Qwen2.5-3B} &
18.08 & 22.49 & 39.72 & 28.99 & 27.86 & 20.45 & 18.93 & 32.95 & 32.80 & 23.74 & 24.89 & 10.95 & 67.50 & 85.05 & 43.88 & 37.50 & 5.00 & 40.97 & 20.61 & 51.91 & 47.53 & 33.42 \\
\midrule

% H2O
\multirow{4}{*}{H2O} &
2K & 1.80 & 9.18 & 11.62 & 12.62 & 8.54 & 7.31 & 2.70 & 8.57 & 5.93 & 6.97 & 16.89 & 4.15 & 38.00 & 21.87 & 7.69 & 16.00 & 1.00 & 3.00 & 2.69 & 37.36 & 22.90 & 11.75 \\

&
4K & 2.82 & 17.34 & 23.27 & 21.55 & 10.18 & 10.47 & 3.03 & 11.05 & 11.06 & 10.73 & 22.93 & 5.77 & 50.75 & 34.93 & 18.03 & 16.25 & 4.35 & 7.32 & 16.64 & 47.74 & 29.42 & 17.89 \\

&
6K & 5.52 & 18.62 & 27.93 & 27.26 & 12.61 & 15.07 & 4.26 & 14.65 & 14.92 & 13.89 & 24.21 & 7.55 & 58.00 & 45.94 & 24.93 & 16.00 & 2.91 & 9.10 & 21.32 & 49.50 & 34.54 & 21.37 \\

&
8K & 6.16 & 19.84 & 32.32 & 29.66 & 16.01 & 17.74 & 4.99 & 19.42 & 20.21 & 16.49 & 24.54 & 9.22 & 64.00 & 56.10 & 32.56 & 20.25 & 3.13 & 11.61 & 21.32 & 50.61 & 38.80 & 24.52 \\
\midrule

% TOVA
\multirow{4}{*}{TOVA} &
2K & 11.69 & 14.94 & 25.33 & 19.90 & 17.29 & 12.58 & 5.91 & 15.34 & 26.67 & 21.49 & 24.78 & 16.58 & 51.50 & 68.80 & 41.79 & 17.75 & 0.23 & 6.00 & 8.68 & 49.79 & 48.60 & 24.08 \\

&
4K & 12.19 & 18.31 & 32.56 & 27.33 & 20.58 & 13.80 & 7.74 & 21.11 & 28.82 & 22.27 & 24.98 & 15.82 & 59.00 & 80.66 & 43.05 & 21.25 & 1.11 & 9.56 & 19.18 & 49.93 & 46.74 & 27.43 \\

&
6K & 13.62 & 19.56 & 34.64 & 28.72 & 21.67 & 16.25 & 8.47 & 27.26 & 30.17 & 23.10 & 24.94 & 14.53 & 63.50 & 81.88 & 42.97 & 26.25 & 1.16 & 10.58 & 21.32 & 51.30 & 47.70 & 29.03 \\

&
8K & 14.66 & 20.93 & 37.77 & 29.72 & 22.57 & 17.08 & 9.63 & 29.10 & 31.12 & 23.17 & 24.83 & 13.48 & 67.00 & 84.11 & 43.55 & 28.25 & 2.06 & 13.08 & 21.32 & 51.32 & 47.64 & 30.11 \\
\midrule

% Sink
\multirow{4}{*}{Sink} &
2K & 9.71 & 13.75 & 22.11 & 20.97 & 11.63 & 14.67 & 4.43 & 11.89 & 27.39 & 19.45 & 24.36 & 13.00 & 56.00 & 58.77 & 42.37 & 22.75 & 2.50 & 8.75 & 4.33 & 48.27 & 49.72 & \underline{23.18} \\

&
4K & 11.46 & 18.28 & 30.40 & 24.02 & 15.50 & 14.62 & 6.97 & 11.48 & 30.08 & 20.12 & 24.86 & 13.35 & 63.00 & 68.77 & 43.11 & 29.50 & 3.00 & 11.75 & 16.75 & 51.76 & 50.47 & 26.63 \\

&
6K & 13.01 & 20.03 & 32.59 & 27.06 & 18.62 & 15.77 & 9.37 & 13.35 & 30.98 & 20.70 & 24.97 & 13.05 & 66.50 & 75.39 & 42.77 & 30.00 & 4.00 & 14.92 & 20.44 & 52.32 & 50.35 & 28.39 \\

&
8K & 10.26 & 21.27 & 35.15 & 29.49 & 24.31 & 17.60 & 9.40 & 17.59 & 31.81 & 21.14 & 24.99 & 12.07 & 68.50 & 79.17 & 43.32 & 34.50 & 1.00 & 24.00 & 20.44 & 51.47 & 49.38 & 29.85 \\

\midrule

% Sink
\multirow{4}{*}{SnapKV} &
2K & 11.70 & 13.91 & 24.28 & 20.75 & 14.80 & 10.89 & 7.42 & 15.09 & 27.40 & 21.63 & 24.64 & 15.71 & 54.50 & 75.35 & 42.72 & 22.38 & 2.50 & 18.33 & 19.06 & 49.65 & 50.59 & \textbf{25.87} \\

&
4K & 12.98 & 22.21 & 31.77 & 26.57 & 18.33 & 14.41 & 10.83 & 21.14 & 29.14 & 22.38 & 24.89 & 15.88 & 61.00 & 84.17 & 42.63 & 21.17 & 3.75 & 25.42 & 22.46 & 50.22 & 48.77 & \textbf{29.05} \\

&
6K & 14.16 & 20.09 & 36.15 & 28.41 & 19.14 & 15.59 & 12.70 & 26.21 & 30.35 & 22.75 & 24.91 & 14.96 & 65.00 & 83.92 & 43.52 & 25.50 & 5.00 & 32.20 & 21.32 & 51.04 & 47.49 & \textbf{30.50} \\

&
8K & 12.76 & 20.88 & 37.10 & 30.10 & 22.49 & 18.19 & 13.83 & 29.54 & 31.33 & 23.37 & 24.80 & 13.75 & 65.50 & 84.88 & 44.49 & 28.00 & 5.20 & 35.83 & 21.32 & 51.31 & 47.82 & 31.55 \\

\midrule

% KeyDiff
\multirow{4}{*}{\ourmethod{}} &
2K & 3.99 & 10.20 & 22.71 & 15.77 & 8.93 & 13.12 & 5.51 & 24.79 & 17.35 & 16.56 & 24.31 & 10.53 & 57.50 & 41.19 & 27.43 & 25.25 & 3.88 & 11.32 & 10.68 & 46.44 & 34.33 & 20.56 \\

&
4K & 9.39 & 18.61 & 31.37 & 23.64 & 18.96 & 18.10 & 7.86 & 27.01 & 25.64 & 22.28 & 24.70 & 10.93 & 65.00 & 63.02 & 37.74 & 30.50 & 4.00 & 20.08 & 26.85 & 49.27 & 39.24 & \underline{27.34} \\

&
6K & 10.51 & 19.71 & 35.51 & 28.89 & 26.92 & 18.28 & 11.47 & 31.83 & 29.32 & 23.63 & 24.90 & 11.46 & 64.50 & 75.36 & 41.51 & 34.50 & 3.57 & 31.41 & 21.32 & 50.60 & 42.95 & \underline{30.39} \\

& 
8K & 12.24 & 20.49 & 38.52 & 29.60 & 23.05 & 19.41 & 15.95 & 30.92 & 31.10 & 23.89 & 24.83 & 11.80 & 67.50 & 79.05 & 41.73 & 36.50 & 3.08 & 40.21 & 21.32 & 51.05 & 45.88 & \textbf{31.82} \\

\bottomrule
\end{tabular}
\end{adjustbox}
}
\end{center}
\end{table*}

%% file: tables/longbench_full_rotated_llama3.2-3b_blocksize_ablation.tex
\begin{table*}[!t]
\caption{\textbf{Llama-3.2-3B-Instruct LongBench Results with prompt block size $B \in [64,256]$ (Higher is better)}. We highlight the best and second best methods within a given budget with \textbf{bold} and \underline{underline}.}
\label{table:longbench_full_rotated_block_64_256}
\begin{center}
\resizebox{0.75\textwidth}{0.9\textheight}{%
\begin{adjustbox}{angle=90}
\begin{tabular}{cccccccccccccccccccccccc}
\toprule

%%%%%%%%%%%%%%%%%%%%%%%%%%%%%%%%%%%%%%%%%%%%%%%%%%%%%%%%
%%%%%%%%%%%%%%%%%%%%% Header Rows %%%%%%%%%%%%%%%%%%%%%&
%%%%%%%%%%%%%%%%%%%%%%%%%%%%%%%%%%%%%%%%%%%%%%%%%%%%%%%%
&
&
\multicolumn{4}{c}{Single Doc. QA} &
\multicolumn{4}{c}{Multi Doc. QA} &
\multicolumn{4}{c}{Summarization} &
\multicolumn{4}{c}{Few$\-$shot Learning} &
\multicolumn{3}{c}{Synthetic} &
\multicolumn{2}{c}{Code} &
\\

\cmidrule(lr){3-6}
\cmidrule(lr){7-10}
\cmidrule(lr){11-14}
\cmidrule(lr){15-18}
\cmidrule(lr){19-21}
\cmidrule(lr){22-23}

&
&
{Narrative QA} & {Qasper} & {MF-en} & {MF-zh} &
{HotpotQA} & {2WikiMQA} & {Musique} & {DuReader} &
{GovReport} & {QMSum} & {MultiNews} & {VCSum} &
{TREC} & {TriviaQA} & {SAMSum} & {LSHT} &
{PCount} & {PR-en} & {PR-zh} &
{Lcc} & {RB-P} & 
{Avg.} 
\\

%%%%%%%%%%%%%%%%%%%%%%%%%%%%%%%%%%%%%%%%%%%%%%%%%%%%%%%%
%%%%%%%%%%%%%%%%% Llama 3.1-8B Results %%%%%%%%%%%%%%%%%
%%%%%%%%%%%%%%%%%%%%%%%%%%%%%%%%%%%%%%%%%%%%%%%%%%%%%%%%
\midrule
\midrule
\addlinespace

% Llama 3.1 8b Dense
\multicolumn{2}{c}{$B=64$} &
23.76 & 40.23 & 50.09 & 55.84 & 50.69 & 42.29 & 26.84 & 36.24 & 33.09 & 24.30 & 25.21 & 16.41 & 72.50 & 90.11 & 42.58 & 34.00 & 3.00 & 96.50 & 20.50 & 56.22 & 56.52 & 42.71 \\
\midrule

% H2O
\multirow{4}{*}{H2O} &
2K & 1.30 & 18.23 & 16.96 & 14.25 & 12.26 & 16.84 & 0.72 & 6.88 & 0.78 & 1.29 & 16.24 & 2.97 & 35.00 & 18.07 & 9.78 & 13.50 & 0.50 & 1.50 & 0.25 & 39.98 & 12.68 & 11.43 \\

&
4K & 2.30 & 31.92 & 31.59 & 32.44 & 25.02 & 22.58 & 4.89 & 9.29 & 5.36 & 5.57 & 23.02 & 4.80 & 50.50 & 33.05 & 18.76 & 13.75 & 1.00 & 2.50 & 3.75 & 50.11 & 21.57 & 18.75 \\

&
6K & 3.12 & 38.87 & 37.63 & 44.58 & 34.38 & 35.35 & 12.13 & 12.82 & 10.43 & 9.38 & 24.31 & 6.63 & 63.00 & 45.46 & 26.08 & 14.25 & 0.00 & 8.50 & 17.50 & 53.84 & 30.66 & 25.19 \\

&
8K & 9.11 & 40.09 & 45.04 & 52.58 & 39.24 & 38.25 & 15.88 & 17.82 & 17.97 & 12.96 & 24.70 & 9.26 & 70.50 & 60.57 & 32.71 & 16.00 & 0.50 & 21.50 & 19.50 & 54.72 & 39.26 & 30.39 \\

\midrule

% TOVA
\multirow{4}{*}{TOVA} &
2K & 17.24 & 30.03 & 31.04 & 32.07 & 36.58 & 28.97 & 12.17 & 10.50 & 26.35 & 19.78 & 25.07 & 15.20 & 60.50 & 87.45 & 41.07 & 15.50 & 1.00 & 10.50 & 6.00 & 55.30 & 52.36 & 29.27 \\

&
4K & 19.59 & 39.27 & 42.16 & 44.54 & 44.58 & 37.63 & 18.62 & 17.44 & 28.82 & 21.46 & 25.18 & 16.49 & 62.50 & 89.48 & 41.89 & 17.50 & 3.50 & 24.00 & 15.00 & 55.14 & 56.58 & 34.35 \\

&
6K & 21.53 & 40.32 & 46.16 & 52.81 & 49.44 & 40.35 & 18.73 & 25.74 & 30.47 & 22.30 & 25.23 & 16.23 & 66.00 & 90.00 & 42.48 & 21.00 & 3.00 & 47.00 & 18.50 & 55.15 & 58.03 & 37.64 \\

&
8K & 21.32 & 40.87 & 50.20 & 54.84 & 49.35 & 42.11 & 24.52 & 30.71 & 31.60 & 23.05 & 25.20 & 16.57 & 69.00 & 90.50 & 41.80 & 23.25 & 5.50 & 74.50 & 19.50 & 55.45 & 58.34 & 40.39 \\

\midrule

% Sink
\multirow{4}{*}{Sink} &
2K & 15.68 & 29.91 & 26.61 & 27.42 & 33.16 & 25.43 & 13.36 & 9.37 & 26.70 & 19.25 & 25.01 & 16.04 & 64.50 & 86.33 & 41.04 & 19.00 & 1.50 & 19.00 & 8.50 & 56.48 & 52.91 & 29.39 \\

&
4K & 19.35 & 37.77 & 36.91 & 41.61 & 41.46 & 35.26 & 12.88 & 10.30 & 29.59 & 20.27 & 25.01 & 16.09 & 69.00 & 88.06 & 41.88 & 21.25 & 2.50 & 36.00 & 16.50 & 55.85 & 52.51 & 33.81 \\

&
6K & 19.36 & 40.01 & 38.16 & 49.57 & 46.39 & 39.01 & 14.20 & 10.41 & 30.54 & 21.64 & 25.07 & 16.93 & 71.00 & 88.50 & 42.09 & 23.50 & 3.50 & 47.50 & 19.50 & 56.02 & 53.97 & 36.04 \\

&
8K & 20.49 & 39.89 & 42.21 & 53.08 & 47.72 & 41.44 & 16.37 & 15.08 & 31.80 & 22.04 & 25.07 & 16.89 & 72.00 & 89.26 & 42.20 & 25.25 & 3.50 & 61.50 & 20.50 & 56.21 & 56.42 & 38.04 \\

\midrule

% SnapKV
\multirow{4}{*}{SnapKV} &
2K & 18.07 & 31.21 & 30.60 & 31.32 & 37.31 & 30.69 & 11.71 & 9.98 & 26.98 & 19.87 & 25.13 & 16.05 & 61.00 & 87.85 & 40.36 & 16.25 & 3.00 & 32.00 & 8.00 & 56.78 & 55.77 & \underline{30.95} \\

&
4K & 19.30 & 39.01 & 40.81 & 44.86 & 47.83 & 37.74 & 16.75 & 16.94 & 29.90 & 21.37 & 25.25 & 16.70 & 65.00 & 88.88 & 40.99 & 16.75 & 3.50 & 59.50 & 18.00 & 55.15 & 57.20 & \underline{36.26} \\

&
6K & 20.85 & 40.32 & 45.58 & 52.80 & 48.03 & 41.63 & 18.47 & 24.98 & 30.68 & 22.32 & 25.12 & 16.71 & 68.50 & 90.00 & 41.49 & 18.25 & 4.50 & 85.00 & 18.50 & 55.46 & 57.30 & \underline{39.36} \\

&
8K & 20.64 & 41.10 & 47.89 & 54.70 & 48.49 & 41.79 & 21.58 & 31.46 & 32.03 & 23.32 & 25.22 & 16.65 & 71.00 & 90.00 & 41.44 & 21.00 & 4.00 & 95.00 & 19.50 & 55.44 & 57.54 & \underline{40.94} \\

\midrule

% KeyDiff
\multirow{4}{*}{\ourmethod{}} &
2K & 17.40 & 38.12 & 45.25 & 47.09 & 45.28 & 34.23 & 13.97 & 27.96 & 28.34 & 21.09 & 24.94 & 13.45 & 56.00 & 83.29 & 38.53 & 24.25 & 1.00 & 63.50 & 12.50 & 54.41 & 40.00 & \textbf{34.79} \\

&
4K & 22.38 & 42.00 & 50.84 & 53.16 & 47.34 & 40.56 & 21.43 & 32.82 & 30.96 & 23.32 & 25.08 & 15.35 & 67.50 & 87.09 & 42.53 & 27.50 & 2.00 & 89.00 & 18.50 & 54.39 & 53.09 & \textbf{40.33} \\

&
6K & 22.25 & 41.55 & 50.32 & 54.91 & 51.61 & 42.02 & 24.62 & 34.60 & 32.40 & 23.63 & 25.26 & 16.18 & 71.00 & 88.42 & 41.90 & 29.25 & 3.50 & 95.00 & 19.50 & 55.69 & 55.27 & \textbf{41.85} \\

&
8K & 21.57 & 41.24 & 50.12 & 55.33 & 49.98 & 43.78 & 27.45 & 34.30 & 32.43 & 23.67 & 25.21 & 15.99 & 71.50 & 90.84 & 42.32 & 30.00 & 3.00 & 96.50 & 19.50 & 55.54 & 56.56 & \textbf{42.23} \\

\midrule
\midrule
\addlinespace

%%%%%%%%%%%%%%%%%%%%%%%%%%%%%%%%%%%%%%%%%%%%%%%%%%%%%%%%
%%%%%%%%%%%%%%%%%%%%% B=256 Results %%%%%%%%%%%%%%%%%%%%
%%%%%%%%%%%%%%%%%%%%%%%%%%%%%%%%%%%%%%%%%%%%%%%%%%%%%%%%

% Dense
\multicolumn{2}{c}{$B=256$} &
23.76 & 40.23 & 50.09 & 55.84 & 50.69 & 42.29 & 26.84 & 36.24 & 33.09 & 24.30 & 25.21 & 16.41 & 72.50 & 90.11 & 42.58 & 34.00 & 3.00 & 96.50 & 20.50 & 56.22 & 56.52 & 42.71 \\
\midrule

% H2O
\multirow{4}{*}{H2O} &
2K & 1.87 & 19.19 & 23.35 & 16.06 & 20.95 & 17.91 & 2.33 & 7.31 & 0.83 & 1.73 & 16.29 & 3.28 & 42.50 & 25.82 & 9.83 & 14.50 & 2.75 & 0.00 & 1.00 & 38.59 & 14.57 & 13.36 \\

&
4K & 5.58 & 31.49 & 33.28 & 32.36 & 26.85 & 28.34 & 8.61 & 10.83 & 5.46 & 6.27 & 23.13 & 4.98 & 54.00 & 41.73 & 19.14 & 15.00 & 1.00 & 8.00 & 6.50 & 51.39 & 24.24 & 20.87 \\

&
6K & 7.49 & 37.99 & 41.54 & 44.71 & 39.53 & 36.18 & 15.46 & 13.87 & 10.77 & 10.46 & 24.60 & 6.94 & 61.50 & 58.64 & 27.20 & 15.50 & 1.00 & 15.00 & 19.00 & 54.30 & 32.45 & 27.34 \\

&
8K & 9.92 & 39.71 & 43.84 & 52.15 & 39.14 & 39.23 & 19.28 & 18.04 & 17.78 & 13.60 & 24.98 & 8.99 & 71.00 & 67.64 & 32.83 & 17.25 & 2.50 & 30.00 & 21.00 & 55.25 & 39.33 & 31.59 \\
\midrule

% TOVA
\multirow{4}{*}{TOVA} &
2K & 18.46 & 30.80 & 33.74 & 32.24 & 39.73 & 32.18 & 14.10 & 10.17 & 26.32 & 20.17 & 25.18 & 15.67 & 62.00 & 89.36 & 40.60 & 16.25 & 2.00 & 16.00 & 4.50 & 55.98 & 53.42 & 30.42 \\

&
4K & 20.36 & 38.18 & 42.53 & 46.18 & 46.83 & 36.60 & 17.81 & 17.09 & 28.99 & 20.54 & 25.23 & 16.34 & 63.50 & 89.13 & 41.55 & 19.12 & 3.50 & 29.50 & 14.50 & 55.55 & 55.91 & 34.71 \\

&
6K & 20.71 & 40.46 & 45.82 & 52.95 & 51.33 & 40.92 & 21.47 & 24.80 & 30.71 & 22.37 & 25.48 & 16.40 & 67.00 & 88.50 & 41.91 & 20.25 & 3.50 & 49.00 & 20.00 & 55.85 & 56.74 & 37.91 \\

&
8K & 20.84 & 40.79 & 48.02 & 54.82 & 50.12 & 40.71 & 25.17 & 30.85 & 31.47 & 22.98 & 25.49 & 16.41 & 71.00 & 89.00 & 41.99 & 22.25 & 2.00 & 76.50 & 21.00 & 55.93 & 57.59 & 40.23 \\
\midrule

% Sink
\multirow{4}{*}{Sink} &
2K & 15.48 & 30.74 & 26.93 & 27.76 & 33.86 & 25.63 & 13.30 & 9.56 & 26.70 & 19.41 & 25.15 & 15.56 & 66.00 & 86.44 & 41.17 & 18.50 & 1.00 & 19.00 & 8.50 & 56.47 & 52.37 & 29.50 \\

&
4K & 18.91 & 38.38 & 37.28 & 41.76 & 41.81 & 35.17 & 13.07 & 9.96 & 29.23 & 20.61 & 25.01 & 16.38 & 70.50 & 88.06 & 42.07 & 21.75 & 1.50 & 36.00 & 17.00 & 56.05 & 51.98 & 33.93 \\

&
6K & 19.37 & 40.30 & 38.14 & 49.77 & 46.14 & 39.60 & 14.18 & 11.12 & 30.48 & 21.54 & 25.20 & 16.28 & 71.00 & 88.80 & 41.79 & 23.50 & 3.50 & 47.00 & 20.00 & 55.65 & 53.53 & 36.04 \\

&
8K & 20.35 & 40.19 & 42.96 & 53.25 & 47.82 & 41.10 & 17.87 & 14.83 & 31.32 & 22.14 & 25.17 & 16.56 & 72.50 & 89.26 & 42.47 & 26.25 & 5.00 & 62.50 & 20.00 & 56.00 & 56.03 & 38.27 \\
\midrule

% SnapKV
\multirow{4}{*}{SnapKV} &
2K & 17.04 & 31.55 & 31.75 & 31.87 & 37.25 & 34.03 & 12.17 & 9.80 & 27.17 & 20.16 & 25.26 & 16.27 & 61.00 & 87.63 & 40.95 & 17.50 & 4.00 & 35.00 & 10.00 & 55.93 & 54.39 & \underline{31.46} \\

&
4K & 19.67 & 39.34 & 40.95 & 45.81 & 44.27 & 38.78 & 16.17 & 16.61 & 29.99 & 21.15 & 25.49 & 16.71 & 66.50 & 89.15 & 40.76 & 19.75 & 3.50 & 66.00 & 19.00 & 55.44 & 56.45 & \underline{36.74} \\

&
6K & 22.98 & 40.13 & 44.80 & 52.39 & 50.61 & 39.28 & 20.31 & 24.84 & 31.28 & 22.17 & 25.42 & 16.55 & 70.00 & 89.79 & 41.91 & 19.75 & 3.50 & 86.50 & 20.00 & 55.93 & 57.49 & \underline{39.79} \\

&
8K & 20.23 & 40.75 & 47.40 & 54.77 & 49.74 & 41.47 & 24.74 & 31.19 & 31.95 & 22.96 & 25.45 & 16.27 & 71.50 & 90.17 & 42.12 & 23.25 & 4.00 & 93.50 & 21.00 & 55.93 & 57.36 & \underline{41.23} \\
\midrule

% KeyDiff
\multirow{4}{*}{\ourmethod{}} &
2K & 18.99 & 37.20 & 46.57 & 46.61 & 45.14 & 33.12 & 15.15 & 29.54 & 28.30 & 21.92 & 25.41 & 14.44 & 58.50 & 86.30 & 37.77 & 23.75 & 1.00 & 67.00 & 13.50 & 54.07 & 41.59 & \textbf{35.52} \\

&
4K & 21.00 & 41.48 & 48.69 & 52.94 & 47.07 & 40.26 & 22.69 & 33.07 & 31.11 & 23.35 & 25.22 & 15.57 & 65.50 & 88.72 & 41.79 & 26.50 & 2.00 & 91.00 & 19.00 & 55.79 & 51.50 & \textbf{40.20} \\

&
6K & 21.61 & 40.78 & 49.68 & 55.08 & 50.64 & 41.38 & 24.57 & 34.62 & 32.02 & 23.36 & 25.56 & 15.86 & 71.50 & 89.42 & 42.37 & 28.75 & 3.00 & 94.00 & 21.00 & 55.96 & 55.96 & \textbf{41.77} \\

& 
8K & 22.24 & 40.83 & 50.23 & 55.45 & 50.50 & 43.28 & 28.37 & 33.88 & 32.67 & 23.85 & 25.48 & 15.83 & 72.50 & 90.34 & 42.31 & 30.50 & 3.00 & 95.50 & 21.00 & 56.01 & 55.58 & \textbf{42.35} \\

\bottomrule
\end{tabular}
\end{adjustbox}
}
\end{center}
\end{table*}

%% file: tables/longbench_full_rotated_llama3.x_keydiff-recent-tokens.tex
\begin{table*}[!t]
\caption{\textbf{\ourmethod{} + Recent tokens Llama-3.1-8B/3.2-3B-Instruct LongBench Results with $B=128$ (Higher is better)}. We highlight the methods showing the best performance within a given budget with \textbf{boldface}. $\text{X}\%$ indicates $\text{X}\%$ of KV cache budget is reserved to keep the recent tokens, while the remaining cache budget is managed by \ourmethod{} algorithm. \textdagger: A subset of samples were evaluated due to OOM errors (183/200 samples are evaluated).}
\label{table:longbench_full_rotated_sliding_window}
\begin{center}
\resizebox{0.75\textwidth}{0.9\textheight}{%
\begin{adjustbox}{angle=90}
\begin{tabular}{cccccccccccccccccccccccc}
\toprule

%%%%%%%%%%%%%%%%%%%%%%%%%%%%%%%%%%%%%%%%%%%%%%%%%%%%%%%%
%%%%%%%%%%%%%%%%%%%%% Header Rows %%%%%%%%%%%%%%%%%%%%%&
%%%%%%%%%%%%%%%%%%%%%%%%%%%%%%%%%%%%%%%%%%%%%%%%%%%%%%%%
&
&
\multicolumn{4}{c}{Single Doc. QA} &
\multicolumn{4}{c}{Multi Doc. QA} &
\multicolumn{4}{c}{Summarization} &
\multicolumn{4}{c}{Few$\-$shot Learning} &
\multicolumn{3}{c}{Synthetic} &
\multicolumn{2}{c}{Code} &
\\

\cmidrule(lr){3-6}
\cmidrule(lr){7-10}
\cmidrule(lr){11-14}
\cmidrule(lr){15-18}
\cmidrule(lr){19-21}
\cmidrule(lr){22-23}

&
&
{Narrative QA} & {Qasper} & {MF-en} & {MF-zh} &
{HotpotQA} & {2WikiMQA} & {Musique} & {DuReader} &
{GovReport} & {QMSum} & {MultiNews} & {VCSum} &
{TREC} & {TriviaQA} & {SAMSum} & {LSHT} &
{PCount} & {PR-en} & {PR-zh} &
{Lcc} & {RB-P} & 
{Avg.} 
\\

%%%%%%%%%%%%%%%%%%%%%%%%%%%%%%%%%%%%%%%%%%%%%%%%%%%%%%%%
%%%%%%%%%%%%%%%%% Llama 3.1-8B Results %%%%%%%%%%%%%%%%%
%%%%%%%%%%%%%%%%%%%%%%%%%%%%%%%%%%%%%%%%%%%%%%%%%%%%%%%%
\midrule
\midrule
\addlinespace

% Llama 3.1 8b Dense
\multicolumn{2}{c}{Llama3.1-8B} &
30.05\textsuperscript{\textdagger} & 47.00 & 56.12 & 59.86 & 57.33 & 47.81 & 32.25 & 35.64 & 34.86 & 25.32 & 27.02 & 17.28 & 73.00 & 91.61 & 43.37 & 45.50 & 8.33 & 99.50 & 99.00 & 61.66 & 51.94 & 49.74 \\
\midrule

% Sink
\multirow{4}{*}{Sink} &
2K & 21.83 & 34.27 & 29.24 & 32.82 & 38.64 & 29.50 & 12.59 & 16.18 & 28.51 & 20.21 & 26.62 & 15.54 & 65.00 & 89.46 & 42.20 & 22.25 & 2.00 & 25.50 & 32.50 & 64.95 & 59.54 & 33.78 \\

&
4K & 22.94 & 43.01 & 39.08 & 46.16 & 44.04 & 41.39 & 19.09 & 16.54 & 31.08 & 21.57 & 26.78 & 16.73 & 70.00 & 91.53 & 42.29 & 29.25 & 3.00 & 38.50 & 71.00 & 62.12 & 58.84 & 39.76 \\

&
6K & 25.41 & 47.40 & 44.13 & 52.78 & 47.39 & 45.73 & 21.90 & 17.55 & 32.53 & 22.19 & 26.87 & 17.05 & 72.00 & 91.25 & 43.41 & 33.75 & 3.08 & 52.50 & 98.00 & 62.22 & 56.24 & 43.49 \\

&
8K & 23.53 & 46.63 & 48.68 & 55.77 & 49.61 & 47.16 & 21.14 & 19.54 & 33.10 & 23.20 & 26.92 & 16.91 & 72.00 & 91.29 & 43.79 & 37.00 & 3.25 & 66.00 & 99.00 & 62.18 & 56.43 & 44.91 \\

\midrule

% SnapKV
\multirow{4}{*}{SnapKV} &
2K & 21.81 & 37.22 & 37.19 & 38.29 & 46.10 & 35.42 & 16.53 & 16.37 & 29.83 & 21.05 & 26.77 & 16.16 & 61.00 & 88.84 & 42.56 & 21.75 & 4.03 & 51.50 & 81.17 & 62.37 & 51.45 & 38.45 \\

&
4K & 24.79 & 44.22 & 47.30 & 50.27 & 48.49 & 46.73 & 20.55 & 22.04 & 32.19 & 22.68 & 26.95 & 16.95 & 67.50 & 90.98 & 43.14 & 25.00 & 5.17 & 89.50 & 96.67 & 61.44 & 51.20 & 44.46 \\

&
6K & 24.10 & 45.57 & 50.44 & 55.27 & 53.12 & 48.41 & 24.27 & 27.46 & 33.43 & 23.53 & 27.03 & 16.84 & 71.50 & 92.28 & 43.58 & 27.00 & 5.25 & 98.00 & 99.00 & 61.32 & 52.16 & 46.65 \\

&
8K & 25.15 & 46.55 & 53.39 & 57.65 & 56.00 & 48.75 & 27.82 & 32.66 & 33.67 & 24.85 & 27.01 & 17.37 & 72.50 & 91.78 & 43.54 & 33.75 & 5.08 & 100.00 & 98.67 & 61.48 & 51.41 & 48.05 \\

\midrule

% KeyDiff + 10% Recent tokens
\multirow{4}{*}{\shortstack[l]{KeyDiff + 10\%\\ Recent Tokens}} &
2K & 27.36 & 40.79 & 50.63 & 50.21 & 49.44 & 44.77 & 25.04 & 30.25 & 29.41 & 23.28 & 26.76 & 15.78 & 64.00 & 85.90 & 43.76 & 44.00 & 4.25 & 98.50 & 95.83 & 62.97 & 48.70 & \textbf{45.79} \\

&
4K & 27.42 & 45.67 & 53.62 & 57.40 & 53.77 & 47.30 & 27.48 & 33.40 & 31.90 & 24.18 & 26.83 & 16.51 & 70.00 & 88.89 & 43.95 & 47.00 & 5.21 & 99.00 & 98.67 & 62.20 & 52.98 & \textbf{48.26} \\

&
6K & 30.58 & 46.24 & 55.58 & 57.82 & 56.44 & 47.74 & 29.03 & 34.33 & 33.15 & 24.67 & 26.91 & 16.98 & 71.50 & 92.22 & 43.96 & 47.00 & 4.23 & 99.50 & 99.00 & 62.56 & 52.98 & \underline{49.16} \\

&
8K & 32.17 & 46.66 & 55.65 & 58.65 & 57.24 & 48.64 & 30.54 & 33.44 & 33.85 & 24.93 & 26.94 & 17.12 & 72.50 & 91.72 & 43.70 & 46.00 & 5.47 & 99.50 & 99.00 & 62.56 & 51.75 & 49.43 \\

\midrule

% KeyDiff + 20% Recent tokens
\multirow{4}{*}{\shortstack[l]{KeyDiff + 20\%\\ Recent Tokens}} &
2K & 26.73 & 41.88 & 49.79 & 48.46 & 49.68 & 42.51 & 26.90 & 30.90 & 28.83 & 23.34 & 26.76 & 16.11 & 62.50 & 87.39 & 43.99 & 42.50 & 4.58 & 94.00 & 95.58 & 63.33 & 50.01 & \underline{45.51} \\

&
4K & 26.05 & 45.11 & 55.51 & 55.87 & 54.14 & 47.41 & 25.52 & 33.51 & 31.87 & 24.40 & 26.95 & 16.09 & 70.00 & 90.22 & 43.76 & 43.50 & 4.00 & 99.50 & 98.67 & 62.06 & 52.27 & 47.92 \\

&
6K & 28.39 & 46.43 & 55.12 & 58.27 & 57.02 & 48.94 & 28.98 & 34.31 & 33.32 & 24.62 & 26.96 & 17.20 & 71.50 & 91.72 & 44.02 & 46.50 & 5.13 & 99.50 & 99.00 & 62.39 & 52.57 & 49.14 \\

&
8K & 31.35 & 46.74 & 54.71 & 58.60 & 58.19 & 48.14 & 31.77 & 34.02 & 33.55 & 24.98 & 26.95 & 17.11 & 72.50 & 91.72 & 44.17 & 46.00 & 7.38 & 99.50 & 99.00 & 62.39 & 52.04 & \textbf{49.56} \\

\midrule

% KeyDiff
\multirow{4}{*}{\ourmethod{}} &
2K & 26.64 & 41.73 & 50.99 & 51.18 & 51.59 & 46.47 & 22.84 & 32.37 & 29.02 & 23.86 & 26.76 & 14.81 & 66.50 & 85.92 & 39.26 & 42.25 & 3.17 & 96.00 & 96.25 & 59.17 & 39.42 & 45.06 \\

&
4K & 28.70 & 45.62 & 56.06 & 56.83 & 54.58 & 49.31 & 28.25 & 33.06 & 32.30 & 25.03 & 27.07 & 16.32 & 70.00 & 90.85 & 42.84 & 44.50 & 4.21 & 99.00 & 97.67 & 60.80 & 48.00 & \underline{48.14} \\

&
6K & 29.90 & 46.33 & 55.11 & 59.00 & 56.80 & 49.50 & 31.52 & 34.97 & 33.44 & 24.58 & 26.98 & 16.80 & 72.00 & 90.99 & 43.10 & 47.00 & 5.27 & 99.50 & 99.00 & 61.40 & 49.70 & \textbf{49.19} \\
&
8K & 33.57 & 46.77 & 55.48 & 59.16 & 56.87 & 49.37 & 30.88 & 34.54 & 34.17 & 25.12 & 27.01 & 17.13 & 72.50 & 92.28 & 42.81 & 46.50 & 5.83 & 99.50 & 98.67 & 61.48 & 50.90 & \underline{49.55} \\

\midrule
\midrule
\addlinespace

%%%%%%%%%%%%%%%%%%%%%%%%%%%%%%%%%%%%%%%%%%%%%%%%%%%%%%%%
%%%%%%%%%%%%%%%%% Llama 3.2-3B Results %%%%%%%%%%%%%%%%%
%%%%%%%%%%%%%%%%%%%%%%%%%%%%%%%%%%%%%%%%%%%%%%%%%%%%%%%%

% Dense
\multicolumn{2}{c}{Llama3.2-3B} &
23.76 & 40.23 & 50.09 & 55.84 & 50.69 & 42.29 & 26.84 & 36.24 & 33.09 & 24.30 & 25.21 & 16.41 & 72.50 & 90.11 & 42.58 & 34.00 & 3.00 & 96.50 & 20.50 & 56.22 & 56.52 & 42.71 \\
\midrule

% Sink
\multirow{4}{*}{Sink} &
2K & 16.85 & 30.69 & 26.58 & 27.32 & 33.26 & 25.27 & 13.82 & 9.38 & 26.74 & 19.15 & 25.15 & 15.88 & 65.00 & 86.17 & 40.79 & 19.50 & 1.50 & 19.50 & 8.50 & 56.65 & 52.73 & 29.54 \\

&
4K & 19.46 & 38.61 & 36.22 & 41.68 & 41.97 & 35.84 & 13.37 & 9.86 & 29.34 & 20.19 & 25.06 & 16.44 & 71.00 & 88.06 & 41.31 & 21.75 & 2.50 & 35.50 & 16.00 & 56.48 & 52.43 & 33.96 \\

&
6K & 19.33 & 40.29 & 37.95 & 49.68 & 46.48 & 40.29 & 15.31 & 11.10 & 30.43 & 21.35 & 25.14 & 16.64 & 71.50 & 88.93 & 42.04 & 23.50 & 3.50 & 47.00 & 19.50 & 56.55 & 54.11 & 36.22 \\

&
8K & 20.15 & 40.02 & 41.94 & 53.57 & 48.15 & 42.24 & 16.01 & 14.76 & 31.64 & 22.10 & 25.20 & 16.50 & 73.00 & 89.26 & 42.37 & 26.25 & 3.50 & 62.50 & 20.50 & 56.86 & 56.63 & 38.25 \\

\midrule

% SnapKV
\multirow{4}{*}{SnapKV} &
2K & 17.38 & 31.37 & 31.48 & 29.65 & 37.77 & 30.05 & 11.54 & 9.66 & 27.03 & 19.93 & 24.97 & 15.97 & 59.00 & 88.13 & 40.48 & 16.25 & 3.50 & 32.50 & 9.00 & 56.32 & 55.91 & 30.85 \\

&
4K & 19.85 & 39.22 & 39.86 & 47.33 & 46.70 & 37.98 & 16.64 & 16.88 & 29.79 & 21.21 & 25.01 & 16.74 & 65.50 & 89.35 & 40.95 & 18.25 & 2.50 & 62.50 & 22.50 & 55.74 & 56.88 & 36.73 \\

&
6K & 20.83 & 39.65 & 44.48 & 51.84 & 49.30 & 40.18 & 20.28 & 25.32 & 31.27 & 22.73 & 25.09 & 16.81 & 69.00 & 89.95 & 41.47 & 18.75 & 4.00 & 85.00 & 20.50 & 55.69 & 57.82 & 39.52 \\

&
8K & 20.49 & 40.80 & 48.16 & 55.44 & 48.78 & 41.65 & 24.79 & 30.40 & 31.81 & 23.46 & 25.17 & 16.44 & 70.00 & 90.17 & 41.99 & 22.00 & 5.00 & 94.00 & 21.50 & 55.77 & 57.29 & 41.20 \\

\midrule

% KeyDiff + 10% Recent tokens
\multirow{4}{*}{\shortstack[l]{KeyDiff + 10\%\\ Recent Tokens}} &
2K & 19.92 & 37.58 & 45.99 & 46.41 & 44.24 & 34.81 & 15.20 & 29.22 & 28.23 & 22.49 & 24.99 & 14.71 & 59.00 & 84.08 & 41.38 & 24.38 & 3.00 & 70.50 & 15.50 & 57.49 & 51.91 & \textbf{36.72} \\

&
4K & 21.97 & 40.70 & 49.76 & 52.53 & 47.79 & 41.92 & 20.71 & 32.83 & 31.12 & 23.37 & 25.07 & 15.77 & 67.50 & 87.47 & 41.28 & 25.00 & 2.50 & 91.00 & 20.00 & 57.14 & 56.25 & \textbf{40.56} \\

&
6K & 23.76 & 40.56 & 50.79 & 54.23 & 50.54 & 42.15 & 25.20 & 34.12 & 31.95 & 23.10 & 25.21 & 15.91 & 71.50 & 88.17 & 41.67 & 26.50 & 3.50 & 96.00 & 20.50 & 56.80 & 55.18 & 41.78 \\

&
8K & 23.65 & 40.58 & 49.96 & 55.71 & 51.52 & 44.10 & 25.95 & 34.41 & 32.80 & 23.77 & 25.25 & 15.75 & 73.50 & 89.67 & 41.83 & 28.75 & 3.00 & 96.00 & 20.50 & 57.00 & 56.12 & 42.37 \\

\midrule

% KeyDiff + 20% Recent tokens
\multirow{4}{*}{\shortstack[l]{KeyDiff + 20\% \\Recent Tokens}} &
2K & 19.27 & 34.86 & 45.26 & 44.94 & 42.81 & 34.15 & 14.27 & 27.31 & 28.10 & 22.13 & 25.08 & 14.96 & 61.50 & 85.07 & 41.62 & 24.17 & 2.00 & 71.50 & 16.00 & 57.56 & 53.58 & \underline{36.48} \\

&
4K & 22.56 & 41.28 & 48.37 & 52.00 & 47.50 & 42.64 & 19.64 & 32.28 & 31.17 & 22.83 & 25.08 & 15.68 & 68.50 & 88.17 & 41.38 & 24.25 & 2.00 & 90.50 & 19.50 & 57.04 & 55.05 & \underline{40.35} \\

&
6K & 22.88 & 40.74 & 50.37 & 55.02 & 49.90 & 42.34 & 25.02 & 34.95 & 31.97 & 23.28 & 25.24 & 16.32 & 71.50 & 89.61 & 41.43 & 27.00 & 3.00 & 95.50 & 20.50 & 56.82 & 55.77 & \underline{41.86} \\

&
8K & 23.82 & 40.58 & 50.02 & 55.71 & 51.25 & 43.64 & 24.52 & 34.33 & 33.02 & 23.82 & 25.20 & 16.08 & 73.50 & 89.67 & 41.99 & 29.25 & 3.00 & 97.00 & 20.50 & 57.05 & 56.48 & \textbf{42.40} \\

\midrule

% KeyDiff
\multirow{4}{*}{\ourmethod{}} &
2K & 18.29 & 36.65 & 45.44 & 47.47 & 46.09 & 35.41 & 13.79 & 28.89 & 28.16 & 21.45 & 25.01 & 13.56 & 60.00 & 85.24 & 37.00 & 24.88 & 1.00 & 60.50 & 12.00 & 54.13 & 42.01 & 35.09 \\

&
4K & 22.34 & 40.60 & 49.15 & 52.56 & 50.14 & 40.30 & 21.65 & 32.46 & 31.38 & 23.44 & 25.06 & 15.28 & 66.50 & 87.92 & 41.41 & 27.50 & 2.50 & 88.50 & 19.50 & 55.55 & 52.24 & 40.28 \\

&
6K & 22.29 & 40.68 & 50.14 & 54.51 & 51.74 & 42.19 & 24.83 & 34.64 & 32.39 & 23.53 & 25.19 & 15.88 & 71.00 & 90.02 & 42.00 & 28.75 & 3.00 & 95.00 & 21.50 & 55.86 & 54.39 & \textbf{41.88} \\

& 
8K & 22.41 & 40.77 & 50.10 & 55.62 & 49.83 & 43.58 & 28.09 & 34.30 & 32.78 & 23.60 & 25.17 & 15.77 & 72.00 & 90.17 & 42.46 & 30.75 & 3.50 & 96.50 & 21.50 & 55.85 & 55.65 & \underline{42.40} \\

\bottomrule
\end{tabular}
\end{adjustbox}
}
\end{center}
\end{table*}

%% file: contents/A05_ablation_study.tex
\section{Ablation study}
\label{appendix:ablation-study}
\input{tables/anchor_vector_ablation}

\paragraph{Selecting the Anchor Vector} We have mainly evaluated \ourmethod{} using the method described in \cref{eq:keydiff}. Scores to determine eviction are measured via cosine similarity with an \emph{anchor vector} which can be computed in several ways. We run LongBench on Llama3.2-3B-Instruct with eviction policies using the following anchor choices: pairwise cosine similarity from \cref{eq:keydiff}, denoted $\operatorname{Pairwise}$; \ourmethod{}, using the mean of all normalized keys as an anchor, and using the median of keys as an anchor, denoted $\operatorname{Median}$.
\cref{table:anchor_vector_ablation} summarizes the average LongBench accuracy for the different methods to Llama 3.2-3B-Instruct.
\ourmethod{} shows similar average scores to $\operatorname{Pairwise}$. %validating the efficient variants of \ourmethod{}. 
Additionally, \ourmethod{} and $\operatorname{Median}$ show similar scores, demonstrating that \ourmethod{} is robust to the selection of the anchor design.

\input{tables/distance_metric_ablation}

\paragraph{Selecting the Similarity Metric} We use cosine similarity as the scoring metric for eviction  in \ourmethod{} based on our discussion in \cref{sec:kv_caching}. 
This could be replaced with other metrics like the dot product or Euclidean distance. 
We evaluate \ourmethod{} variants using dot product and Euclidean distance as the similarity metric, denoted $\operatorname{DotProd}$ and $\operatorname{Euclidean}$ respectively, and report the results in \cref{table:similarity_metric_ablation}.
\ourmethod{} and $\operatorname{DotProd}$ show similar performance for 6K and 8K budgets. However, \ourmethod{} outperforms $\operatorname{DotProd}$ for smaller cache sizes. This implies that considering both the direction and the magnitude of the keys to compute similarity are important for identifying the tokens to evict. On the other hand, $\operatorname{Euclidean}$ shows a significant performance drop relative to \ourmethod{}.

\subsection{Needle in a Haystack results for Sink attention}

\begin{figure*}[t!]
     \centering
     \includegraphics[width=.5\linewidth]{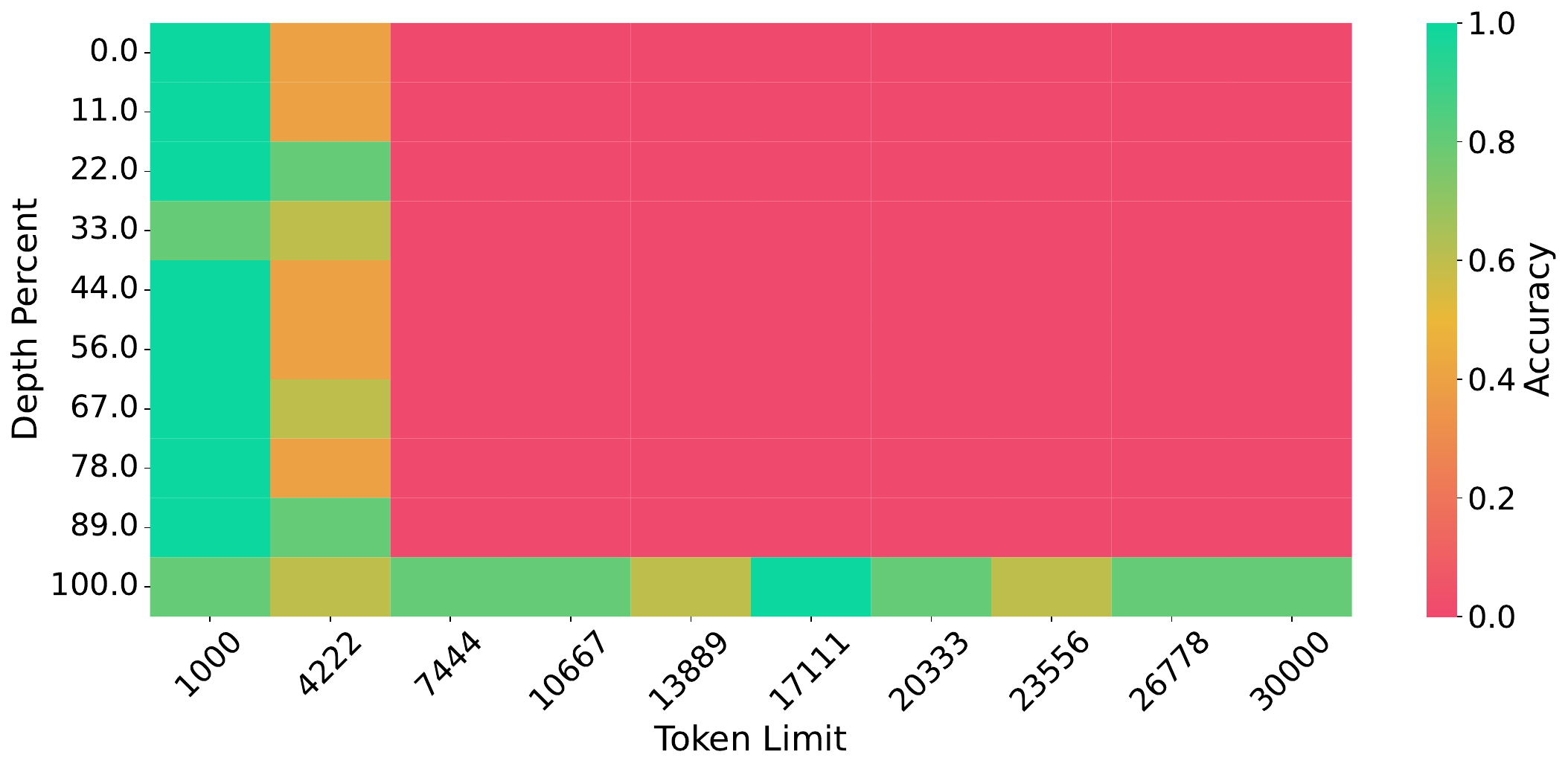}
     \caption{Needle in a Haystack results for Sink Attention \cite{xiaoefficient}}
     \label{fig:needle-in-a-haystack-sink}
\end{figure*}

%% file: tables/anchor_vector_ablation.tex
\begin{table}[t]
\caption{\textbf{Anchor vector ablation study.} Average of Full Longbench results for Llama 3.2-3B-Instruct. \ourmethod{} results match \cref{table:longbench_full_rotated}. 
(Higher is better)
}
\label{table:anchor_vector_ablation}
\begin{center}\resizebox{0.35\textwidth}{!}{
\begin{tabular}{ccccc}

\toprule
& 2K & 4K & 6K & 8K \\
\midrule
\ourmethod{}                  & 35.09 & 40.28 & 41.88 & 42.40 \\
\midrule
$\operatorname{Pairwise}$     & 35.24 & 40.61 & 41.87 & 42.45 \\
$\operatorname{Median}$       & 35.43 & 40.67 & 41.89 & 42.26 \\

\bottomrule
\end{tabular}
}
\end{center}
\vspace{-5mm}
\end{table}

%% file: tables/distance_metric_ablation.tex
\begin{table}[t]
\caption{\textbf{Distance metric ablation study.} Average of Full Longbench results of Llama 3.2-3B-Instruct. \ourmethod{} results match \cref{table:longbench_full_rotated}. (Higher is better)}
\label{table:similarity_metric_ablation}
\begin{center}\resizebox{0.35\textwidth}{!}{
\begin{tabular}{ccccc}

\toprule
& 2K & 4K & 6K & 8K \\
\midrule
\ourmethod{}               & 35.09 & 40.28 & 41.88 & 42.40 \\
\midrule
$\operatorname{DotProd}$   & 30.14 & 38.01 & 41.09 & 42.23 \\
$\operatorname{Euclidean}$ & 13.68 & 21.06 & 25.91 & 29.53 \\

\bottomrule
\end{tabular}
}
\end{center}
\vspace{-5mm}
\end{table}

%% file: contents/A06_additional_evaluation.tex
\newpage
\section{Additional Experiments and Analyses}
\label{app:additional_results}

To further substantiate the empirical findings presented in
\cref{sec:method,sec:experiment}, we provide an extended set of experiments and complementary analyses. These additional studies examine the correlation between key geometry and attention, evaluate retrieval-critical and reasoning scenarios, and measure the efficiency of \ourmethod{} under constrained hardware settings. Together, they offer a broader view of the method’s robustness, efficiency, and general applicability across diverse inference conditions.

\subsection{Retrieval-Critical Evaluation: Phonebook Lookup}
\label{app:phonebook}

We evaluate \ourmethod{} on a retrieval-critical setting using the \textit{Phonebook Lookup} task, where the model retrieves a phone number corresponding to a queried name
from a long list of entries. Accuracy is averaged across five random phonebooks of varying lengths. As shown in \cref{tab:phonebook}, \ourmethod{} maintains high retrieval accuracy for shorter contexts and degrades more gracefully than attention-based baselines as context length increases.

\begin{table}[h!]
\centering
\caption{Accuracy (\%) on the Phonebook Lookup task using Llama-3.2-3B-Instruct with a 6k cache budget.}
\label{tab:phonebook}
\scriptsize
\begin{tabular}{lcccccccccc}
\toprule
Method & 100 & 478 & 856 & 1233 & 1611 & 1989 & 2367 & 2744 & 3122 & 3500 \\
\midrule
Dense   & 1.0 & 1.0 & 1.0 & 0.8 & 1.0 & 0.8 & 0.6 & 0.6 & 0.8 & 0.6 \\
KeyDiff & 1.0 & 1.0 & 1.0 & 0.6 & 0.4 & 0.2 & 0.0 & 0.0 & 0.2 & 0.2 \\
TOVA    & 1.0 & 1.0 & 1.0 & 0.4 & 0.2 & 0.0 & 0.0 & 0.0 & 0.0 & 0.0 \\
\bottomrule
\end{tabular}
\end{table}

\subsection{RULER Benchmark Validation}
\label{app:ruler}

To assess the generalizability of \ourmethod{} across diverse architectures,
we reference the results reported by the community-maintained
\href{https://huggingface.co/spaces/kvpress/leaderboard}{KVPress RULER benchmark}.
\ourmethod{} consistently achieves competitive or superior leaderboard scores
compared to methods such as TOVA, SnapKV, QFilter, and Knorm
on both Llama-3.2-3B and Qwen-3-8B backbones, demonstrating its robustness and transferability.

\subsection{Needle-in-a-Haystack Recall Saturation}
\label{app:niah}

We further analyze the recall behavior of \ourmethod{} under varying
context lengths in the \textit{Needle-in-a-Haystack} benchmark.
Table~\ref{tab:niah} reports the recall difference between
\ourmethod{} and the full-cache baseline as a function of depth and context length. As shown in \cref{tab:niah}, \ourmethod{} achieves near-parity recall with the full-cache baseline up to 30k tokens, confirming its stability under long-context compression.

\begin{table}[h]
\centering
\caption{Recall difference (\ourmethod{}-Full) on NIAH benchmark for Llama-3.2-3B.}
\label{tab:niah}
\scriptsize
\begin{tabular}{lrrrrrrrrrr}
\toprule
Depth & 1000 & 4222 & 7444 & 10667 & 13889 & 17111 & 20333 & 23556 & 26778 & 30000 \\
\midrule
0.0 & 0.02 & $-$0.06 & $-$0.06 & 0.20 & 0.16 & $-$0.16 & 0.18 & 0.20 & 0.06 & 0.04 \\
11.0 & 0.10 & $-$0.04 & $-$0.04 & 0.30 & 0.28 & $-$0.18 & 0.26 & 0.28 & 0.10 & $-$0.08 \\
22.0 & 0.04 & $-$0.12 & 0.00 & 0.32 & 0.30 & $-$0.18 & 0.16 & 0.28 & 0.14 & $-$0.04 \\
33.0 & 0.02 & 0.06 & $-$0.20 & 0.34 & 0.08 & $-$0.16 & 0.26 & 0.26 & 0.20 & $-$0.16 \\
44.0 & $-$0.06 & $-$0.08 & $-$0.16 & 0.30 & 0.26 & $-$0.12 & 0.24 & 0.22 & $-$0.10 & 0.04 \\
56.0 & $-$0.06 & 0.00 & $-$0.24 & 0.26 & 0.26 & $-$0.20 & 0.18 & 0.14 & 0.04 & $-$0.02 \\
67.0 & 0.02 & 0.02 & $-$0.14 & 0.36 & 0.28 & $-$0.26 & 0.24 & 0.16 & 0.04 & $-$0.08 \\
78.0 & $-$0.10 & 0.08 & $-$0.26 & 0.26 & 0.30 & $-$0.26 & 0.20 & 0.30 & 0.02 & $-$0.04 \\
89.0 & 0.02 & 0.08 & $-$0.06 & 0.30 & 0.26 & $-$0.32 & 0.30 & 0.22 & 0.06 & $-$0.12 \\
100.0 & 0.00 & $-$0.06 & $-$0.26 & 0.00 & $-$0.06 & $-$0.02 & 0.00 & $-$0.04 & 0.14 & $-$0.04 \\
\bottomrule
\end{tabular}
\end{table}

\subsection{On-Device Latency Evaluation}
\label{app:latency_ondevice}

We further evaluate the runtime characteristics of \ourmethod{} on a
mobile-class device by measuring the latency required to compute key-eviction scores on an Android platform. All methods are tested on a recent Samsung smartphone under identical
FP16 inference precision, and the results are normalized by the latency of
\ourmethod{} with a cache budget of 512~KVs.

As shown in \cref{tab:latency_android}, \ourmethod{} performs on par with competing methods for small cache sizes and achieves substantially lower scoring latency as the cache size increases, demonstrating both scalability and minimal overhead on edge hardware. Minor runtime fluctuations can be attributed to kernel-level optimizations
and proprietary hardware characteristics.

\begin{table}[h!]
\centering
\caption{Relative latency of key scoring on an Android device (normalized by \textsc{KeyDiff} at 512~KVs). Lower is better.}
\label{tab:latency_android}
\scriptsize
\begin{tabular}{lccccc}
\toprule
Method & 512 & 1024 & 2048 & 4096 & 8192 \\
\midrule
SnapKV & 1.92 & 2.46 & 4.50 & 8.94 & 11.42 \\
TOVA   & 1.02 & 0.94 & 1.25 & 1.95 & 3.06 \\
KeyDiff & \textbf{1.00} & 1.03 & 1.01 & 1.07 & 1.37 \\
H2O    & 1.10 & 0.99 & 1.52 & 2.32 & 4.23 \\
\bottomrule
\end{tabular}
\end{table}

%% file: contents/checklist.tex
% \newpage
% \section{NeurIPS Paper Checklist}

\begin{enumerate}
\item {\bf{Claims}}
    \begin{itemize}[label={}]
    \vspace{1mm}
    \item Question: Do the main claims made in the abstract and introduction accurately reflect the paper’s contributions and scope?
    \vspace{1mm}
    \item Answer: \answerYes
    \vspace{1mm}
    \item Justification: The claims made in the abstract are corroborated by both our theoretical justification \cref{subsec:keydiff-theory}, experiments \cref{sec:experiment} and auxiliary results and discussion.
    \vspace{1mm}
    \end{itemize}

\item  {\bf{Limitations}}
\begin{itemize}[label={}]
    \vspace{1mm}
    \item Question: Does the paper discuss the limitations of the work performed by the authors?
    % \vspace{.5mm}
    \item Answer: \answerYes
    \vspace{1mm}
    \item Justification: In \cref{sec:conclusion} we discuss the fact that \ourmethod{} is only tested on decoder based models featuring GQA and would like to extend it in the future to models with different designs such as MLA.
    \vspace{1mm}
    \end{itemize}
\item {\bf{Theory Assumptions and Proofs}}
\begin{itemize}[label={}]
    \vspace{1mm}
    \item Question: For each theoretical result, does the paper provide the full set of assumptions and a complete (and correct) proof?
    \vspace{1mm}
    \item Answer: \answerYes
    \vspace{1mm}
    \item Justification: Proofs of all results are discussed briefly in their statement in \cref{subsec:keydiff-theory} and rigorously proved in \cref{appendix:theoretical-justification}.
    \vspace{1mm}
    \end{itemize}

\item {\bf{Experimental Result Reproducibility}}
\begin{itemize}[label={}]
    \vspace{1mm}
    \item Question:  If the contribution is a dataset or model, what steps did you take to make your results reproducible or verifiable? Depending on the contribution, reproducibility can be accomplished in various ways. For example, if the contribution is a novel architecture, describing the architecture fully might suffice, or if the contribution is a specific model and empirical evaluation, it may be necessary to either make it possible for others to replicate the model with the same dataset, or provide access to the model. In general. releasing code and data is often one good way to accomplish this, but reproducibility can also be provided via detailed instructions for how to replicate the results, access to a hosted model (e.g., in the case of a large language model), release of a model checkpoint, or other means that are appropriate to your research.
    \vspace{1mm}
    \item Answer: \answerNA
    \vspace{1mm}
    \item Justification: The paper releases neither a dataset or a model, though the cache management algorithm can be easily implemented using the description from the paper.
    \vspace{1mm}
    \end{itemize}
\item {\bf{Open access to data and code}}
\begin{itemize}[label={}]
    \vspace{1mm}
    \item Question: Does the paper provide open access to the data and code, with sufficient instructions to faithfully reproduce the main experimental results, as described in supplemental material?
    \vspace{1mm}
    \item Answer: \answerNo
    \vspace{1mm}
    \item Justification: Code is not provided as it is proprietary. However the algorithm may be readily implemented and results replicated using the description in the paper.
    \vspace{1mm}
    \end{itemize}

\item {\bf{Experimental Setting/Details}}
\begin{itemize}[label={}]
    \vspace{1mm}
    \item Question: If you ran experiments, did you specify all the training details (e.g., data splits, hyperparameters, how they were chosen)? The full details can be provided with the code, but the important details should be in the main paper, and information about how hyperparameters were selected should appear either in the paper or supplementary materials.
    \vspace{1mm}
    \item Answer: \answerYes
    \vspace{1mm}
    \item Justification: All hyperparameters selected for the benchmarks run to evaluate \ourmethod{} are specified in \cref{sec:experiment,appendix:subsec:longbench-setup,appendix:math500}. Hyperparameters to recreate figures are given in figure descriptions.
    \vspace{1mm}
    \end{itemize}

\item {\bf{ Experiment Statistical Significance}}
\begin{itemize}[label={}]
    \vspace{1mm}
    \item Question: Does the paper report error bars suitably and correctly defined or other appropriate information about the statistical significance of the experiments?
    \vspace{1mm}
    \item Answer: \answerYes
    \vspace{1mm}
    \item Justification: The only experiment using stochasticity \cref{fig:ttft_flash} and \cref{fig:ttft_eager} reports error bars and statistical significance. The Math500 results as well discuss the use of sampling and stochasticity in generating responses. Aside from this, the remaining results are deterministic.
    \vspace{1mm}
    \end{itemize}

\item {\bf{Experiments Compute Resources}}
\begin{itemize}[label={}]
    \vspace{1mm}
    \item Question: For each experiment, does the paper provide sufficient information on the computer resources (type of compute workers, memory, time of execution) needed to reproduce the experiments?
    \vspace{1mm}
    \item Answer: \answerYes
    \vspace{1mm}
    \item Justification: In all relevant cases, particularly experiments on TTFT \cref{fig:ttft_eager} and \cref{fig:ttft_flash}, we mention the GPUs and compute resources used.
    \vspace{1mm}
    \end{itemize}

\item {\bf{Code Of Ethics}}
\begin{itemize}[label={}]
    \vspace{1mm}
    \item Question: Does the research conducted in the paper conform, in every respect, with the NeurIPS Code of Ethics https://neurips.cc/public/EthicsGuidelines?
    \vspace{1mm}
    \item Answer: \answerYes
    \vspace{1mm}
    \item Justification: The research proposes a new cache management algorithm based on the geometry of the attention mechanism in decoder-based LLMs. It does not involve human subjects or obviously ethically sensitive applications, and it conforms to the NeurIPS Code of Ethics.
    \vspace{1mm}
    \end{itemize}

\item {\bf{Broader Impacts}}
\begin{itemize}[label={}]
    \vspace{1mm}
    \item Question: Does the paper discuss both potential positive societal impacts and negative societal impacts of the work performed?
    \vspace{1mm}
    \item Answer: \answerNA
    \vspace{1mm}
    \item Justification: This research is foundational in nature and is not tied to particular applications or deployments. As such we do not see any direct path to negative applications of our work.
    \vspace{1mm}
    \end{itemize}

\item {\bf{Safeguards}}
\begin{itemize}[label={}]
    \vspace{1mm}
    \item Question: Does the paper describe safeguards that have been put in place for responsible release of data or models that have a high risk for misuse (e.g., pretrained language models, image generators, or scraped datasets)?
    \vspace{1mm}
    \item Answer: \answerNA
    \vspace{1mm}
    \item Justification: The paper proposes a new cache eviction algorithm. Neither the models used, data or cache eviction algorithm appear to pose a high risk for abuse necessitating release safeguards beyond standard practices.
    \vspace{1mm}
    \end{itemize}

\item {\bf{Licenses for existing assets}}
\begin{itemize}[label={}]
    \vspace{1mm}
    \item Question: Are the creators or original owners of assets (e.g., code, data, models), used in the paper, properly credited and are the license and terms of use explicitly mentioned and properly respected?
    \vspace{1mm}
    \item Answer: \answerYes
    \vspace{1mm}
    \item Justification: Creators for datasets/benchmarks LongBench and Math500 are properly credited with citations. Open source models from the Llama and Qwen model classes are used to test \ourmethod{} are properly cited as are their deepseek distill r1 variants.
    \vspace{1mm}
    \end{itemize}

\item {\bf{Assets}}
\begin{itemize}[label={}]
    \vspace{1mm}
    \item Question: If you are releasing new assets, did you document them and provide these details alongside the assets? Researchers should communicate the details of the dataset or the model as part of their submissions via structured templates. This includes details about training, license, limitations, etc. 
    \vspace{1mm}
    \item Answer: \answerNA
    \vspace{1mm}
    \item Justification: The research does not release new assets.
    \vspace{1mm}
    \end{itemize}

\item {\bf{Crowdsourcing and Research with Human Subjects}}
\begin{itemize}[label={}]
    \vspace{1mm}
    \item Question: If you used crowdsourcing or conducted research with human subjects, did you include the full text of instructions given to participants and screenshots, if applicable, as well as details about compensation (if any)? Including this information in the supplemental material is fine, but if the main contribution of your paper involves human subjects, then we strongly encourage you to include as much detail as possible in the main paper. According to the NeurIPS Code of Ethics, you must pay workers involved in data collection, curation, or other labor at least the minimum wage in your country. 
    \vspace{1mm}
    \item Answer: \answerNA
    \vspace{1mm}
    \item Justification: The research does not involve crowdsourcing experiments or research with human subjects.
    \vspace{1mm}
    \end{itemize}

\item {\bf{IRB Approvals}}
\begin{itemize}[label={}]
    \vspace{1mm}
    \item Question: Does the paper describe potential risks incurred by study participants, whether such risks were disclosed to the subjects, and whether Institutional Review Board (IRB) approvals (or an equivalent approval/review based on the requirements of your country or institution) were obtained?
    \vspace{1mm}
    \item Answer: \answerNA
    \vspace{1mm}
    \item Justification: The research does not involve human subjects, therefore IRB approval is not applicable.
    \vspace{1mm}
    \end{itemize}

\item {\bf{Declaration of LLM usage}}
\begin{itemize}[label={}]
    \vspace{1mm}
    \item Question: Does the paper describe the usage of LLMs if it is an important, original, or non-standard component of the core methods in this research? Note that if the LLM is used only for writing, editing, or formatting purposes and does not impact the core methodology, scientific rigorousness, or originality of the research, declaration is not required.
    \vspace{1mm}
    \item Answer: \answerNA
    \vspace{1mm}
    \item Justification: The research does not include any input from an LLM. In particular the core methodology and research is original to the authors.
    \vspace{1mm}
    \end{itemize}

\end{enumerate}

%% file: reference.bib
@inproceedings{xiaoefficient,
  title={Efficient Streaming Language Models with Attention Sinks},
  author={Xiao, Guangxuan and Tian, Yuandong and Chen, Beidi and Han, Song and Lewis, Mike},
  booktitle={The Twelfth International Conference on Learning Representations},
  year={2024}
}

@article{zhang2024h2o,
  title={H2o: Heavy-hitter oracle for efficient generative inference of large language models},
  author={Zhang, Zhenyu and Sheng, Ying and Zhou, Tianyi and Chen, Tianlong and Zheng, Lianmin and Cai, Ruisi and Song, Zhao and Tian, Yuandong and R{\'e}, Christopher and Barrett, Clark and others},
  journal={Advances in Neural Information Processing Systems},
  volume={36},
  year={2024}
}

@article{oren2024transformers,
  title={Transformers are multi-state rnns},
  author={Oren, Matanel and Hassid, Michael and Adi, Yossi and Schwartz, Roy},
  journal={arXiv preprint arXiv:2401.06104},
  year={2024}
}

@inproceedings{bai-etal-2024-longbench,
    title = "{L}ong{B}ench: A Bilingual, Multitask Benchmark for Long Context Understanding",
    author = "Bai, Yushi  and
      Lv, Xin  and
      Zhang, Jiajie  and
      Lyu, Hongchang  and
      Tang, Jiankai  and
      Huang, Zhidian  and
      Du, Zhengxiao  and
      Liu, Xiao  and
      Zeng, Aohan  and
      Hou, Lei  and
      Dong, Yuxiao  and
      Tang, Jie  and
      Li, Juanzi",
    editor = "Ku, Lun-Wei  and
      Martins, Andre  and
      Srikumar, Vivek",
    booktitle = "Proceedings of the 62nd Annual Meeting of the Association for Computational Linguistics (Volume 1: Long Papers)",
    month = aug,
    year = "2024",
    address = "Bangkok, Thailand",
    publisher = "Association for Computational Linguistics",
    url = "https://aclanthology.org/2024.acl-long.172",
    doi = "10.18653/v1/2024.acl-long.172",
    pages = "3119--3137",
}

@article{vaswani2017attention,
  title={Attention is all you need},
  author={Vaswani, A},
  journal={Advances in Neural Information Processing Systems},
  year={2017}
}

@article{dao2022flashattention,
  title={Flashattention: Fast and memory-efficient exact attention with io-awareness},
  author={Dao, Tri and Fu, Dan and Ermon, Stefano and Rudra, Atri and R{\'e}, Christopher},
  journal={Advances in Neural Information Processing Systems},
  volume={35},
  pages={16344--16359},
  year={2022}
}

@article{shazeer2020glu,
  title={Glu variants improve transformer},
  author={Shazeer, Noam},
  journal={arXiv preprint arXiv:2002.05202},
  year={2020}
}

@article{li2024snapkv,
  title={Snapkv: Llm knows what you are looking for before generation},
  author={Li, Yuhong and Huang, Yingbing and Yang, Bowen and Venkitesh, Bharat and Locatelli, Acyr and Ye, Hanchen and Cai, Tianle and Lewis, Patrick and Chen, Deming},
  journal={arXiv preprint arXiv:2404.14469},
  year={2024}
}

@article{yang2024pyramidinfer,
  title={PyramidInfer: Pyramid KV Cache Compression for High-throughput LLM Inference},
  author={Yang, Dongjie and Han, XiaoDong and Gao, Yan and Hu, Yao and Zhang, Shilin and Zhao, Hai},
  journal={arXiv preprint arXiv:2405.12532},
  year={2024}
}

@article{xu2024think,
  title={Think: Thinner key cache by query-driven pruning},
  author={Xu, Yuhui and Jie, Zhanming and Dong, Hanze and Wang, Lei and Lu, Xudong and Zhou, Aojun and Saha, Amrita and Xiong, Caiming and Sahoo, Doyen},
  journal={arXiv preprint arXiv:2407.21018},
  year={2024}
}

@article{ainslie2023gqa,
  title={Gqa: Training generalized multi-query transformer models from multi-head checkpoints},
  author={Ainslie, Joshua and Lee-Thorp, James and de Jong, Michiel and Zemlyanskiy, Yury and Lebr{\'o}n, Federico and Sanghai, Sumit},
  journal={arXiv preprint arXiv:2305.13245},
  year={2023}
}

@article{hooper2024kvquant,
  title={Kvquant: Towards 10 million context length llm inference with kv cache quantization},
  author={Hooper, Coleman and Kim, Sehoon and Mohammadzadeh, Hiva and Mahoney, Michael W and Shao, Yakun Sophia and Keutzer, Kurt and Gholami, Amir},
  journal={arXiv preprint arXiv:2401.18079},
  year={2024}
}

@article{liu2024kivi,
  title={Kivi: A tuning-free asymmetric 2bit quantization for kv cache},
  author={Liu, Zirui and Yuan, Jiayi and Jin, Hongye and Zhong, Shaochen and Xu, Zhaozhuo and Braverman, Vladimir and Chen, Beidi and Hu, Xia},
  journal={arXiv preprint arXiv:2402.02750},
  year={2024}
}

@article{yang2024no,
  title={No token left behind: Reliable kv cache compression via importance-aware mixed precision quantization},
  author={Yang, June Yong and Kim, Byeongwook and Bae, Jeongin and Kwon, Beomseok and Park, Gunho and Yang, Eunho and Kwon, Se Jung and Lee, Dongsoo},
  journal={arXiv preprint arXiv:2402.18096},
  year={2024}
}

@article{zhang2024kv,
  title={KV Cache is 1 Bit Per Channel: Efficient Large Language Model Inference with Coupled Quantization},
  author={Zhang, Tianyi and Yi, Jonah and Xu, Zhaozhuo and Shrivastava, Anshumali},
  journal={arXiv preprint arXiv:2405.03917},
  year={2024}
}

@article{sun2024massive,
  title={Massive activations in large language models},
  author={Sun, Mingjie and Chen, Xinlei and Kolter, J Zico and Liu, Zhuang},
  journal={arXiv preprint arXiv:2402.17762},
  year={2024}
}

@inproceedings{kwon2023efficient,
  title={Efficient memory management for large language model serving with pagedattention},
  author={Kwon, Woosuk and Li, Zhuohan and Zhuang, Siyuan and Sheng, Ying and Zheng, Lianmin and Yu, Cody Hao and Gonzalez, Joseph and Zhang, Hao and Stoica, Ion},
  booktitle={Proceedings of the 29th Symposium on Operating Systems Principles},
  pages={611--626},
  year={2023}
}

@article{tang2024quest,
  title={Quest: Query-Aware Sparsity for Efficient Long-Context LLM Inference},
  author={Tang, Jiaming and Zhao, Yilong and Zhu, Kan and Xiao, Guangxuan and Kasikci, Baris and Han, Song},
  journal={arXiv preprint arXiv:2406.10774},
  year={2024}
}

@article{rehg2024kv,
  title={KV-Compress: Paged KV-Cache Compression with Variable Compression Rates per Attention Head},
  author={Rehg, Isaac},
  journal={arXiv preprint arXiv:2410.00161},
  year={2024}
}

@article{dubey2024llama,
  title={The llama 3 herd of models},
  author={Dubey, Abhimanyu and Jauhri, Abhinav and Pandey, Abhinav and Kadian, Abhishek and Al-Dahle, Ahmad and Letman, Aiesha and Mathur, Akhil and Schelten, Alan and Yang, Amy and Fan, Angela and others},
  journal={arXiv preprint arXiv:2407.21783},
  year={2024}
}

@article{liu2024lost,
  title={Lost in the middle: How language models use long contexts},
  author={Liu, Nelson F and Lin, Kevin and Hewitt, John and Paranjape, Ashwin and Bevilacqua, Michele and Petroni, Fabio and Liang, Percy},
  journal={Transactions of the Association for Computational Linguistics},
  volume={12},
  pages={157--173},
  year={2024},
  publisher={MIT Press One Broadway, 12th Floor, Cambridge, Massachusetts 02142, USA~…}
}

@misc{kamradt2023needle,
    author = {G. Kamradt},
    title = {Needle in a haystack - pressure testing llms},
    year = {2023},
    howpublished = {GitHub repository},
    url = {https://github.com/gkamradt/LLMTest_NeedleInAHaystack},
}

@article{zhu2024sampleattention,
  title={SampleAttention: Near-Lossless Acceleration of Long Context LLM Inference with Adaptive Structured Sparse Attention},
  author={Zhu, Qianchao and Duan, Jiangfei and Chen, Chang and Liu, Siran and Li, Xiuhong and Feng, Guanyu and Lv, Xin and Cao, Huanqi and Chuanfu, Xiao and Zhang, Xingcheng and others},
  journal={CoRR},
  year={2024}
}

@inproceedings{ribarsparq,
  title={SparQ Attention: Bandwidth-Efficient LLM Inference},
  author={Ribar, Luka and Chelombiev, Ivan and Hudlass-Galley, Luke and Blake, Charlie and Luschi, Carlo and Orr, Douglas},
  booktitle={Forty-first International Conference on Machine Learning}
}

@article{brown2020language,
  title={Language models are few-shot learners},
  author={Brown, Tom and Mann, Benjamin and Ryder, Nick and Subbiah, Melanie and Kaplan, Jared D and Dhariwal, Prafulla and Neelakantan, Arvind and Shyam, Pranav and Sastry, Girish and Askell, Amanda and others},
  journal={Advances in neural information processing systems},
  volume={33},
  pages={1877--1901},
  year={2020}
}

@article{raffel2020exploring,
  title={Exploring the limits of transfer learning with a unified text-to-text transformer},
  author={Raffel, Colin and Shazeer, Noam and Roberts, Adam and Lee, Katherine and Narang, Sharan and Matena, Michael and Zhou, Yanqi and Li, Wei and Liu, Peter J},
  journal={Journal of machine learning research},
  volume={21},
  number={140},
  pages={1--67},
  year={2020}
}

@article{touvron2023llama,
  title={Llama 2: Open foundation and fine-tuned chat models},
  author={Touvron, Hugo and Martin, Louis and Stone, Kevin and Albert, Peter and Almahairi, Amjad and Babaei, Yasmine and Bashlykov, Nikolay and Batra, Soumya and Bhargava, Prajjwal and Bhosale, Shruti and others},
  journal={arXiv preprint arXiv:2307.09288},
  year={2023}
}

@article{wei2022chain,
  title={Chain-of-thought prompting elicits reasoning in large language models},
  author={Wei, Jason and Wang, Xuezhi and Schuurmans, Dale and Bosma, Maarten and Xia, Fei and Chi, Ed and Le, Quoc V and Zhou, Denny and others},
  journal={Advances in neural information processing systems},
  volume={35},
  pages={24824--24837},
  year={2022}
}

@article{yao2024tree,
  title={Tree of thoughts: Deliberate problem solving with large language models},
  author={Yao, Shunyu and Yu, Dian and Zhao, Jeffrey and Shafran, Izhak and Griffiths, Tom and Cao, Yuan and Narasimhan, Karthik},
  journal={Advances in Neural Information Processing Systems},
  volume={36},
  year={2024}
}

@article{kojima2022large,
  title={Large language models are zero-shot reasoners},
  author={Kojima, Takeshi and Gu, Shixiang Shane and Reid, Machel and Matsuo, Yutaka and Iwasawa, Yusuke},
  journal={Advances in neural information processing systems},
  volume={35},
  pages={22199--22213},
  year={2022}
}

@article{lewis2020retrieval,
  title={Retrieval-augmented generation for knowledge-intensive nlp tasks},
  author={Lewis, Patrick and Perez, Ethan and Piktus, Aleksandra and Petroni, Fabio and Karpukhin, Vladimir and Goyal, Naman and K{\"u}ttler, Heinrich and Lewis, Mike and Yih, Wen-tau and Rockt{\"a}schel, Tim and others},
  journal={Advances in Neural Information Processing Systems},
  volume={33},
  pages={9459--9474},
  year={2020}
}

@article{alizadeh2023llm,
  title={Llm in a flash: Efficient large language model inference with limited memory},
  author={Alizadeh, Keivan and Mirzadeh, Iman and Belenko, Dmitry and Khatamifard, Karen and Cho, Minsik and Del Mundo, Carlo C and Rastegari, Mohammad and Farajtabar, Mehrdad},
  journal={arXiv preprint arXiv:2312.11514},
  year={2023}
}

@article{liu2024mobilellm,
  title={Mobilellm: Optimizing sub-billion parameter language models for on-device use cases},
  author={Liu, Zechun and Zhao, Changsheng and Iandola, Forrest and Lai, Chen and Tian, Yuandong and Fedorov, Igor and Xiong, Yunyang and Chang, Ernie and Shi, Yangyang and Krishnamoorthi, Raghuraman and others},
  journal={arXiv preprint arXiv:2402.14905},
  year={2024}
}

@article{van2024gptvq,
  title={Gptvq: The blessing of dimensionality for llm quantization},
  author={van Baalen, Mart and Kuzmin, Andrey and Nagel, Markus and Couperus, Peter and Bastoul, Cedric and Mahurin, Eric and Blankevoort, Tijmen and Whatmough, Paul},
  journal={arXiv preprint arXiv:2402.15319},
  year={2024}
}

@article{agrawal2023sarathi,
  title={Sarathi: Efficient llm inference by piggybacking decodes with chunked prefills},
  author={Agrawal, Amey and Panwar, Ashish and Mohan, Jayashree and Kwatra, Nipun and Gulavani, Bhargav S and Ramjee, Ramachandran},
  journal={arXiv preprint arXiv:2308.16369},
  year={2023}
}

@article{holmes2024deepspeed,
  title={Deepspeed-fastgen: High-throughput text generation for llms via mii and deepspeed-inference},
  author={Holmes, Connor and Tanaka, Masahiro and Wyatt, Michael and Awan, Ammar Ahmad and Rasley, Jeff and Rajbhandari, Samyam and Aminabadi, Reza Yazdani and Qin, Heyang and Bakhtiari, Arash and Kurilenko, Lev and others},
  journal={arXiv preprint arXiv:2401.08671},
  year={2024}
}

@article{xu2024empowering,
  title={Empowering 1000 tokens/second on-device llm prefilling with mllm-npu},
  author={Xu, Daliang and Zhang, Hao and Yang, Liming and Liu, Ruiqi and Huang, Gang and Xu, Mengwei and Liu, Xuanzhe},
  journal={arXiv preprint arXiv:2407.05858},
  year={2024}
}

@article{barbero2025llms,
  title={Why do LLMs attend to the first token?},
  author={Barbero, Federico and Arroyo, {\'A}lvaro and Gu, Xiangming and Perivolaropoulos, Christos and Bronstein, Michael and Pascanu, Razvan and others},
  journal={arXiv preprint arXiv:2504.02732},
  year={2025}
}

@article{godey2024anisotropy,
  title={Anisotropy is inherent to self-attention in transformers},
  author={Godey, Nathan and de la Clergerie, {\'E}ric and Sagot, Beno{\^\i}t},
  journal={arXiv preprint arXiv:2401.12143},
  year={2024}
}

@article{hendrycks2021measuring,
  title={Measuring mathematical problem solving with the math dataset},
  author={Hendrycks, Dan and Burns, Collin and Kadavath, Saurav and Arora, Akul and Basart, Steven and Tang, Eric and Song, Dawn and Steinhardt, Jacob},
  journal={arXiv preprint arXiv:2103.03874},
  year={2021}
}

@article{guo2025deepseek,
  title={Deepseek-r1: Incentivizing reasoning capability in llms via reinforcement learning},
  author={Guo, Daya and Yang, Dejian and Zhang, Haowei and Song, Junxiao and Zhang, Ruoyu and Xu, Runxin and Zhu, Qihao and Ma, Shirong and Wang, Peiyi and Bi, Xiao and others},
  journal={arXiv preprint arXiv:2501.12948},
  year={2025}
}

@article{beltagy2020longformer,
  title={Longformer: The long-document transformer},
  author={Beltagy, Iz and Peters, Matthew E and Cohan, Arman},
  journal={arXiv preprint arXiv:2004.05150},
  year={2020}
}

@article{devoto2024simple,
  title={A Simple and Effective $ L\_2 $ Norm-Based Strategy for KV Cache Compression},
  author={Devoto, Alessio and Zhao, Yu and Scardapane, Simone and Minervini, Pasquale},
  journal={arXiv preprint arXiv:2406.11430},
  year={2024}
}

@misc{yang2024qwen2technicalreport,
      title={Qwen2 Technical Report}, 
      author={An Yang and Baosong Yang and Binyuan Hui and Bo Zheng and Bowen Yu and Chang Zhou and Chengpeng Li and Chengyuan Li and Dayiheng Liu and Fei Huang and Guanting Dong and Haoran Wei and Huan Lin and Jialong Tang and Jialin Wang and Jian Yang and Jianhong Tu and Jianwei Zhang and Jianxin Ma and Jianxin Yang and Jin Xu and Jingren Zhou and Jinze Bai and Jinzheng He and Junyang Lin and Kai Dang and Keming Lu and Keqin Chen and Kexin Yang and Mei Li and Mingfeng Xue and Na Ni and Pei Zhang and Peng Wang and Ru Peng and Rui Men and Ruize Gao and Runji Lin and Shijie Wang and Shuai Bai and Sinan Tan and Tianhang Zhu and Tianhao Li and Tianyu Liu and Wenbin Ge and Xiaodong Deng and Xiaohuan Zhou and Xingzhang Ren and Xinyu Zhang and Xipin Wei and Xuancheng Ren and Xuejing Liu and Yang Fan and Yang Yao and Yichang Zhang and Yu Wan and Yunfei Chu and Yuqiong Liu and Zeyu Cui and Zhenru Zhang and Zhifang Guo and Zhihao Fan},
      year={2024},
      eprint={2407.10671},
      archivePrefix={arXiv},
      primaryClass={cs.CL},
      url={https://arxiv.org/abs/2407.10671}, 
}
